\documentclass[10pt]{article} 
\usepackage[accepted]{tmlr}

\usepackage{amsmath,amsfonts,bm}

\def\eqref#1{equation~\ref{#1}}

\def\1{\bm{1}}

\DeclareMathAlphabet{\mathsfit}{\encodingdefault}{\sfdefault}{m}{sl}
\SetMathAlphabet{\mathsfit}{bold}{\encodingdefault}{\sfdefault}{bx}{n}

\usepackage{hyperref}
\usepackage{url}
\usepackage{algorithm}
\usepackage{algpseudocode}
\usepackage{amsmath,amsfonts,amssymb,amscd,amsthm,dsfont}
\usepackage{graphicx,epsfig}
\usepackage{subcaption}
\newtheorem{assumption}{\textbf{Assumption}}
\newtheorem{theorem}{\textbf{Theorem}}

\newtheorem{lemma}{\textbf{Lemma}}
\newtheorem{remark}{\textbf{Remark}}

\usepackage{pifont}
\newcommand{\cmark}{\ding{51}}%
\newcommand{\xmark}{\ding{55}}%
\newcommand{\zhang}[1]{{\color{blue}{#1}}}
\allowdisplaybreaks

\title{Convergence Guarantees for RMSProp and Adam in Generalized-smooth Non-convex Optimization with Affine Noise Variance}

\author{\name Qi Zhang \email qzhan261@asu.edu \\
      \addr School of Electrical, Computer and Energy Engineering\\
      Arizona State University
      \AND
      \name Yi Zhou \email yi.zhou@tamu.edu \\
      \addr Department of Computer Science and Engineerin\\
       Texas {\normalfont A\&M} University
      \AND
      \name Shaofeng Zou \email zou@asu.edu\\
      \addr School of Electrical, Computer and Energy Engineering\\
      Arizona State University}

\def\month{02}  
\def\year{2025} 
\def\openreview{\url{https://openreview.net/forum?id=QIzRdjIWnS}} 

\begin{document}

\maketitle

\begin{abstract}
This paper provides the first tight convergence analyses for RMSProp and Adam for non-convex optimization under the most relaxed assumptions of coordinate-wise generalized smoothness and affine noise variance. 
RMSProp is firstly analyzed, which is a special case of Adam with adaptive learning rates but without first-order momentum. Specifically, to solve the challenges due to the dependence among adaptive update, unbounded gradient estimate and Lipschitz constant, we demonstrate that the first-order term in the descent lemma converges and its denominator is upper bounded by a function of gradient norm. Based on this result, we show that RMSProp with proper hyperparameters converges to an $\epsilon$-stationary point with an iteration complexity of $\mathcal O(\epsilon^{-4})$. We then generalize our analysis to Adam, where the additional challenge is due to a mismatch between the gradient and the first-order momentum. We develop a new upper bound on the first-order term in the descent lemma, which is also a function of the gradient norm. We show that Adam with proper hyperparameters converges to an $\epsilon$-stationary point with an iteration complexity of $\mathcal O(\epsilon^{-4})$. Our complexity results for both RMSProp and Adam match with the complexity lower bound established in \citet{arjevani2023lower}.
\end{abstract}

\section{Introduction}\label{sec:intro}
RMSProp \citep{hinton2012lecture} and Adam \citep{kingma2014adam} are among the most popular and powerful adaptive optimizers in training state-of-the-art machine learning models \citep{brock2018large,brown2020language,cutkosky2020momentum,dosovitskiy2020image}. 
RMSProp and Adam only
require first-order gradients with little memory requirement, and thus are efficient to use in practice. RMSProp is based on the idea of adaptive learning rates for each individual parameter, and  Adam
combines the benefits of RMSprop \citep{hinton2012lecture} and AdaGrad \citep{duchi2011adaptive}, which consists of two key components of adaptive learning rates and momentum. Despite their empirical success, theoretical understandings on the convergence and complexity, especially when optimizing non-convex loss functions, e.g., neural networks, still remain underdeveloped until very recently. 

Recently, there have been a series of works in examining the convergence and complexity of RMSProp and Adam for non-convex loss functions (see Table \ref{table:1} for a detailed review). However, these works do not completely explain the performance of RMSProp and Adam in training neural networks, as they rely on assumptions that may not necessarily hold. For example, \citet{zhang2019gradient} pointed out that neural networks are not $L$-smooth, and instead satisfy the generalized $(L_0,L_1)$-smoothness, where the gradient Lipschitz constant increases linearly with the gradient norm. Furthermore, many of these works assumed that the stochastic gradient has a bounded norm/variance, which however does not even hold for linear regression
\citep{wang2023convergence}, and instead a relaxed affine noise variance condition shall be used. 

In this paper, we derive the convergence guarantee and iteration complexity for RMSProp and Adam with coordinate-wise generalized $(L_0,L_1)$-smooth loss function and affine noise variance. To the best of our knowledge, this is one of the most relaxed assumption sets in the convergence analyses of RMSProp and Adam that best describe the training of some neural networks. 
We prove that RMSProp and Adam with proper hyperparameters converge to an $\epsilon$-stationary point with a complexity of $\mathcal O(\epsilon^{-4})$, which matches with the lower bound for first-order optimization in \citet{arjevani2023lower}.

\subsection{Related work}
\subsubsection{Relaxed Assumptions}
\textbf{Affine Noise Variance:}  In most of the studies on stochastic optimization, access to an unbiased estimate of the gradient with uniformly bounded variance is assumed \citep{nemirovskij1983problem,ghadimi2013stochastic,bubeck2015convex,foster2019complexity}. \citet{ghadimi2013stochastic} first showed that for $L$-smooth objectives, the SGD algorithm converges to a first-order $\epsilon$-stationary point with an iteration complexity of  $\mathcal O(\epsilon^{-4})$ if the stochastic gradient has uniformly bounded variance. Furthermore, \citet{arjevani2023lower} proved that for any first-order algorithm with uniformly bounded gradient variance, the iteration complexity of $\mathcal O(\epsilon^{-4})$ is optimal. For overparameterized neural networks, \citet{vaswani2019fast} considered another gradient noise assumption: the strong growth condition, where the upper bound on the second-order moment of  {the norm of gradient estimate} scales with the gradient square norm. Both the uniformly bounded variance and strong growth condition are special cases of the affine noise variance. It was demonstrated in \citet{bottou2018optimization} that for non-adaptive algorithms with affine noise variance, the optimal iteration complexity of $\mathcal O(\epsilon^{-4})$ can be achieved. The extension of affine noise variance assumption to adaptive algorithms is not straightforward and was studied in \citet{jin2021non,chen2023generalized,wang2022provable,wang2023convergence,shi2020rmsprop,faw2022power,faw2023beyond,hong2023high}. In this paper, we study two adaptive optimizers: RMSProp and Adam with affine noise variance.

\textbf{Generalized Smoothness:} In stochastic optimization, the $L$-smooth objectives are widely assumed \citep{ghadimi2013stochastic,ghadimi2016mini}. However, it was demonstrated in  \citet{zhang2019gradient} that the $L$-smoothness does not hold for some neural networks and polynomial functions with degree larger than $2$. Then, extensive experiments were conducted to verify that these functions satisfy the generalized ($L_0,L_1$)-smoothness condition, where the
gradient Lipschitz constant increases linearly with the gradient norm. Compared with $L$-smoothness, ($L_0,L_1$)-smoothness introduces extra second-order error terms, thus making the optimization problem hard to solve. The clipping algorithms for generalized smooth function were studied in \citet{zhang2019gradient,zhang2020improved}. However, they require the gradient norm to be bounded. A relaxed assumption on bounded gradient variance was studied in \citet{reisizadeh2023variance}, where the SPIDER algorithm was applied. Furthermore, \citet{chen2023generalized} showed that for generalized smooth objectives with affine noise variance, the SPIDER algorithm still finds a stationary point. Under the same assumption, \cite{jin2021non} provided the convergence rate for a normalized momentum algorithm. With extensive experiments, \citet{crawshaw2022robustness} showed that in the training of Transformer, the $(L_0,L_1)$-smoothness holds coordinate-wisely.  {This condition is widely used in coordinate-wise type
optimizers like generalized SignSGD and Adam. Note that for the original Adam, proving the expectation of gradient norm converges with $(L_0,L_1)$-smoothness remains an unresolved issue.}
In this paper, we consider functions that are coordinate-wise $(L_0,L_1)$-smooth. 

\subsubsection{Adaptive Optimizers}
Adaptive optimizers are widely used in deep learning due to their ability to adapt to changing data and conditions. Adagrad \citep{duchi2011adaptive} is the first adaptive algorithm, which calculates the accumulated sum of the past gradient norms and uses the reciprocal of its square root to scale the current gradient. Recently, \citet{wang2023convergence,faw2023beyond} studied Adagrad under generalized smoothness and affine noise variance conditions. However, the training of the above Adagrad algorithm may stop in advance since the accumulated sum does not shrink, and thus the learning rate can be extremely close to $0$. To overcome this problem, RMSProp \citep{hinton2012lecture} was proposed, where a momentum method is employed to replace the accumulated sum. Thus, the adaptive learning rate can increase or decrease. There is a rich literature on the convergence analyses of RMSProp \citep{zaheer2018adaptive,de2018convergence,shi2020rmsprop}. However, all of them focus on $L$-smooth objectives, and only \cite{shi2020rmsprop} considered the affine noise variance.
RMSProp is a special case of Adam, which only includes the second-order momentum, and is widely studied e.g., \citet{defossez2020simple,zou2019sufficient,chen2018convergence,zhang2022adam,wang2022provable,guo2021novel,hong2023high,li2023convergence,wang2023closing}. 
 {There are also two recent works \citep{hong2023high,wang2023closing} that studied Adam for $L$-smooth objectives with affine noise variance. However, their methods can not be generalized to $(L_0,L_1)$ smooth objectives due to the additional terms invalidating the key inequalities or requirements of bounded Lipchitz constant in their key Lemma.}
In \citet{li2023convergence},  {Adam for  $(L_0,L_1)$-smooth objectives with sub-Gaussian norm was studied, where the gradient estimate bias follows a sub-Gaussian distribution. Under this assumption, based on the gradient estimate, the real gradient belongs to a bounded set with high probability, which converts the unbounded Lipschitz constant to a bounded one. However, the bounded Lipschitz constant is quite large, which leads to small step sizes and slow practical convergence.}
 {Adam on $(L_0,L_1)$-smoothness with affine noise variance (for the special case of finite sum problems)  were in \citet{wang2022provable}. However, they only showed that Adam converges to the neighborhood of a stationary point with a constant learning rate.}
More details can be found in Table \ref{table:1}.
\begin{table}[h!]\small
\centering
\setlength{\tabcolsep}{1.0mm}{
\begin{tabular}{||c c c c c c c||} 
 \hline
 Method& Smoothness\footnotemark[1] &Algorithm&Convergence\footnotemark[2]&Assumption\footnotemark[3] &Batch size &Complexity \\ [0.3ex] 
 \hline\hline
 \citet{de2018convergence} & (LS) & RMSProp&\cmark & (BN)\footnotemark[4] & $\mathcal O(1)$&$\mathcal O({\epsilon}^{-4})$\\
 \citet{zaheer2018adaptive} & (LS)  & RMSProp & \cmark & (BN)&$\mathcal O({\epsilon}^{-2})$ &$\mathcal O({\epsilon}^{-4})$ \\ 
  \citet{shi2020rmsprop} & (LS)  & RMSProp & \cmark & (FSAN)&- &- \\ 
 \citet{defossez2020simple} & (LS)  & Adam&\cmark &(BN)&$\mathcal O(1)$ &$\mathcal O({\epsilon}^{-4})$\\
\citet{zou2019sufficient} & (LS)  & Adam&\cmark & (BSM)& $\mathcal O(1)$&$\mathcal O({\epsilon}^{-4})$\\
\citet{chen2018convergence} & (LS)  & Adam&\xmark &  (BN) &-&-\\
\citet{zhang2022adam}  & (LS) & Adam&\xmark & (FSAN) & -&-\\
\citet{wang2022provable} & (FSGS) & Adam&\xmark & (FSAN) & -&-\\
\citet{guo2021novel} & (LS) & Adam&\cmark & (AN)\footnotemark[5] & $\mathcal O(1)$&$\mathcal O({\epsilon}^{-4})$\\
\citet{hong2023high}  & (LS)  & Adam&\cmark & (CAN) & $\mathcal O(1)$&$ \tilde{\mathcal O}({\epsilon}^{-4})$\\
\citet{wang2023closing}  & (LS)  & Adam&\cmark & (CAN) & $\mathcal O(1)$&$\mathcal O({\epsilon}^{-4})$\\
\citet{li2023convergence}  & (GS) & Adam&\cmark & (SGN) & $\mathcal O(1)$&$\mathcal O({\epsilon}^{-4})\footnotemark[6]$\\
\citet{wang2024convergence}  & (GS) & Scalar Adam&\cmark & (AN) & $\mathcal O(1)$&$\mathcal O({\epsilon}^{-4})$\\
 Our method & (CWGS) & Adam &\cmark& (CAN)& $\mathcal O(1)$& $\mathcal O(\epsilon^{-4})$\\
 \hline
\end{tabular}}

\caption{{Comparison for existing} RMSProp and Adam analyses. 
For $\nabla f(\boldsymbol x)$ with its estimate $\boldsymbol g$, the bounded norm assumption is $\|\boldsymbol g\|\le G$ (almost surely), where $G$ is some positive constant. The bounded second-order moment assumption is that $\mathbb E[\|\boldsymbol g\|^2]\le G^2$. The bounded sub-Gaussian norm assumption  is that $\|\boldsymbol g-\nabla f(\boldsymbol 
 x)\|$ follows a sub-Gaussion distribution, which is weaker than the bounded norm assumption but stronger than the bounded variance assumption.  {The batch size refers to the number of samples necessary to compute the gradient estimate $\boldsymbol g$ and complexity denotes the total computational effort required to achieve an $\epsilon$-stationary point.} Explanation on the upper footmarks: $1:$ (LS) indicates the standard $L$-smoothness, (GS) denotes the generalized $(L_0, L_1)$-smoothness, (FSGS) denotes the finite sum $(L_0, L_1)$-smoothness and (CWGS) indicates the coordinate-wise $(L_0, L_1)$-smoothness. $2:$ \xmark \ indicates the algorithm only converges to the neighborhood of a stationary point, whose radius can not be made small. $3:$ (BN) indicates Bounded Norm, (FSAN) indicates Finite Sum Affine Noise, (BSM) indicates Bounded Second-order Moment, (AN) indicates Affine Noise, (CGN) indicates Sub-Gaussian Norm.
$4:$ \citet{de2018convergence} also requires the signs of the gradients to be the same across batches. $5:$ \citet{guo2021novel} also requires the upper bound on the gradient norm. $6:$ A variance-reduced method is also investigated in \citet{li2023convergence}, and the complexity is $\mathcal O(\epsilon^{-3})$.}
\label{table:1}
\end{table}

When preparing this work, we have observed a concurrent work by \citet{wang2024convergence}, which studies a scalar—or ”norm”—version of Adam. 
In this paper, we study the
per-coordinate version of Adam with the practical and challenging coordinate-wise $(L_0,L_1)$-smooth objectives (see Algorithm \ref{alg:1} for more details).

\section{Preliminaries}\label{sec:problemmodel}
Let  $f:\mathbb R^d \to \mathbb R$ be a differentiable non-convex loss function. For a positive integer $d$, let $[d]$ denote the set $\{1,2,..., d\}$. Let $\boldsymbol  x\in \mathbb R^d$ be an optimization variable.
Our goal is to minimize the objective function $f(x)$:
$$\min_{\boldsymbol x} f(\boldsymbol x).$$  For a differentiable function $f$, a point $\boldsymbol x\in \mathbb R^d$ is called a first-order $\epsilon$-stationary point if  $\|\nabla f(\boldsymbol x)\|\le \epsilon$. Denote by $\boldsymbol x_t\in \mathbb R^d$ the optimization variable at the $t$-th iteration and we have access to an estimate $\boldsymbol g_t$ of $\nabla f(\boldsymbol x_t)$. Define $\mathcal F_t:=\sigma(\boldsymbol g_1,...,\boldsymbol g_{t-1})$ is the sigma field of the stochastic gradients up to $t-1$. We focus on the Adam algorithm shown in Algorithm \ref{alg:1}, where $\odot$ denotes the Hadamard product. For any $i\in [d]$, $\partial_i f(\boldsymbol x_t), \boldsymbol g_{t,i}, \boldsymbol m_{t,i}$ and $\boldsymbol v_{t,i}$ are the $i$-th element of $\nabla f(\boldsymbol x_t), \boldsymbol g_{t}, \boldsymbol m_{t}$ and $\boldsymbol v_{t}$, respectively. 


The Adam algorithm is provided in Algorithm \ref{alg:1}. Compared with the original Adam, we make a minor change in the adaptive stepsize from  $\boldsymbol \eta_t=  \frac{\eta}{\sqrt{{\boldsymbol v_t} }+ \zeta }$ to  $ \frac{\eta}{\sqrt{{\boldsymbol v_t} + \zeta}}$.
This minor change does not influence the adaptivity of the algorithm but makes the analysis much easier. 

\subsection{Technical Assumptions}
Throughout this paper, we make the following assumptions.
\begin{assumption}\label{assump:lowerbound}
$f(\boldsymbol x)$ is bounded from below such that $\inf_{\boldsymbol x} f(\boldsymbol x)>-\infty$.
\end{assumption}
\begin{assumption}[Coordinate-wise affine noise variance \citep{wang2023closing,hong2023high}]\label{assump:variance}
 We have access to an unbiased gradient estimate $\boldsymbol g_t$ such that $\mathbb E[\boldsymbol  g_t|\mathcal F_t]=\nabla f(\boldsymbol x_t)$ and for any $i\in [d]$, $\mathbb E[\boldsymbol g_{t,i}^2|\mathcal F_t]\le D_0+D_1 (\partial_i f(\boldsymbol x_{t}))^2$, where $D_0,D_1\ge 0$ are some constants.
\end{assumption}

As discussed in \citet{hong2023high}, this assumption allows the magnitude of noise to scale with the corresponding gradient coordinate. Many widely used noise assumptions are special cases of this affine noise variance assumption. For example, when $D_1=0$, it is the bounded second-order moment assumption in \citet{zou2019sufficient} and when $D_1=1$, it is equivalent to coordinate-wise bounded gradient variance. However, as pointed out in \citet{wang2023convergence}, these two assumptions of bounded second-order moment and bounded gradient variance do not even hold for linear regression problems.  {For example, let $f(\omega)=\mathbb E_{z\sim \mathcal D}(\langle z,  \omega\rangle)^2= \omega^2$, where $z$ is a sample and $\mathcal D$ is a standard Gaussian distribution over $\mathbb R$. It can be shown that $g=2z^2\omega$ is an unbiased estimate of $\nabla f(\omega)$. However, both the variance and second-order moment of $g$ is in the order of $\mathcal O(\omega^2) $ which are unbounded when $\omega \to \infty$.  } When $D_0=0$, it is called the “strong growth condition” \citep{vaswani2019fast},  {which is shown to be reasonable for overparameterized neural networks that can interpolate all data points \citep{vaswani2019fast}.} Under Assumption \ref{assump:variance},  the norm of the gradient increases with the norm of the true gradient. This is important for model parameters that are multiplicatively perturbed by noise, e.g., multilayer network \citep{faw2022power}. In this paper, we study the coordinate-wise affine noise variance assumption, which was also used in \citet{hong2023high,wang2023closing}.

Though the $L$-smoothness assumption is widely used in optimization studies, recently it has been observed that in the training of neural networks,  {such as LSTMs \citep{zhang2019gradient}, ResNets\citep{zhang2019gradient} and Transformers \citep{crawshaw2022robustness},} this assumption does not hold. Instead, it is numerically verified that the following generalized smoothness assumption better models the training of neural networks \citep{zhang2019gradient}: $\|\nabla f(\boldsymbol x)-\nabla f(\boldsymbol y)\|\le (L_0+L_1\|\nabla f(\boldsymbol x)\|)\|\boldsymbol x-\boldsymbol y\|$ for some positive $L_0$ and $L_1$.  This assumption is widely studied in the literature, e.g., \citet{jin2021non,chen2023generalized,li2023convergence,wang2022provable,wang2023convergence,faw2023beyond}. Compared with $L$-smooth functions,  for the generalized smooth functions, the Lipschitz constant scales with the true gradient norm thus may not be bounded. For the training of Transformer models, \citet{crawshaw2022robustness} finds the following coordinate-wise $(L_0,L_1)$-smoothness, which provides a more accurate characterization of the objective. It is generalized from the coordinate-wise $L$-smoothness \citep{richtarik2014iteration,khaled2020better,bernstein2018signsgd}. In this paper, we focus on coordinate-wise $(L_0,L_1)$-smooth functions as defined below:
\begin{assumption}[Coordinate-wise $(L_0,L_1)$-smoothness]\label{assump:generalsmooth}
A function $f$ is coordinate-wise $(L_0,L_1)$-smooth if for any $\boldsymbol x,\boldsymbol y\in \mathbb R^d$ and $i\in[d]$,
\begin{flalign}
    |\partial_i f(\boldsymbol x)-\partial_i f(\boldsymbol y)|\le \left(\frac{L_0}{\sqrt{d}}+L_1|\partial_i f(\boldsymbol x)|\right)\|\boldsymbol x-\boldsymbol y\|.
\end{flalign}
\end{assumption}
The training of Adam enjoys adaptive learning rates for each parameter individually due to the coordinate-wise update process.  Moreover, extensive experiments in \citet{crawshaw2022robustness} show that the $L_0$ and $L_1$ for each coordinate vary a lot. Therefore, it is more accurate to leverage the coordinate-wise  $(L_0,L_1)$-smoothness.  In this paper, for the sake of simplicity and coherence with the coordinate-wise affine noise variance assumption, we assume the $L_0$ and $L_1$ to be identical for each coordinate. Our results can be easily adapted to the case with distinct $L_0$ and $L_1$ for different coordinates.

\subsection{Challenges and Insights }

Our theoretical analyses address several major challenges: (i) dependence between stepsize and gradient,  (ii) potential unbounded gradients, (iii) mismatch between gradient and first-order momentum, and (iv) additional bias terms due to affine variance and coordinate-wise $(L_0,L_1)$-smoothness. Prior research circumvented most of these challenges by introducing extra assumptions, whereas we provide several new insights and show that these assumptions may not be needed.

\begin{algorithm}
\caption{Adam}
\begin{algorithmic}
  \State Initialize parameters: $\boldsymbol x_1$, learning rates $\eta$, $\beta_1,\beta_2$, $\zeta$, Iteration $T$
  \State Initialize first and second moment estimates:  $\boldsymbol v_0 \in \mathbb R^{+},\boldsymbol m_0=0$
  \State Initialize time step: $t = 1$

  \While{$t\le T$}
    
    \State Generate stochastic gradient: $\boldsymbol g_t $
    \State Update first-order momentum estimate: $\boldsymbol m_t \gets \beta_1  \boldsymbol m_{t-1} + (1 - \beta_1)  \boldsymbol g_t$
    \State Update second-order momentum estimate: $\boldsymbol v_t \gets \beta_2  \boldsymbol v_{t-1} + (1 - \beta_2) \boldsymbol g_t \odot \boldsymbol g_t$
    \State Update parameters: $\boldsymbol x_{t+1} \gets\boldsymbol  x_{t} - \eta  \frac{1}{\sqrt{{\boldsymbol v_t} + \zeta }} \odot \boldsymbol m_t$
    \State $t \gets t + 1$
  \EndWhile
\end{algorithmic}\label{alg:1}
\end{algorithm}
In this paper, we have an access to an unbiased estimate $\boldsymbol g$ of $\nabla f(\boldsymbol x)$ such that $\mathbb E[\boldsymbol g| \boldsymbol x]=\nabla f(\boldsymbol x)$.
Consider the Adam algorithm in Algorithm \ref{alg:1}, which reduces to RMSProp if $\beta_1=0$.
For coordinate-wise $(L_0,L_1)$-smooth objective functions, we have the following descent inequality (Lemma 1 in \citet{crawshaw2022robustness}):
\begin{flalign}
       &\underbrace{\mathbb E\left[\left\langle \nabla f(\boldsymbol x_t), \boldsymbol x_t-\boldsymbol x_{t+1} \right\rangle|\mathcal F_t\right]}_{\text{first-order}}
    \le   f(\boldsymbol x_t)-\mathbb E[f(\boldsymbol x_{t+1})|\mathcal F_t]\nonumber\\
&+\underbrace{\sum_{i=1}^d\frac{L_0}{2\sqrt{d}}\mathbb E[\|\boldsymbol x_{t+1}-\boldsymbol x_t\||\boldsymbol x_{t+1,i}-\boldsymbol x_{t,i}||\mathcal F_t]}_{\text{second-order}}  \label{eq:generalsmooth}  +\underbrace{\sum_{i=1}^d\frac{L_1|\partial_i f(\boldsymbol x_t)|}{2}\mathbb E[\|\boldsymbol x_{t+1}-\boldsymbol x_t\||\boldsymbol x_{t+1,i}-\boldsymbol x_{t,i}||\mathcal F_t]}_{\text{additional term}},
\end{flalign}
where the last term is the additional term due to the $(L_0,L_1)$-smooth assumption.

\textbf{Challenge 1: dependence between stepsize and gradient.} 
We use the RMSProp optimizer to explain our technical novelty. The challenge is the same for Adam. For RMSProp the optimized parameter $\boldsymbol x$ is updated: $\boldsymbol x_{t+1}=\boldsymbol x_t-\frac{\eta \boldsymbol g_t}{\sqrt{\boldsymbol v_t}+\zeta}$, where the adaptive stepsize $\boldsymbol \eta_t=\frac{\eta}{\sqrt{\boldsymbol v_t}+\zeta}$ depends on the current gradient estimate $\boldsymbol g_t$, which makes it hard to bound the conditional expectation of the first-order term in \eqref{eq:generalsmooth}. To address this challenge, studies on Adagrad \citep{ward2020adagrad,defossez2020simple,faw2022power} and studies on RMSProp \citep{zaheer2018adaptive} propose a surrogate $\tilde{\boldsymbol v}_t$ of $\boldsymbol v_t$ which is independent of $\boldsymbol g_t$, and then the first-order term is divided into two parts: the first-order.a term $\mathbb E\left[\left\langle \nabla f(\boldsymbol x_t), \frac{-\eta \boldsymbol g_t}{\sqrt{\tilde{\boldsymbol v}_t}+\zeta} \right\rangle\Big |\mathcal F_t\right]$ and first-order.b term $\mathbb E\left[\left\langle \nabla f(\boldsymbol x_t), \frac{-\eta \boldsymbol g_t}{\sqrt{\boldsymbol v_t}+\zeta} +\frac{\eta \boldsymbol g_t}{\sqrt{\tilde{\boldsymbol v}_t}+\zeta}\right\rangle\Big |\mathcal F_t\right]$.
The main challenge lies in the first-order.b term (surrogate error). In \citet{zaheer2018adaptive}, this term is bounded based on the assumption of bounded gradient norm, which does not hold in this paper.

\textit{Insight 1: reduce surrogate error.} { We choose the same surrogate $\tilde{\boldsymbol v}_t=\beta_2 \boldsymbol v_{t-1}$ as the one in \citet{zaheer2018adaptive}, which requires that the gradient norm is bounded. To remove this assumption, we change the adaptive stepsize from $\frac{\eta}{\sqrt{\boldsymbol v_t}+\zeta}$ to $\frac{\eta}{\sqrt{\boldsymbol v_t+\zeta}}$ (details seen in Remark \ref{remark:1}). For Adagrag \citep{wang2023convergence}, due to the non-increasing stepsize, in the $t$-th iteration, the error term can be written as $\mathbb E [\phi(\boldsymbol v_{t-1})-\phi(\boldsymbol v_{t})|\mathcal F_t]$ where $\phi $ is some function. Thus, these terms cancel out with each other by taking the sum from $t=1$ to $T$. However, this method does not work for Adam due to the employment of second-order momentum. With our modified adaptive stepsize, we can also show the sum of error terms is bounded (details can be found in Lemma \ref{lemma:1}). We notice that a different surrogate is selected in \citet{wang2023closing}, which aims to bound the surrogate error for $L$-smooth objectives while our surrogate is tailored to the specific requirements of our generalized smooth objectives.}

\textbf{Challenge 2: unbounded gradient.}  
Previous works, e.g., \citet{de2018convergence,defossez2020simple,zaheer2018adaptive} 
assumed that the gradient norm is bounded, based on which, they proved the convergence. 
However, with affine gradient noise, the gradient may be infinite, and thus those approaches do not apply.

\textit{Insight 2: recursive bounding technique.} 
{For RMSProp, 
with bounded surrogate error in Lemma \ref{lemma:1}, we first show that $\mathbb E\left[\frac{1}{T}\sum_{t=1}^T\frac{\|\nabla f(\boldsymbol x_t)\|^2}{{\sqrt{\|\tilde{\boldsymbol v}_t\|+\zeta}}}\right]$ is of the order of $\mathcal{O}{(\epsilon^2)}$. 
If the gradient norm is upper bounded, then $\tilde{\boldsymbol v}_t$ is bounded, and the convergence result directly follows. However, under affine gradient noise, the gradient norm may not be bounded. For generalized smooth objectives, we develop a novel approach to bound $\mathbb E\left[\frac{1}{T}\sum_{t=1}^T \sqrt{\|\tilde{\boldsymbol v}_t\|+\zeta}\right]$ using $\mathbb E\left[{\sum_{t=1}^T\|\nabla f(\boldsymbol x_t)\|}\right]$  instead of a constant (see Lemma \ref{lemma:2} for details). 
Applying H\"older’s inequality \citep{hardy1952inequalities} we will obtain the convergence result.
The complexity result matches with the lower bound in \citet{arjevani2023lower}. A similar method is applied in \citet{wang2023closing}, which focuses on the $L$-smooth objectives and bound $\mathbb E\left[{\sum_{t=1}^T\|\nabla f(\boldsymbol x_t)\|}\right]$ by a constant.}

\textbf{Challenge 3: mismatch between gradient and first-order momentum.} 
Compared with RMSProp, Adam employs the first-order momentum $\boldsymbol m_t$. The momentum $\boldsymbol m_t$ is dependent on the surrogate stepsize $\frac{\eta}{\sqrt{\tilde{\boldsymbol v}_t+\zeta}}$. Moreover, the momentum $\boldsymbol m_t$ is a biased estimate of the current true gradient. Both the above challenges make it hard to theoretically characterize the convergence rate of Adam. {These mismatch challenges also accur in the analysis for SGDM \citep{liu2020improved} and Adam \citep{wang2023closing},  where a potential function of $f(\boldsymbol u_t)$ with $\boldsymbol u_t=\frac{\boldsymbol x_t-\boldsymbol x_{t-1}{\beta_1}/{\sqrt{\beta_2}}}{1-{\beta_1}/{\sqrt{\beta_2}}}$ is studied. It can be shown that $\boldsymbol u_{t+1}-\boldsymbol u_t$ is close to a function of $\frac{\boldsymbol g_t}{\sqrt{\tilde{\boldsymbol v}_t+\zeta}}$, which is much easier to analyze compared with $\frac{\boldsymbol m_t}{\sqrt{\tilde{\boldsymbol v}_t+\zeta}}$. However, both of them are limited to $L$-smooth objectives.}

\textit{Insight 3:  bounding first-order term using $\mathbb E\big[\frac{1}{T}\sum_{t=1}^T\|\nabla f(\boldsymbol x_t)\|\big]$.} {In this paper, we choose the same potential function but different surrogate in \citet{wang2023closing}. Using the descent lemma of $ f(\boldsymbol{u}_t)$, we show that the first-order term is also bounded by a $\epsilon^2$-level constant plus a function of $\mathbb E\left[\frac{1}{T}\sum_{t=1}^T\|\nabla f(\boldsymbol x_t)\|\right]$. Compared to RMSProp, this additional function is introduced due to the bias of $\boldsymbol m_t$. Then via  H\"older’s inequality, we show Adam converges as fast as RMSProp.  } 

\textbf{Challenge 4: additional error terms due to affine variance and $(L_0,L_1)$-smoothness.} Compared with $L$-smooth objectives, in the analysis for RMSProp with $(L_0,L_1)$-smooth objectives, there is an additional second-order error term: $\sum_{i=1}^d\frac{L_1|\partial_i f(\boldsymbol x_t)|}{2}[\|\boldsymbol x_{t+1}-\boldsymbol x_t\||\boldsymbol x_{t+1,i}-\boldsymbol x_{t,i}|]$, which is hard to bound since $|\partial_i f(\boldsymbol x_t)|$ may be unbounded. Moreover, for RMSProp, since $\mathbb E[|\boldsymbol x_{t+1,i}-\boldsymbol x_{t,i}||\mathcal F_t]\le\mathbb E\left[\frac{\eta |\boldsymbol g_{t,i}|}{\sqrt{\tilde{\boldsymbol v}_{t,i}+\zeta}}\Big |\mathcal F_t\right]$ and $\boldsymbol g_{t,i}$ is independent of $\tilde{\boldsymbol v}_{t,i}$ given $\mathcal F_t$, the affine noise variance assumption can be leveraged to bound the second-order error term directly. Nevertheless, for Adam due to the dependence between  $\boldsymbol m_{t,i}$ and $\tilde{\boldsymbol v}_{t,i}$, the above approach cannot be applied directly.

\textit{Insight 4: bounding additional term by function of first-order term.} 
For RMSProp, we can show that $\|\boldsymbol x_{t+1}-\boldsymbol x_t\|\le \frac{\eta\sqrt{d}}{\sqrt{1-\beta_2}}$ and $|\partial_i f(\boldsymbol x_t)||\boldsymbol x_{t+1,i}-\boldsymbol x_{t,i}|\le \frac{|\partial_i f(\boldsymbol x_{t})|^2+\eta^2\boldsymbol g_{t,i}^2}{2\sqrt{\tilde{\boldsymbol v}_{t,i}+\zeta}}$. According to the affine noise assumption, we have that $\mathbb E[\boldsymbol g_{t,i}^2|\mathcal F_t]$ is upper bounded by a linear function of $(\partial_i f(\boldsymbol x_t))^2$, thus we can bound the additional error term. However, we cannot directly apply this method to Adam since $\mathbb E[\boldsymbol m_{t,i}^2|\mathcal F_t]$ is hard to bound. Instead, we provide a new upper bound on $\sum_{t=1}^T\frac{\boldsymbol m_{t,i}^2}{\sqrt{\boldsymbol v_{t,i}}}$ using the gradient norm (details can be found in Lemma \ref{lemma:6}). 

\section{Convergence Analysis of RMSProp}\label{sec:lsmmoth}
To provide a fundamental understanding of Adam, in this section, we focus on RMSProp which consists of the design of adaptive learning rates for each individual parameter in Adam. 

For RMSProp, the main challenges come from the dependence between stepsize and gradient, potential unbounded gradients due to the affine noise variance, and additional error terms due to $(L_0,L_1)$-smoothness.
The analysis can be extended to general Adam, which requires additional efforts to handle the first-order momentum. Define $c=\sqrt{\zeta}+d\sqrt{D_0+\|\boldsymbol v_0\|}$. We then present our results of RMSProp in the following theorem.
\begin{theorem}[Informal]\label{theorem:1}
     Let Assumptions \ref{assump:lowerbound}, \ref{assump:variance} and \ref{assump:generalsmooth} hold. Let $1-\beta_2= \mathcal O(\epsilon^{2})$, $\eta = \mathcal O(\epsilon^{2})$,  and $T= \mathcal O (\epsilon^{-4})$. For $\epsilon$ such that $\epsilon\le \frac{\sqrt{5dD_0}}{\sqrt{D_1}\sqrt[4]{\zeta}}$, we have that
 \begin{flalign}
     \frac{1}{T} \sum_{t=1}^T  \mathbb E[\|\nabla f(\boldsymbol x_t)\|]\le \left(\frac{2d\sqrt{35D_0D_1}}{\sqrt[4]{\zeta}}+\sqrt{c}\right)\epsilon.
\end{flalign}
\end{theorem}
To the best of our knowledge,  this paper provides the first convergence analysis of RMSProp for $(L_0,L_1)$-smooth functions with affine noise variances. Existing studies mostly assume bounded gradient norm or variance \citep{de2018convergence,zaheer2018adaptive} or only show the algorithm converges asymptotically \citep{shi2020rmsprop}. More importantly, our result matches the lower bound in \citet{arjevani2023lower}, and thus is optimal. 

The formal version of the theorem and the detailed proof can be found in Appendix \ref{proof:theorem1}.

Below, we provide a proof sketch to highlight our major technical novelties.
\begin{proof}[Proof sketch]   
Our proof can be divided into three stages: Stage \uppercase\expandafter{\romannumeral1}: develop an upper bound of $\mathbb E\left[\frac{ \|\nabla f(\boldsymbol x_t)\|^2}{\sqrt{{\beta_2 \|\boldsymbol v_{t-1}} \|+ \zeta}}\right]$;  Stage   \uppercase\expandafter{\romannumeral2}: develop an upper bound on $\mathbb E[\sqrt{{\beta_2 \|\boldsymbol v_{t-1}}\| + \zeta}]$; and Stage   \uppercase\expandafter{\romannumeral3}: 
show $\mathbb E\left[{ \|\nabla f(\boldsymbol x_t)\|}\right]$ converges using results from Stages \uppercase\expandafter{\romannumeral1}, \uppercase\expandafter{\romannumeral2} and H\"older's inequality.

\textbf{Stage I:} upper bound of $\mathbb E\left[\frac{ \|\nabla f(\boldsymbol x_t)\|^2}{\sqrt{{\beta_2 \|\boldsymbol v_{t-1}} \|+ \zeta}}\right]$. As discussed in Section \ref{sec:intro}, for coordinate-wise $(L_0,L_1)$-smooth functions,  following the descent inequality (Lemma 1 in \citet{crawshaw2022robustness}) we can get (\ref{eq:generalsmooth}). 
We first obtain 
$\frac{ \|\nabla f(\boldsymbol x_t)\|^2}{\sqrt{{\beta_2 \|\boldsymbol v_{t-1}} \|+ \zeta}}\le \sum_{i=1}^d\frac{\eta (\partial_i f(\boldsymbol x_{t}))^2}{\sqrt{{\beta_2 \boldsymbol v_{t-1,i}} + \zeta}}$. Therefore, in the following we will bound $\mathbb E\left[\sum_{i=1}^d\frac{\eta (\partial_i f(\boldsymbol x_{t}))^2}{\sqrt{{\beta_2 \boldsymbol v_{t-1,i}} + \zeta}}\right]$.
Towards this goal, in Step 1.1, we will show the LHS of \eqref{eq:generalsmooth} is lower bounded by a function of $\sum_{i=1}^d\frac{\eta (\partial_i f(\boldsymbol x_{t}))^2}{\sqrt{{\beta_2 \boldsymbol v_{t-1,i}} + \zeta}}$; and in Step 1.2 we will show the RHS of \eqref{eq:generalsmooth} is upper bounded by a function of $\sum_{i=1}^d\frac{\eta (\partial_i f(\boldsymbol x_{t}))^2}{\sqrt{{\beta_2 \boldsymbol v_{t-1,i}} + \zeta}}$. Combining the two steps, we will obtain an upper bound on $\sum_{i=1}^d\frac{\eta (\partial_i f(\boldsymbol x_{t}))^2}{\sqrt{{\beta_2 \boldsymbol v_{t-1,i}} + \zeta}}$.

\textit{Step 1.1: lower bound on the first-order term in \eqref{eq:generalsmooth}.}
Since the adaptive stepsize and the gradient estimate are dependent, we design a surrogate $\tilde{\boldsymbol v}_t=\beta_2 \boldsymbol v_{t-1}$ to decompose the first-order term in \eqref{eq:generalsmooth} into two parts:
\begin{flalign}\label{eq:descentproof}
    &\underbrace{\mathbb E[\langle \nabla f(\boldsymbol x_t), \boldsymbol x_{t}-\boldsymbol x_{t+1} \rangle|\mathcal F_t]}_{\text{first-order}}\nonumber\\
    &=\underbrace{\mathbb E\left[\left\langle \nabla f(\boldsymbol x_t), \frac{\eta \boldsymbol g_t}{\sqrt{\tilde{\boldsymbol v}_t+\zeta}} \right\rangle\Big |\mathcal F_t\right]}_{\text{first-order.a}} + \underbrace{\mathbb E\left[\left\langle \nabla f(\boldsymbol x_t), \frac{\eta \boldsymbol g_t}{\sqrt{\boldsymbol v_t+\zeta}} -\frac{\eta \boldsymbol g_t}{\sqrt{\tilde{\boldsymbol v}_t+\zeta}}\right\rangle\Big |\mathcal F_t\right]}_{\text{first-order.b}}.
\end{flalign}
For the first-order.a term in \eqref{eq:descentproof}, we can show that  $$\mathbb E\left[\left\langle \nabla f(\boldsymbol x_t), \frac{\eta \boldsymbol g_t}{\sqrt{\tilde{\boldsymbol v}_t+\zeta}} \right\rangle|\mathcal F_t\right]=\sum_{i=1}^d\frac{\eta (\partial_i f(\boldsymbol x_{t}))^2}{\sqrt{{\beta_2 \boldsymbol v_{t-1,i}} + \zeta}}.$$ 

We then bound the first-order.b term in \eqref{eq:descentproof}.
\begin{remark}[Importance of modified adaptive stepsize $\boldsymbol \eta$]\label{remark:1}
    For the original RMSProp, \citet{zaheer2018adaptive} chose the same surrogate $\tilde{\boldsymbol v}_t=\beta_2 \boldsymbol v_{ t-1}$ as in this paper. 
Then the first-order.b term was lower bounded by a function of $\sum_{i=1}^d \mathbb E\left[\frac{-\boldsymbol g_{t,i}^2}{\sqrt{{\beta_2 \boldsymbol v_{t-1,i}}}+\zeta}\Big |\mathcal F_t\right]\times\mathbb E\left[\frac{(1-\beta_2)\boldsymbol g_{t,i}^2}{(\sqrt{{ \boldsymbol v_{t,i}}}+\sqrt{{\beta_2 \boldsymbol v_{t-1,i}}})^2}\Big |\mathcal F_t\right] $. 
Then they developed an upper bound on the second term $\mathbb E\left[\frac{(1-\beta_2)\boldsymbol g_{t,i}^2}{(\sqrt{{ \boldsymbol v_{t,i}}}+\sqrt{{\beta_2 \boldsymbol v_{t-1,i}}})^2}\Big |\mathcal F_t\right]\le 1$ which is quite loose, thus they introduced an additional assumption on the upper bound of $|\boldsymbol g_{t,i}|$ \citep{zaheer2018adaptive}. 

In contrast, using our adaptive stepsize  $ \eta  \frac{\boldsymbol g_t}{\sqrt{{\boldsymbol v_t} + \zeta}}$, we can show that the  first-order.b term can be lower bounded by a function of $\sum_{i=1}^d \mathbb E\left[\frac{-\boldsymbol g_{t,i}^2}{\sqrt{{\beta_2 \boldsymbol v_{t-1,i}}+\zeta}}\Big |\mathcal F_t\right]\times\mathbb E\left[\frac{(1-\beta_2)\boldsymbol g_{t,i}^2}{(\sqrt{{ \boldsymbol v_{t,i}+\zeta}}+\sqrt{{\beta_2 \boldsymbol v_{t-1,i}}+\zeta})^2}\Big |\mathcal F_t\right]$, which can be further bounded by $\sum_{i=1}^d \mathbb E\big[{-\boldsymbol g_{t,i}^2}|\mathcal F_t\big]\mathbb E\left[\frac{1}{\sqrt{{\beta_2 \boldsymbol v_{t-1,i}} + \zeta}}-\frac{1}{\sqrt{{ \boldsymbol v_{t,i}} + \zeta}}\Big |\mathcal F_t\right]$. Applying the affine noise variance assumption in Assumption \ref{assump:variance}, we obtain a lower bound on the first-order.b term in the following lemma.
\end{remark}

\begin{lemma}[Informal]\label{lemma:1}
Under Assumptions \ref{assump:variance} and \ref{assump:generalsmooth}, we have that
\begin{flalign}
    &\mathbb E\left[\left\langle \nabla f(\boldsymbol x_t), \frac{\eta \boldsymbol g_t}{\sqrt{\boldsymbol v_t+\zeta}} -\frac{\eta \boldsymbol g_t}{\sqrt{\beta_2 \boldsymbol {v}_{t-1}+\zeta}}\right\rangle\Big |\mathcal F_t\right] \nonumber\\
    \ge &-\sum_{i=1}^d\Bigg(\mathcal O\Bigg( \frac{\eta(\partial_i f(\boldsymbol x_{t}))^2}{\sqrt{\beta_2 \boldsymbol v_{t-1,i}+\zeta}}\Bigg) +\mathcal O\Bigg( \frac{\eta^2}{\sqrt{1-\beta_2}
    }\frac{(\partial_i f(\boldsymbol x_{t}))^2}{\sqrt{\beta_2 \boldsymbol v_{t-1,i}+\zeta}}\Bigg)\nonumber\\
    &+\mathcal O\Bigg(\frac{\eta}{\sqrt{\beta_2 \boldsymbol v_{t-1,i}+\zeta}}-\mathbb E\Bigg[{\frac{\eta}{\sqrt{\boldsymbol v_{t,i}+\zeta}}}\Big |\mathcal F_t\Bigg]\Bigg) +\mathcal O\left(\frac{\eta(\partial_i f(\boldsymbol x_{t-1}))^2}{\sqrt{{\beta_2 \boldsymbol v_{t-1,i}} + \zeta}}-\mathbb E\left[\frac{\eta(\partial_i  f(\boldsymbol x_{t}))^2}{\sqrt{{ \boldsymbol v_{t,i}} + \zeta}}\Big |\mathcal F_t\right]\right)\Bigg)\nonumber\\
    &-\text{small error}.\label{eq:lemma32}
\end{flalign}
\end{lemma}
The formal version and detailed proof of Lemma \ref{lemma:1} can be found in the Appendix \ref{proof:lemma1}, which is based on our modified update process on $\boldsymbol v_t$, H\"older's inequality and Assumption \ref{assump:generalsmooth}. 

In the RHS of \eqref{eq:lemma32}, consider the term $\mathbb E\left[\frac{(\partial_i f(\boldsymbol x_{t-1}))^2}{\sqrt{{\beta_2 \boldsymbol v_{t-1,i}} + \zeta}}-\frac{(\partial_i f(\boldsymbol x_{t}))^2}{\sqrt{{ \boldsymbol v_{t,i}} + \zeta}}\Big |\mathcal F_t\right].$
Taking sum from $t=1$ to $T$, 
the terms $\frac{(\partial_i f(\boldsymbol x_{t-1}))^2}{\sqrt{{\beta_2 \boldsymbol v_{t-1,i}} + \zeta}}$ and  $\frac{(\partial_i f(\boldsymbol x_{t-1}))^2}{\sqrt{{ \boldsymbol v_{t-1,i}} + \zeta}}$ shall be close to each other as $\beta_2 \to 1$. Similarly, $\frac{\eta}{\sqrt{\beta_2 \boldsymbol v_{t-1,i}+\zeta}}-{\frac{\eta}{\sqrt{\boldsymbol v_{t-1,i}+\zeta}}}$ can also be bounded.

\textit{Step 1.2: upper bound on second-order and additional terms in \eqref{eq:generalsmooth}.} We first focus on the second-order term. Based on the update process of $\boldsymbol v_t$ and Assumption \ref{assump:variance}, we get that 
\begin{flalign}\label{eq:second}
&{\sum_{i=1}^d\frac{L_0}{2\sqrt{d}}\mathbb E[\|\boldsymbol x_{t+1}-\boldsymbol x_t\||\boldsymbol x_{t+1,i}-\boldsymbol x_{t,i}||\mathcal F_t]}
    \le \sum_{i=1}^d \frac{L_0\eta^2}{2\sqrt{\zeta}} \frac{D_0+D_1(\partial_i f(\boldsymbol x_t))^2}{\sqrt{\beta_2 \boldsymbol v_{t-1,i}+\zeta}}.
\end{flalign}
We then focus on the additional term in \eqref{eq:generalsmooth} and we provide a new upper bound using $\sum_{i=1}^d\frac{ (\partial_i f(\boldsymbol x_t))^2}{\sqrt{{\beta_2 \boldsymbol v_{t-1,i}} + \zeta}} $ with some small errors: for any $\alpha_2>0$, we have that
\begin{flalign} \label{eq:addtional}
    &\frac{L_1|\partial_i f (\boldsymbol x_t)|}{2}\mathbb E[\|\boldsymbol x_{t+1}-\boldsymbol x_t\||\boldsymbol x_{t+1,i}-\boldsymbol x_{t,i}||\mathcal F_t]\nonumber\\
        & \le \frac{\eta^2\sqrt{d}}{2\sqrt{1-\beta_2}}\left(L_1\sqrt{D_1}+\frac{1}{\alpha_2\sqrt{1-\beta_2}}\right)\frac{(\partial_i f(\boldsymbol x_t))^2}{\sqrt{\beta_2 \boldsymbol v_{t-1,i}+\zeta}} +  \frac{\eta^2\sqrt{d}\alpha_2L_1^2D_0}{2{}\sqrt{\beta_2 \boldsymbol v_{t-1,i}+\zeta}}.
\end{flalign}
Plugging Lemma \ref{lemma:1}, (\ref{eq:second}) and (\ref{eq:addtional}) in (\ref{eq:generalsmooth}) we get the following lemma.
\begin{lemma} \label{co:1}
    Let Assumptions \ref{assump:lowerbound}, \ref{assump:variance} and \ref{assump:generalsmooth} hold. With the same parameters in Theorem \ref{theorem:1}, we have that
    \begin{flalign}
        \frac{1}{T} \sum_{t=1}^T \mathbb E\left[\frac{ \|\nabla f(\boldsymbol x_t)\|^2}{\sqrt{{\beta_2 \|\boldsymbol v_{t-1}}\| + \zeta}} \right]\le \epsilon^2.
    \end{flalign}
\end{lemma}
The details of the proof can be found in the Appendix \ref{proof:co1}.
The major idea in the proof of Lemma \ref{co:1} is to bound the first-order.b, the second-order term and the additional term using $\sum_{i=1}^d\mathbb E\left[\frac{ (\partial_i f(\boldsymbol x_t))^2}{\sqrt{{\beta_2 \boldsymbol v_{t-1,i}} + \zeta}} \right]$. 

\textbf{Stage II:} upper bound of $\mathbb E[\sqrt{{\beta_2 \|\boldsymbol v_{t-1}}\| + \zeta}]$. With the bounded gradient norm assumption, $\frac{1}{T} \sum_{t=1}^T \mathbb E[{ \|\nabla f(\boldsymbol x_t)\|^2}] = \mathcal O(\epsilon^2) $ follows directly from Lemma \ref{co:1}. However, under the affine gradient noise assumption in this paper, the gradient norm may not be bounded. 
Here, we establish a key observation  that $\frac{1}{T} \sum_{t=1}^T  \mathbb E[\sqrt{{\beta_2 \|\boldsymbol v_{t-1}}\| + \zeta}]$ can be bounded using $\frac{1}{T} \sum_{t=1}^T  \mathbb E[\|\nabla f(\boldsymbol x_t)\|]$. By H\"older's inequality, we have 
\begin{flalign}\label{eq:stage2}
    &\left(\frac{1}{T} \sum_{t=1}^T  \mathbb E[\|\nabla f(\boldsymbol x_t)\|]\right)^2
    \le \left(\frac{1}{T} \sum_{t=1}^T  \mathbb E[\sqrt{{\beta_2 \|\boldsymbol v_{t-1}}\| + \zeta}]\right) \times\left(\frac{1}{T} \sum_{t=1}^T \mathbb E\left[\frac{ \|\nabla f(\boldsymbol x_t)\|^2}{\sqrt{{\beta_2\| \boldsymbol v_{t-1}}\| + \zeta}} \right]\right)
    .
\end{flalign}
It is worth noting that in the RHS of (\ref{eq:stage2}), the second term is bound in  Lemma \ref{co:1}. If the first term is upper bounded using  $\frac{1}{T} \sum_{t=1}^T  \mathbb E[\|\nabla f(\boldsymbol x_t)\|]$, we then can prove an upper bound on  $\frac{1}{T} \sum_{t=1}^T  \mathbb E[\|\nabla f(\boldsymbol x_t)\|]$, and show the algorithm converges to a stationary point. In the following lemma, we show a novel bound on $\frac{1}{T} \sum_{t=1}^T  \mathbb E\left[\sqrt{{\beta_2 \|\boldsymbol v_{t-1}}\| + \zeta}\right]$.
\begin{lemma}\label{lemma:2}
    Let Assumption \ref{assump:variance} hold. Then, we have
    \begin{flalign}
        &\frac{1}{T} \sum_{t=1}^T  \mathbb E\left[\sqrt{{\beta_2 \|\boldsymbol v_{t-1}}\| + \zeta}\right]
        \le  c
   +\frac{2\sqrt{dD_1}}{\sqrt{(1-\beta_2)}} \frac{\sum_{t=1}^T\mathbb E[\|\nabla f(\boldsymbol x_t)\|]}{T}.
    \end{flalign}
\end{lemma}
The detailed proof can be found in the Appendix \ref{proof:lemma2}, where we recursively apply Jensen’s inequality \citep{10.1007/BF02418571}. The proof only depends on the affine noise variance and the update process on $v_t$. Thus, it works for both RMSProp and Adam with $(L_0,L_1)$-smooth objectives. 

\textbf{Stage \uppercase\expandafter{\romannumeral3}:} upper bound of $\mathbb E\left[{ \|\nabla f(\boldsymbol x_t)\|}\right]$.
Now we show that the algorithm converges to a stationary point. Define $e=\frac{1}{T} \sum_{t=1}^T  \mathbb E[\|\nabla f(\boldsymbol x_t)\|]$. By (\ref{eq:stage2}), Lemma \ref{co:1} and Lemma \ref{lemma:2} we have that
\begin{flalign}
    e^2\le \epsilon^2\left(c+\frac{2\sqrt{dD_1}}{\sqrt{1-\beta_2}}e\right).
\end{flalign}
Thus if $\epsilon\le \frac{\sqrt{5dD_0}}{\sqrt{D_1}\sqrt[4]{\zeta}}$, we have $$e\le\left(\frac{2d\sqrt{35D_0D_1}}{\sqrt[4]{\zeta}}+\sqrt{c}\right)\epsilon,$$ which indicates the algorithm converges to a stationary point.
\end{proof}
\section{Convergence Analysis of Adam}
In this section, we extend our convergence analysis of RMSProp to Adam. Such a result is attractive since empirical results on complicated tasks show Adam may perform better,  e.g., the mean reward is improved by $88\%$ via RMSProp and $110\%$ via Adam for Atari games \citep{agarwal2020optimistic}. 

We present the convergence result for Adam as follows.
\begin{theorem}[Informal]\label{theorem:2}
     Let Assumptions \ref{assump:lowerbound}, \ref{assump:variance} and \ref{assump:generalsmooth} hold. Let  $1-\beta_2= \mathcal O(\epsilon^{2})$,  $0<\beta_1\le \sqrt{\beta_2}<1$, $\eta = \mathcal O(\epsilon^{2})$,  and $T= \mathcal O(\epsilon^{-4})$. For small $\epsilon$ such that $\epsilon\le \frac{\sqrt{2C_2}}{\sqrt{7\alpha_0C_6D_1}}$, we have that
 \begin{flalign}
     \frac{1}{T} \sum_{t=1}^T  \mathbb E[\|\nabla f(\boldsymbol x_t)\|]\le \left(2c+\sqrt{2c}+\frac{4\sqrt{dD_1}}{\sqrt{C_6}}\right)\epsilon,
\end{flalign}
where $C_6$ is a positive constant defined in Appendix \ref{proof:theorem2}.
\end{theorem}
To the best of our knowledge, Theorem \ref{theorem:2} is the first convergence result of Adam to a stationary point under some of the most relaxed assumptions of  $(L_0,L_1)$-smoothness and affine gradient noise. In \citet{li2023convergence,wang2023closing}, the authors show that the Adam converges to a stationary point with affine noise variance. However, their methods only work for $L$-smooth objectives, while in this paper we focus on the $(L_0,L_1)$-smooth functions. Moreover, we show that for Adam, the overall computational complexity matches with the lower bound in \citet{arjevani2023lower}, while there is an additional 
logarithmic term in \citet{li2023convergence}. The normalized momentum algorithm \citep{jin2021non} can also be viewed as a special case of the Adam family, which applies the momentum on the first-order gradient. However, in their algorithm, a mini-batch of samples is required in each training iteration, while we do not require such a mini-batch. Thus, in the distributed setting with heterogeneous data, where the data distributions under each computational node are different, Algorithm \ref{alg:1} can be used directly. However, the normalized momentum in \citet{jin2021non} may require gradient information from many computational nodes, making the problem degrade to the centralized setting.  

The formal version of the theorem and detailed proof can be found in Appendix \ref{proof:theorem2}. Below, we provide a proof sketch to underscore our key technical novelties. 
\begin{proof}[Proof Sketch]
Similar to the proof of RMSProp, we divided our proof into three stages. The key difference lies in  Stage I, which is because of the dependence between $\boldsymbol m_t$ and $\tilde{\boldsymbol v}_t$.

\textbf{Stage I:} upper bound of $\mathbb E\left[\frac{ \|\nabla f(\boldsymbol x_t)\|^2}{\sqrt{{\beta_2 \|\boldsymbol v_{t-1}} \|+ \zeta}}\right]$. 
For the Adam optimizer, the first-order.a term  $\mathbb E\left[\left\langle \nabla f(\boldsymbol x_t), \frac{-\eta \boldsymbol m_t}{\sqrt{\tilde{\boldsymbol v}_t+\zeta}} \right\rangle\Big |\mathcal F_t\right]$ is challenging to bound due to the dependence between $\boldsymbol m_t$ and $\tilde{\boldsymbol v}_t$. Following the recent analyses of SGDM \citep{liu2020improved} and Adam \citep{wang2023closing}, we study a potential function $f(\boldsymbol u_t)$ with $\boldsymbol u_t=\frac{\boldsymbol x_t-\frac{\beta_1}{\sqrt{\beta_2}}\boldsymbol x_{t-1}}{1-\frac{\beta_1}{\sqrt{\beta_2}}}$. The benefit of analyzing $\boldsymbol u_t$ is that we have $\mathbb E\left[\left\langle \nabla f(\boldsymbol x_t), \boldsymbol u_t-\boldsymbol u_{t+1} \right\rangle|\mathcal F_t\right]\approx \frac{(1-\beta_1)}{1-\frac{\beta_1}{\sqrt{\beta_2}}}\mathbb E\left[\left\langle \nabla f(\boldsymbol x_t),\frac{\eta \boldsymbol g_t}{\sqrt{\tilde{\boldsymbol v}_t+\zeta}}\right\rangle\Big |\mathcal F_t\right]$, where the numerator and denominator in the last term are independent.

In Step 1.1, we will show that the LHS of the descent lemma for $f(\boldsymbol u_t)$ (shown in \eqref{eq:7eq1}) is lower bounded by a function of $\sum_{i=1}^d\frac{\eta (\partial_i f(\boldsymbol x_{t}))^2}{\sqrt{{\beta_2 \boldsymbol v_{t-1,i}} + \zeta}}$; and in Step 1.2 we will show the RHS of the descent lemma of $f(\boldsymbol u_t)$ is upper bounded by a function of $\sum_{i=1}^d\frac{\eta (\partial_i f(\boldsymbol x_{t}))^2}{\sqrt{{\beta_2 \boldsymbol v_{t-1,i}} + \zeta}}$. Combining the two steps, we will obtain an upper bound on $\sum_{i=1}^d\frac{\eta (\partial_i f(\boldsymbol x_{t}))^2}{\sqrt{{\beta_2 \boldsymbol v_{t-1,i}} + \zeta}}$.

\textit{Step $1.1$: lower bound of the first-order term ${\mathbb E[\langle \nabla f(\boldsymbol u_t), \boldsymbol u_{t}-\boldsymbol u_{t+1} \rangle|\mathcal F_t]}$.} We first divide the first-order term into two parts: 
\begin{flalign}
    &{\mathbb E[\langle \nabla f(\boldsymbol u_t), \boldsymbol u_{t}-\boldsymbol u_{t+1} \rangle|\mathcal F_t]} ={\mathbb E[\langle \nabla f(\boldsymbol x_t), \boldsymbol u_{t}-\boldsymbol u_{t+1} \rangle|\mathcal F_t]} +{\mathbb E\left[\left\langle \nabla f(\boldsymbol u_t)-\nabla f(x_t), \boldsymbol u_{t}-\boldsymbol u_{t+1} \right\rangle|\mathcal F_t\right]}.\nonumber
\end{flalign}
 The first part mimics the first-order term in the proof of RMSProp and the second one is due to a mismatch between $\boldsymbol u_t$ and $\boldsymbol x_t$. 
By bounding these two parts separately, we have a lemma similar to Lemma \ref{lemma:1}
but with two additional terms: $\mathbb E\left[\frac{\boldsymbol m_{t,i}^2}{{\boldsymbol v_{t,i}+\zeta}}\Big |\mathcal F_t\right]$ and $\mathbb E\left[\frac{\boldsymbol m_{t,i}^2}{\sqrt{\boldsymbol v_{t,i}+\zeta}}\Big |\mathcal F_t\right]$ due to the employment of the first-order momentum (see details in Lemma \ref{formal:lemma3}).
For the additional terms, \citet{wang2023closing} showed that $\sum_{t=1}^T\frac{\boldsymbol m_{t,i}^2}{{\boldsymbol v_{t,i}}}$
can be bounded by a function of $\boldsymbol  v_{T,i}$ (shown in Lemma \ref{lemma:5}). It is worth noting that due to the $(L_0,L_1)$-smoothness our objective is harder to bound and only Lemma \ref{lemma:5} is not enough. Here, we bound the $\sum_{t=1}^T\frac{\boldsymbol m_{t,i}^2}{\sqrt{\boldsymbol v_{t,i}}}$ term by a function of $\frac{\sqrt{\boldsymbol v_{T,i}}-\sqrt{\boldsymbol v_{0,i}}}{1-\beta_2}+\sum_{t=1}^T \sqrt{\boldsymbol v_{t-1,i}}$ (the details can be found in Lemma \ref{lemma:6}).

\textit{Step 1.2: upper bound on the second-order and additional terms.} 
Based on the update process of $\boldsymbol u_t$ and $\boldsymbol x_t$, we bound the second-order and additional terms similarly to those in \eqref{eq:second} and \eqref{eq:addtional}, but with $\boldsymbol g_t$ replaced by $\boldsymbol m_t$ (see details in \eqref{eq:7eq2} and \eqref{eq:7eq3}). Unlike in the proof of RMSProp, where $\mathbb E\left[\frac{\boldsymbol g_{t,i}^2}{\sqrt{\beta_2\boldsymbol v_{t,i}+\zeta}}\Big|\mathcal F_t\right]$ can be bounded by $\frac{D_0}{\sqrt{\zeta}}+\frac{D_1 (\partial_i f(\boldsymbol x_t))^2}{\sqrt{\boldsymbol v_{t,i}+\zeta}}$, Assumption \ref{assump:variance} does not hold for $\boldsymbol m_t$ and this is the reason why two terms $\mathbb E\left[\frac{\boldsymbol m_{t,i}^2}{{\boldsymbol v_{t,i}+\zeta}}\Big |\mathcal F_t\right]$ and $\mathbb E\left[\frac{\boldsymbol m_{t,i}^2}{\sqrt{\boldsymbol v_{t,i}+\zeta}}\Big |\mathcal F_t\right]$ are kept in the upper bound of the second-order and additional terms.
Then based on the descent lemma of $f(\boldsymbol u_t)$, we can show $\mathbb E\left[\sum_{i=1}^d\frac{\eta (\partial_i f(\boldsymbol x_{t}))^2}{\sqrt{{\beta_2 \boldsymbol v_{t-1,i}} + \zeta}}\right]$ can be upper bounded by a function of $\sum_{t=1}^T\frac{\boldsymbol m_{t,i}^2}{{\boldsymbol v_{t,i}}}$ and $\sum_{t=1}^T\frac{\boldsymbol m_{t,i}^2}{\sqrt{\boldsymbol v_{t,i}}}$, which further can be bounded by a function of $\mathbb E[\nabla f(\boldsymbol x_t)]$. We then get the following lemma
\begin{lemma} \label{co:2}
    Let Assumptions \ref{assump:lowerbound}, \ref{assump:variance} and \ref{assump:generalsmooth} hold. With the same parameters as in Theorem \ref{theorem:2}, we have that
    \begin{flalign}
        \frac{1}{T} \sum_{t=1}^T \mathbb E[\frac{ \|\nabla f(\boldsymbol x_t)\|^2}{\sqrt{{\beta_2 \|\boldsymbol v_{t-1}} \|+ \zeta}} ]\le \epsilon^2+\frac{\epsilon}{T}\sum_{t=1}^T\mathbb E[\|\nabla f(\boldsymbol x_t)\|].\nonumber
    \end{flalign}
\end{lemma}
The details of the proof can be found in the Appendix \ref{proof:co2}.

\textbf{Stage \uppercase\expandafter{\romannumeral2}} is the same as the proof of RMSProp and thus is omitted here.

\textbf{Stage \uppercase\expandafter{\romannumeral3}:} upper bound of $\mathbb E\left[{ \|\nabla f(\boldsymbol x_t)\|}\right]$. As we mentioned before, Lemma \ref{lemma:2} and \eqref{eq:stage2} hold for Adam. It is worth noting that in the RHS of (\ref{eq:stage2}), the first term is bounded in Lemma \ref{co:2}, which is more complicated than Lemma \ref{co:1} and has an additional term. Let $e=\frac{1}{T} \sum_{t=1}^T  \mathbb E[\|\nabla f(\boldsymbol x_t)\|]$. By (\ref{eq:stage2}), Lemma \ref{co:2} and Lemma \ref{lemma:2} we have that
\begin{flalign}
    e^2\le c\epsilon^2+\frac{2\sqrt{dD_1}}{\sqrt{C_6}\epsilon}e\epsilon^2+ce\epsilon +\frac{e^2}{2}.
\end{flalign}
Thus, $e\sim \mathcal O(\epsilon)$, which shows the algorithm converges to a stationary point.
\end{proof}

\section{Comparison with Existing Works}
We observe that there are two recent works \citep{li2023convergence,wang2024convergence} on Adam with $(L_0,L_1)$-smooth objectives. In this section, we provide detailed comparisons with them.

\citet{li2023convergence} studies the original Adam where the adaptive learning rate is $\frac{\eta}{\sqrt{\boldsymbol v_t}+\lambda}$ and in this paper we study the modified one where the adaptive learning rate is $\frac{\eta}{\sqrt{\boldsymbol v_t+\zeta}}$. Fig. \ref{fig:1} {and Fig. \ref{fig:2}} demonstrates that this modification has little influence to the model performance. Moreover, 
though computational complexities of the method in \citet{li2023convergence} and our method are dependent on $\zeta$ (e.g.  $\mathcal O(\zeta^{-2})$ for our paper and $\mathcal O(\zeta^{-4})$ for \citet{li2023convergence}), in practice, the selection of $\zeta$ makes minor differences.  {The analysis of \citet{li2023convergence} is fundamentally different from this paper, which relies on a stopping time. Thus the authors in \citet{li2023convergence} only show Adam converges with high probability. Their proof relies on the fact that there exists a large constant $G$ such that $\|\nabla f(\boldsymbol x_t)\|\le G$ for any $t$ before their stopping time, which requires a stronger assumption on gradient noise. Thus the generalized smooth problem is converted to a standard $L$-smooth problem with a large smoothness constant, leading to very small step sizes which make the convergence slow in practice.} In this paper, we do not need this stopping time and we show that the expectation of gradient norm converges, which is stronger than the convergence with high probability.

\begin{figure}[htbp]\includegraphics[width=0.8\linewidth]{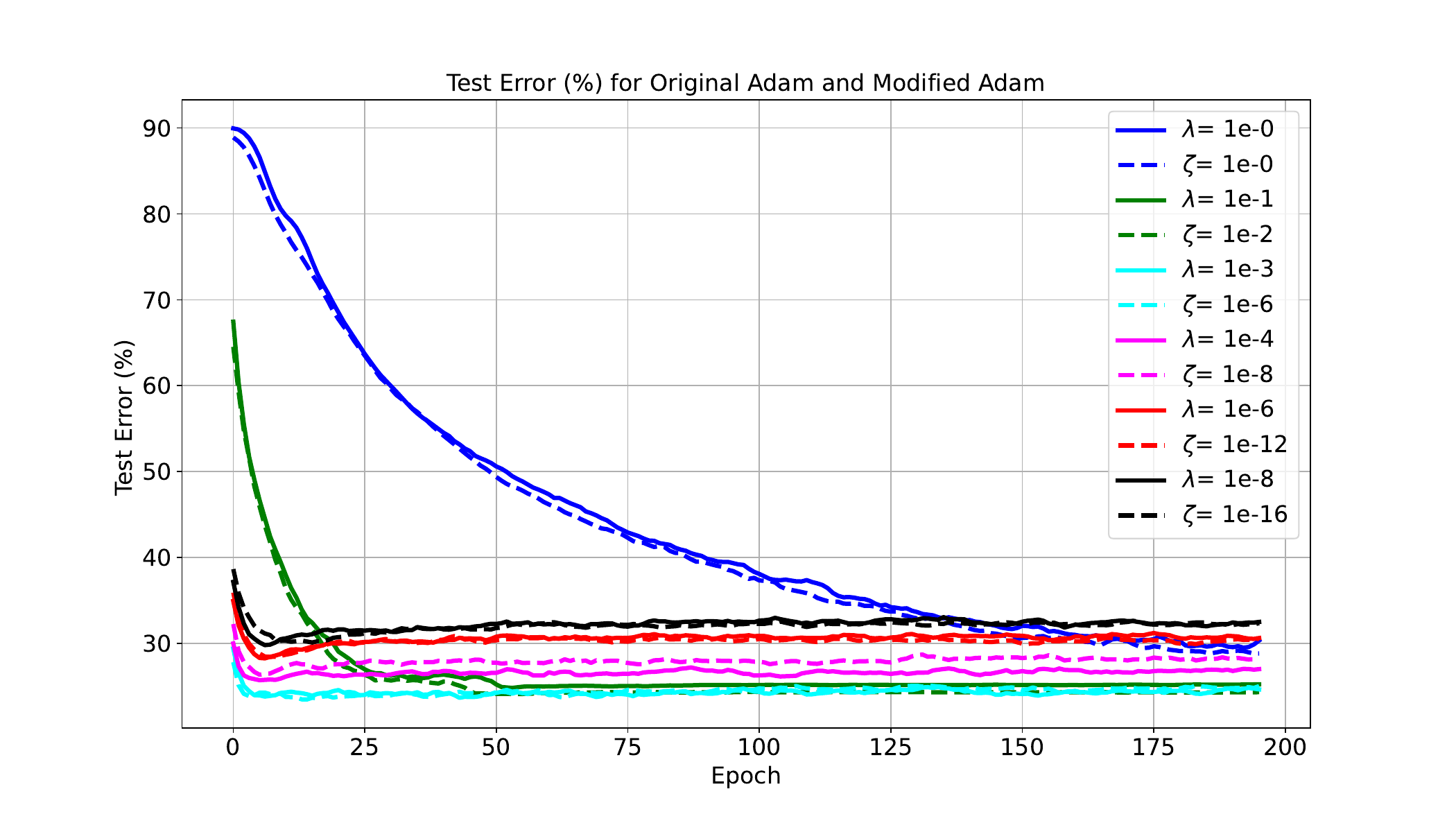}\centering \caption{Test Error for Original Adam and Modified Adam. {The stepsize in the original Adam is set to $\frac{\eta}{\sqrt{\boldsymbol v_t}+\lambda}$ and our stepsize is set to $\frac{\eta}{\sqrt{\boldsymbol v_t+\zeta}}$.} The parameters are the same as CNN task in Fig. 1 of \cite{li2023convergence}, where $\eta=0.001,\beta_1=0.9,\beta_2=0.999$ and we build a six layers CNN for CIFAR 10. }\label{fig:1}
\vspace{-0.3cm}
\end{figure}

\begin{figure}[htbp!]\includegraphics[width=0.6\linewidth]{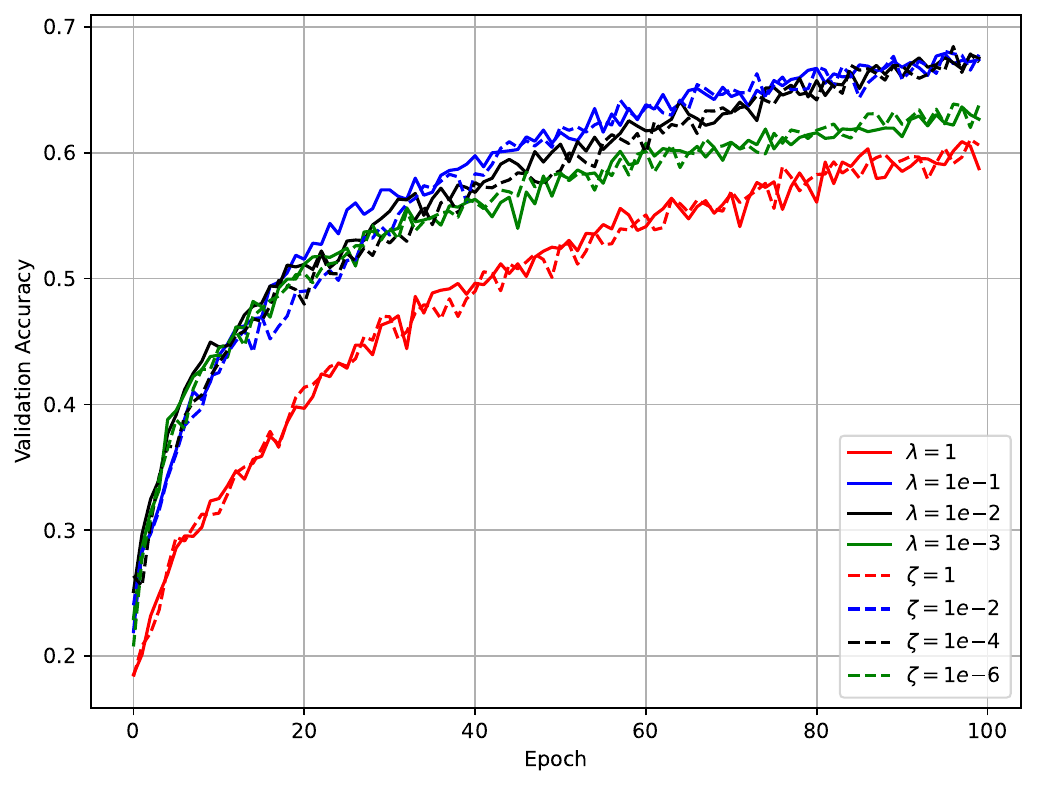}\centering \caption{{Test accuracy for Original Adam and Modified Adam. The stepsize in the original Adam is set to $\frac{\eta}{\sqrt{\boldsymbol v_t}+\lambda}$ and our stepsize is set to $\frac{\eta}{\sqrt{\boldsymbol v_t+\zeta}}$. We follow the setting in \cite{yoshioka2024visiontransformers} to build a vision-transformers for CIFAR 10. The stepsize is set to $\eta=0.001,\beta_1=0.9,\beta_2=0.999$.  }}\label{fig:2}
\vspace{-0.3cm}
\end{figure}

Both this paper and \citet{wang2023closing} extend the work of \citet{wang2023convergence}, which focused on the Adagrad algorithm. Our RMSProp analysis was developed concurrently with \citet{wang2023closing}, but there are significant differences between their proofs and ours. Specifically, we consider the generalized smoothness objectives while \citet{wang2023closing} only considers the $L$-smooth objectives. Thus we choose a different surrogate, modify the Adam and bound the first-order term and its denominator by different functions. When we extend our findings from RMSProp to Adam, we study a potential function $f(\boldsymbol u_t)$ with $\boldsymbol u_t=\frac{\boldsymbol x_t-\beta\boldsymbol x_{t-1}}{1-\beta}$, which was introduced by \cite{liu2020improved} and has been extensively employed in the analysis of momentum-based algorithms. The only thing inspired by \citet{wang2023closing} in this paper {is} to set this $\beta=\frac{\beta_1}{\sqrt{\beta_2}}$. \citet{wang2024convergence} extends the work of \citet{wang2023closing} and is a concurrent work of this paper, which focuses on the scalar version of Adam with $(L_0,L_1)$-smooth objectives. This paper focuses on the per-coordinate version of Adam and our proof is different from \citet{wang2023closing,wang2024convergence}.

\section{Conclusion}
In this paper, we provide tight convergence analyses for RMSProp and Adam for non-convex objectives under some of the most relaxed assumptions of generalized smoothness and affine gradient noise.  The complexity to achieve an $\epsilon$-stationary point for both algorithms is shown to be $\mathcal O(\epsilon^{-4})$, which matches with the lower bound for first-order algorithms established in \citet{arjevani2023lower}.  {In the future, we will explore the convergence of the original Adam with the challenging $(L_0,L_1)$-smoothness condition and affine gradient noise variance. }

\section{Acknowledgment}
The work of Q. Zhang and S. Zou was partially supported by the National Science Foundation under Grants ECCS-2438392 and ECCS-2501649. Y. Zhou’s work was supported by the National Science Foundation under grants DMS-2134223, ECCS-2237830.

\bibliography{main}

\begin{thebibliography}{48}
\providecommand{\natexlab}[1]{#1}
\providecommand{\url}[1]{\texttt{#1}}
\expandafter\ifx\csname urlstyle\endcsname\relax
  \providecommand{\doi}[1]{doi: #1}\else
  \providecommand{\doi}{doi: \begingroup \urlstyle{rm}\Url}\fi

\bibitem[Agarwal et~al.(2020)Agarwal, Schuurmans, and Norouzi]{agarwal2020optimistic}
Rishabh Agarwal, Dale Schuurmans, and Mohammad Norouzi.
\newblock An optimistic perspective on offline reinforcement learning.
\newblock In \emph{International Conference on Machine Learning}, pp.\  104--114. PMLR, 2020.

\bibitem[Arjevani et~al.(2023)Arjevani, Carmon, Duchi, Foster, Srebro, and Woodworth]{arjevani2023lower}
Yossi Arjevani, Yair Carmon, John~C Duchi, Dylan~J Foster, Nathan Srebro, and Blake Woodworth.
\newblock Lower bounds for non-convex stochastic optimization.
\newblock \emph{Mathematical Programming}, 199\penalty0 (1-2):\penalty0 165--214, 2023.

\bibitem[Bernstein et~al.(2018)Bernstein, Wang, Azizzadenesheli, and Anandkumar]{bernstein2018signsgd}
Jeremy Bernstein, Yu-Xiang Wang, Kamyar Azizzadenesheli, and Animashree Anandkumar.
\newblock signsgd: Compressed optimisation for non-convex problems.
\newblock In \emph{International Conference on Machine Learning}, pp.\  560--569. PMLR, 2018.

\bibitem[Bottou et~al.(2018)Bottou, Curtis, and Nocedal]{bottou2018optimization}
L{\'e}on Bottou, Frank~E Curtis, and Jorge Nocedal.
\newblock Optimization methods for large-scale machine learning.
\newblock \emph{SIAM review}, 60\penalty0 (2):\penalty0 223--311, 2018.

\bibitem[Brock et~al.(2018)Brock, Donahue, and Simonyan]{brock2018large}
Andrew Brock, Jeff Donahue, and Karen Simonyan.
\newblock Large scale gan training for high fidelity natural image synthesis.
\newblock \emph{arXiv preprint arXiv:1809.11096}, 2018.

\bibitem[Brown et~al.(2020)Brown, Mann, Ryder, Subbiah, Kaplan, Dhariwal, Neelakantan, Shyam, Sastry, Askell, et~al.]{brown2020language}
Tom Brown, Benjamin Mann, Nick Ryder, Melanie Subbiah, Jared~D Kaplan, Prafulla Dhariwal, Arvind Neelakantan, Pranav Shyam, Girish Sastry, Amanda Askell, et~al.
\newblock Language models are few-shot learners.
\newblock \emph{Advances in Neural Information Processing Systems}, 33:\penalty0 1877--1901, 2020.

\bibitem[Bubeck et~al.(2015)]{bubeck2015convex}
S{\'e}bastien Bubeck et~al.
\newblock Convex optimization: Algorithms and complexity.
\newblock \emph{Foundations and Trends{\textregistered} in Machine Learning}, 8\penalty0 (3-4):\penalty0 231--357, 2015.

\bibitem[Chen et~al.(2018)Chen, Liu, Sun, and Hong]{chen2018convergence}
Xiangyi Chen, Sijia Liu, Ruoyu Sun, and Mingyi Hong.
\newblock On the convergence of a class of {A}dam-type algorithms for non-convex optimization.
\newblock \emph{arXiv preprint arXiv:1808.02941}, 2018.

\bibitem[Chen et~al.(2023)Chen, Zhou, Liang, and Lu]{chen2023generalized}
Ziyi Chen, Yi~Zhou, Yingbin Liang, and Zhaosong Lu.
\newblock Generalized-smooth nonconvex optimization is as efficient as smooth nonconvex optimization.
\newblock \emph{arXiv preprint arXiv:2303.02854}, 2023.

\bibitem[Crawshaw et~al.(2022)Crawshaw, Liu, Orabona, Zhang, and Zhuang]{crawshaw2022robustness}
Michael Crawshaw, Mingrui Liu, Francesco Orabona, Wei Zhang, and Zhenxun Zhuang.
\newblock Robustness to unbounded smoothness of generalized signsgd.
\newblock \emph{Advances in Neural Information Processing Systems}, 35:\penalty0 9955--9968, 2022.

\bibitem[Cutkosky \& Mehta(2020)Cutkosky and Mehta]{cutkosky2020momentum}
Ashok Cutkosky and Harsh Mehta.
\newblock Momentum improves normalized {SGD}.
\newblock In \emph{International Conference on Machine Learning}, pp.\  2260--2268. PMLR, 2020.

\bibitem[De et~al.(2018)De, Mukherjee, and Ullah]{de2018convergence}
Soham De, Anirbit Mukherjee, and Enayat Ullah.
\newblock Convergence guarantees for {RMSP}rop and {A}dam in non-convex optimization and an empirical comparison to {N}esterov acceleration.
\newblock \emph{arXiv preprint arXiv:1807.06766}, 2018.

\bibitem[D{\'e}fossez et~al.(2020)D{\'e}fossez, Bottou, Bach, and Usunier]{defossez2020simple}
Alexandre D{\'e}fossez, L{\'e}on Bottou, Francis Bach, and Nicolas Usunier.
\newblock A simple convergence proof of {A}dam and {A}dagrad.
\newblock \emph{arXiv preprint arXiv:2003.02395}, 2020.

\bibitem[Dosovitskiy et~al.(2020)Dosovitskiy, Beyer, Kolesnikov, Weissenborn, Zhai, Unterthiner, Dehghani, Minderer, Heigold, Gelly, et~al.]{dosovitskiy2020image}
Alexey Dosovitskiy, Lucas Beyer, Alexander Kolesnikov, Dirk Weissenborn, Xiaohua Zhai, Thomas Unterthiner, Mostafa Dehghani, Matthias Minderer, Georg Heigold, Sylvain Gelly, et~al.
\newblock An image is worth 16x16 words: Transformers for image recognition at scale.
\newblock \emph{arXiv preprint arXiv:2010.11929}, 2020.

\bibitem[Duchi et~al.(2011)Duchi, Hazan, and Singer]{duchi2011adaptive}
John Duchi, Elad Hazan, and Yoram Singer.
\newblock Adaptive subgradient methods for online learning and stochastic pptimization.
\newblock \emph{Journal of Machine Learning Research}, 12\penalty0 (7), 2011.

\bibitem[Faw et~al.(2022)Faw, Tziotis, Caramanis, Mokhtari, Shakkottai, and Ward]{faw2022power}
Matthew Faw, Isidoros Tziotis, Constantine Caramanis, Aryan Mokhtari, Sanjay Shakkottai, and Rachel Ward.
\newblock The power of adaptivity in {SGD}: Self-tuning step sizes with unbounded gradients and affine variance.
\newblock In \emph{Conference on Learning Theory}, pp.\  313--355. PMLR, 2022.

\bibitem[Faw et~al.(2023)Faw, Rout, Caramanis, and Shakkottai]{faw2023beyond}
Matthew Faw, Litu Rout, Constantine Caramanis, and Sanjay Shakkottai.
\newblock Beyond uniform smoothness: A stopped analysis of adaptive {SGD}.
\newblock \emph{arXiv preprint arXiv:2302.06570}, 2023.

\bibitem[Foster et~al.(2019)Foster, Sekhari, Shamir, Srebro, Sridharan, and Woodworth]{foster2019complexity}
Dylan~J Foster, Ayush Sekhari, Ohad Shamir, Nathan Srebro, Karthik Sridharan, and Blake Woodworth.
\newblock The complexity of making the gradient small in stochastic convex optimization.
\newblock In \emph{Conference on Learning Theory}, pp.\  1319--1345. PMLR, 2019.

\bibitem[Ghadimi \& Lan(2013)Ghadimi and Lan]{ghadimi2013stochastic}
Saeed Ghadimi and Guanghui Lan.
\newblock Stochastic first-and zeroth-order methods for nonconvex stochastic programming.
\newblock \emph{SIAM Journal on Optimization}, 23\penalty0 (4):\penalty0 2341--2368, 2013.

\bibitem[Ghadimi et~al.(2016)Ghadimi, Lan, and Zhang]{ghadimi2016mini}
Saeed Ghadimi, Guanghui Lan, and Hongchao Zhang.
\newblock Mini-batch stochastic approximation methods for nonconvex stochastic composite optimization.
\newblock \emph{Mathematical Programming}, 155\penalty0 (1-2):\penalty0 267--305, 2016.

\bibitem[Guo et~al.(2021)Guo, Xu, Yin, Jin, and Yang]{guo2021novel}
Zhishuai Guo, Yi~Xu, Wotao Yin, Rong Jin, and Tianbao Yang.
\newblock A novel convergence analysis for algorithms of the {A}dam family and beyond.
\newblock \emph{arXiv preprint arXiv:2104.14840}, 2021.

\bibitem[Hardy et~al.(1952)Hardy, Littlewood, and P{\'o}lya]{hardy1952inequalities}
Godfrey~Harold Hardy, John~Edensor Littlewood, and George P{\'o}lya.
\newblock \emph{Inequalities}.
\newblock Cambridge university press, 1952.

\bibitem[Hinton et~al.(2012)Hinton, Srivastava, and Swersky]{hinton2012lecture}
G~Hinton, N~Srivastava, and K~Swersky.
\newblock Lecture 6d-a separate, adaptive learning rate for each connection.
\newblock \emph{Slides of Lecture Neural Networks for Machine Learning}, 1:\penalty0 1--31, 2012.

\bibitem[Hong \& Lin(2023)Hong and Lin]{hong2023high}
Yusu Hong and Junhong Lin.
\newblock High probability convergence of {A}dam under unbounded gradients and affine variance noise.
\newblock \emph{arXiv preprint arXiv:2311.02000}, 2023.

\bibitem[Jensen(1906)]{10.1007/BF02418571}
J.~L. W.~V. Jensen.
\newblock {Sur les fonctions convexes et les inégalités entre les valeurs moyennes}.
\newblock \emph{Acta Mathematica}, 30:\penalty0 175 -- 193, 1906.
\newblock \doi{10.1007/BF02418571}.
\newblock URL \url{https://doi.org/10.1007/BF02418571}.

\bibitem[Jin et~al.(2021)Jin, Zhang, Wang, and Wang]{jin2021non}
Jikai Jin, Bohang Zhang, Haiyang Wang, and Liwei Wang.
\newblock Non-convex distributionally robust optimization: Non-asymptotic analysis.
\newblock \emph{Advances in Neural Information Processing Systems}, 34:\penalty0 2771--2782, 2021.

\bibitem[Khaled \& Richt{\'a}rik(2020)Khaled and Richt{\'a}rik]{khaled2020better}
Ahmed Khaled and Peter Richt{\'a}rik.
\newblock Better theory for sgd in the nonconvex world.
\newblock \emph{arXiv preprint arXiv:2002.03329}, 2020.

\bibitem[Kingma \& Ba(2014)Kingma and Ba]{kingma2014adam}
Diederik~P Kingma and Jimmy Ba.
\newblock {A}dam: A method for stochastic optimization.
\newblock \emph{arXiv preprint arXiv:1412.6980}, 2014.

\bibitem[Li et~al.(2023)Li, Jadbabaie, and Rakhlin]{li2023convergence}
Haochuan Li, Ali Jadbabaie, and Alexander Rakhlin.
\newblock Convergence of {A}dam under relaxed assumptions.
\newblock \emph{arXiv preprint arXiv:2304.13972}, 2023.

\bibitem[Liu et~al.(2020)Liu, Gao, and Yin]{liu2020improved}
Yanli Liu, Yuan Gao, and Wotao Yin.
\newblock An improved analysis of stochastic gradient descent with momentum.
\newblock \emph{Advances in Neural Information Processing Systems}, 33:\penalty0 18261--18271, 2020.

\bibitem[Merity et~al.(2017)Merity, Keskar, and Socher]{merity2017regularizing}
Stephen Merity, Nitish~Shirish Keskar, and Richard Socher.
\newblock Regularizing and optimizing lstm language models.
\newblock \emph{arXiv preprint arXiv:1708.02182}, 2017.

\bibitem[Mikolov et~al.(2010)Mikolov, Karafi{\'a}t, Burget, Cernock{\`y}, and Khudanpur]{mikolov2010recurrent}
Tomas Mikolov, Martin Karafi{\'a}t, Lukas Burget, Jan Cernock{\`y}, and Sanjeev Khudanpur.
\newblock Recurrent neural network based language model.
\newblock In \emph{Interspeech}, volume~2, pp.\  1045--1048. Makuhari, 2010.

\bibitem[Nemirovskij \& Yudin(1983)Nemirovskij and Yudin]{nemirovskij1983problem}
Arkadij~Semenovi{\v{c}} Nemirovskij and David~Borisovich Yudin.
\newblock Problem complexity and method efficiency in optimization.
\newblock 1983.

\bibitem[Reisizadeh et~al.(2023)Reisizadeh, Li, Das, and Jadbabaie]{reisizadeh2023variance}
Amirhossein Reisizadeh, Haochuan Li, Subhro Das, and Ali Jadbabaie.
\newblock Variance-reduced clipping for non-convex optimization.
\newblock \emph{arXiv preprint arXiv:2303.00883}, 2023.

\bibitem[Richt{\'a}rik \& Tak{\'a}{\v{c}}(2014)Richt{\'a}rik and Tak{\'a}{\v{c}}]{richtarik2014iteration}
Peter Richt{\'a}rik and Martin Tak{\'a}{\v{c}}.
\newblock Iteration complexity of randomized block-coordinate descent methods for minimizing a composite function.
\newblock \emph{Mathematical Programming}, 144\penalty0 (1-2):\penalty0 1--38, 2014.

\bibitem[Shi et~al.(2020)Shi, Li, Hong, and Sun]{shi2020rmsprop}
Naichen Shi, Dawei Li, Mingyi Hong, and Ruoyu Sun.
\newblock {RMSP}rop converges with proper hyper-parameter.
\newblock In \emph{International Conference on Learning Representations}, 2020.

\bibitem[Vaswani et~al.(2019)Vaswani, Bach, and Schmidt]{vaswani2019fast}
Sharan Vaswani, Francis Bach, and Mark Schmidt.
\newblock Fast and faster convergence of {SGD} for over-parameterized models and an accelerated perceptron.
\newblock In \emph{International Conference on Artificial Intelligence and Atatistics}, pp.\  1195--1204. PMLR, 2019.

\bibitem[Wang et~al.(2022)Wang, Zhang, Zhang, Meng, Ma, Liu, and Chen]{wang2022provable}
Bohan Wang, Yushun Zhang, Huishuai Zhang, Qi~Meng, Zhi-Ming Ma, Tie-Yan Liu, and Wei Chen.
\newblock Provable adaptivity in {A}dam.
\newblock \emph{arXiv preprint arXiv:2208.09900}, 2022.

\bibitem[Wang et~al.(2023{\natexlab{a}})Wang, Fu, Zhang, Zheng, and Chen]{wang2023closing}
Bohan Wang, Jingwen Fu, Huishuai Zhang, Nanning Zheng, and Wei Chen.
\newblock Closing the gap between the upper bound and the lower bound of {A}dam's iteration complexity.
\newblock \emph{arXiv preprint arXiv:2310.17998}, 2023{\natexlab{a}}.

\bibitem[Wang et~al.(2023{\natexlab{b}})Wang, Zhang, Ma, and Chen]{wang2023convergence}
Bohan Wang, Huishuai Zhang, Zhiming Ma, and Wei Chen.
\newblock Convergence of {A}dagrad for non-convex objectives: Simple proofs and relaxed assumptions.
\newblock In \emph{Conference on Learning Theory}, pp.\  161--190. PMLR, 2023{\natexlab{b}}.

\bibitem[Wang et~al.(2024)Wang, Zhang, Meng, Sun, Ma, and Chen]{wang2024convergence}
Bohan Wang, Huishuai Zhang, Qi~Meng, Ruoyu Sun, Zhi-Ming Ma, and Wei Chen.
\newblock On the convergence of adam under non-uniform smoothness: Separability from sgdm and beyond.
\newblock \emph{arXiv preprint arXiv:2403.15146}, 2024.

\bibitem[Ward et~al.(2020)Ward, Wu, and Bottou]{ward2020adagrad}
Rachel Ward, Xiaoxia Wu, and Leon Bottou.
\newblock Adagrad stepsizes: Sharp convergence over nonconvex landscapes.
\newblock \emph{The Journal of Machine Learning Research}, 21\penalty0 (1):\penalty0 9047--9076, 2020.

\bibitem[Yoshioka(2024)]{yoshioka2024visiontransformers}
Kentaro Yoshioka.
\newblock vision-transformers-cifar10: Training vision transformers (vit) and related models on cifar-10.
\newblock \url{https://github.com/kentaroy47/vision-transformers-cifar10}, 2024.

\bibitem[Zaheer et~al.(2018)Zaheer, Reddi, Sachan, Kale, and Kumar]{zaheer2018adaptive}
Manzil Zaheer, Sashank Reddi, Devendra Sachan, Satyen Kale, and Sanjiv Kumar.
\newblock Adaptive methods for nonconvex optimization.
\newblock \emph{Advances in Neural Information Processing Systems}, 31, 2018.

\bibitem[Zhang et~al.(2020)Zhang, Jin, Fang, and Wang]{zhang2020improved}
Bohang Zhang, Jikai Jin, Cong Fang, and Liwei Wang.
\newblock Improved analysis of clipping algorithms for non-convex optimization.
\newblock \emph{Advances in Neural Information Processing Systems}, 33:\penalty0 15511--15521, 2020.

\bibitem[Zhang et~al.(2019)Zhang, He, Sra, and Jadbabaie]{zhang2019gradient}
Jingzhao Zhang, Tianxing He, Suvrit Sra, and Ali Jadbabaie.
\newblock Why gradient clipping accelerates training: A theoretical justification for adaptivity.
\newblock \emph{arXiv preprint arXiv:1905.11881}, 2019.

\bibitem[Zhang et~al.(2022)Zhang, Chen, Shi, Sun, and Luo]{zhang2022adam}
Yushun Zhang, Congliang Chen, Naichen Shi, Ruoyu Sun, and Zhi-Quan Luo.
\newblock {A}dam can converge without any modification on update rules.
\newblock \emph{Advances in Neural Information Processing Systems}, 35:\penalty0 28386--28399, 2022.

\bibitem[Zou et~al.(2019)Zou, Shen, Jie, Zhang, and Liu]{zou2019sufficient}
Fangyu Zou, Li~Shen, Zequn Jie, Weizhong Zhang, and Wei Liu.
\newblock A sufficient condition for convergences of {A}dam and {RMSP}rop.
\newblock In \emph{Proceedings of the IEEE/CVF Conference on computer vision and pattern recognition}, pp.\  11127--11135, 2019.

\end{thebibliography}
\bibliographystyle{tmlr}

\appendix

\section{Formal Version of Lemma \ref{lemma:1} and Its Proof} \label{proof:lemma1}
\begin{lemma}
Under Assumptions \ref{assump:variance} and \ref{assump:generalsmooth}, for any $\alpha_0,\alpha_1>0, \boldsymbol x_0=\boldsymbol x_1$ and $t>0$, we have that
\begin{flalign}\label{eq:lemma33}
    &\mathbb E\left[\left\langle \nabla f(\boldsymbol x_t), \frac{\eta \boldsymbol g_t}{\sqrt{\boldsymbol v_t+\zeta}} -\frac{\eta \boldsymbol g_t}{\sqrt{\beta_2 \boldsymbol {v}_{t-1}+\zeta}}\right\rangle\Big |\mathcal F_t\right] \nonumber\\
    \ge &-\Bigg( \sum_{i=1}^d \frac{\eta (\partial_i f(\boldsymbol x_t))^2}{2\alpha_0\sqrt{{\beta_2 \boldsymbol v_{t-1,i}} + \zeta}} +\sum_{i=1}^d\frac{\eta \alpha_0D_0}{2}\mathbb E\left[\frac{1}{\sqrt{{\beta_2 \boldsymbol v_{t-1,i}} + \zeta}}-\frac{1}{\sqrt{{ \boldsymbol v_{t,i}} + \zeta}}\Big |\mathcal F_t\right]\nonumber\\
    &+\sum_{i=1}^d\frac{\eta \alpha_0D_1}{2}\mathbb E\left[\frac{(\partial_i f(\boldsymbol x_{t-1}))^2}{\sqrt{{\beta_2 \boldsymbol v_{t-1,i}} + \zeta}}-\frac{(\partial_i f(\boldsymbol x_{t}))^2}{\sqrt{{ \boldsymbol v_{t,i}} + \zeta}}\Big |\mathcal F_t\right]\nonumber\\
    &+\sum_{i=1}^d \frac{\eta \alpha_0D_1}{2\sqrt{{\beta_2 \boldsymbol v_{t-1,i}} + \zeta}}\left(\frac{(\partial_i f(\boldsymbol x_t))^2}{\alpha_1 D_1}+\frac{\alpha_1d D_1L_0^2\eta^2}{1-\beta_2}+\frac{2\sqrt{d}\eta L_1(\partial_i f(\boldsymbol x_t))^2}{\sqrt{1-\beta_2}}\right)\Bigg)
\end{flalign}
\end{lemma}
\begin{proof}
Since $\boldsymbol v_t=\beta_2 \boldsymbol v_{t-1}+(1-\beta_2)\boldsymbol g_t \odot \boldsymbol  g_t$, for any $i\in [d]$, we have that 
\begin{flalign}\label{eq:lemma1eq1}
&-\frac{1}{\sqrt{{ \boldsymbol v_{t,i}} + \zeta}}+\frac{1}{\sqrt{{\beta_2 \boldsymbol v_{t-1,i}} + \zeta}}\nonumber\\
=& \frac{\boldsymbol v_{t,i}-\beta_2 \boldsymbol v_{t-1,i}}{\sqrt{{ \boldsymbol v_{t,i}} + \zeta}\sqrt{{\beta_2 \boldsymbol v_{t-1,i}} + \zeta}(\sqrt{{ \boldsymbol v_{t,i}}+\zeta}+\sqrt{{\beta_2 \boldsymbol v_{t-1,i}}+\zeta})}\nonumber\\
=& \frac{(1-\beta_2)\boldsymbol g_{t,i}^2}{\sqrt{{ \boldsymbol v_{t,i}} + \zeta}\sqrt{{\beta_2 \boldsymbol v_{t-1,i}} + \zeta}(\sqrt{{ \boldsymbol v_{t,i}}+\zeta}+\sqrt{{\beta_2 \boldsymbol v_{t-1,i}}+\zeta})}.
\end{flalign}
Thus, the first-order.b term can be bounded as 
\begin{flalign}\label{proof:eq1}
     &\mathbb E\left[\left\langle \nabla f(\boldsymbol x_t), \left[\frac{-\eta }{\sqrt{\boldsymbol v_t+\zeta  }} +\frac{\eta}{\sqrt{\beta_2 \boldsymbol{v}_{t-1}+\zeta }}\right]\odot \boldsymbol g_t\right\rangle\Bigg|\mathcal F_t\right] \nonumber\\
     \le &\sum_{i=1}^d \mathbb E\left[\eta |\partial_i f(\boldsymbol x_t)||\boldsymbol g_{t,i}|\left| \frac{-1}{\sqrt{\boldsymbol v_{t,i}+\zeta}} +\frac{1}{\sqrt{\beta_2 \boldsymbol {v}_{t-1,i}+\zeta}}\right|\Bigg|\mathcal F_t\right]\nonumber\\
     = &\sum_{i=1}^d \mathbb E\left[\eta |\partial_i f(\boldsymbol x_t)||\boldsymbol g_{t,i}| \frac{(1-\beta_2)\boldsymbol g_{t,i}^2 }{\sqrt{{\boldsymbol  v_{t,i}} + \zeta}\sqrt{{\beta_2 \boldsymbol v_{t-1,i}} + \zeta}(\sqrt{{ \boldsymbol v_{t,i}} + \zeta}+\sqrt{{\beta_2 \boldsymbol v_{t-1,i}} + \zeta})}\Bigg|\mathcal F_t\right]\nonumber\\
     \le & \sum_{i=1}^d \frac{\eta|\partial_i f(\boldsymbol x_t)|}{\sqrt{{\beta_2 \boldsymbol v_{t-1,i}} + \zeta}}\mathbb E\left[ \frac{\sqrt{1-\beta_2}\boldsymbol g_{t,i}^2}{(\sqrt{{ \boldsymbol v_{t,i}} + \zeta}+\sqrt{{\beta_2 \boldsymbol v_{t-1,i}} + \zeta})}\Bigg|\mathcal F_t\right],
\end{flalign}
where the second equality is due to \eqref{eq:lemma1eq1} and the last inequality is because $\frac{|\boldsymbol g_{t,i}|}{\sqrt{\boldsymbol v_{t,i}+\zeta}}\le \frac{1}{\sqrt{1-\beta_2}}$. For any $a,b \in \mathbb R$, we have $ab\le \frac{a^2+b^2}{2}$. Thus for any $\alpha_0>0$, the RHS of \eqref{proof:eq1} can be further bounded as
\begin{flalign}\label{proof:eq2}
     & \frac{\eta |\partial_i f(\boldsymbol x_{t})|}{\sqrt{{\beta_2 \boldsymbol v_{t-1,i}} + \zeta}}\mathbb E\left[ \frac{\sqrt{1-\beta_2}\boldsymbol g_{t,i}^2}{(\sqrt{{ \boldsymbol v_{t,i}} + \zeta}+\sqrt{{\beta_2 \boldsymbol v_{t-1,i}} + \zeta})}\Bigg|\mathcal F_t\right]\nonumber\\
     \le &\frac{\eta (\partial_i f(\boldsymbol x_t))^2}{2\alpha_0\sqrt{{\beta_2 \boldsymbol v_{t-1,i}} + \zeta}}+\frac{\eta\alpha_0 }{2\sqrt{{\beta_2 \boldsymbol v_{t-1,i}} + \zeta}}\left(\mathbb E\left[\frac{\sqrt{1-\beta_2}\boldsymbol g_{t,i}^2}{\sqrt{{ \boldsymbol v_{t,i}} + \zeta}+\sqrt{{\beta_2 \boldsymbol v_{t-1,i}} + \zeta}}\Bigg|\mathcal F_t\right]\right)^2.
\end{flalign}
For the last term of \eqref{proof:eq2}, due to H\"older's inequality, we have that 
\begin{flalign}\label{proof:eq119}
    &\frac{\eta }{2\sqrt{{\beta_2 \boldsymbol v_{t-1,i}} + \zeta}}\left(\mathbb E\left[\frac{\sqrt{1-\beta_2}\boldsymbol g_{t,i}^2}{\sqrt{{ \boldsymbol v_{t,i}} + \zeta}+\sqrt{{\beta_2 \boldsymbol v_{t-1,i}} + \zeta}}\Bigg|\mathcal F_t\right]\right)^2\nonumber\\
    \le& \frac{(1-\beta_2)\eta}{2\sqrt{{\beta_2 \boldsymbol v_{t-1,i}} + \zeta}}\mathbb E[\boldsymbol g_{t,i}^2|\mathcal F_t]\mathbb E\left[\frac{\boldsymbol g_{t,i}^2}{(\sqrt{{ \boldsymbol v_{t,i}} + \zeta}+\sqrt{{\beta_2 \boldsymbol v_{t-1,i}} + \zeta})^2}\Bigg|\mathcal F_t\right]\nonumber\\
    \le & \frac{\eta}{2}(D_0+D_1 (\partial_i f(\boldsymbol x_t))^2) 
    \mathbb E\left[\frac{(1-\beta_2)\boldsymbol g_{t,i}^2}{\sqrt{{ \boldsymbol v_{t,i}} + \zeta}\sqrt{{\beta_2 \boldsymbol v_{t-1,i}} + \zeta}(\sqrt{{ \boldsymbol v_{t,i}} + \zeta}+\sqrt{{\beta_2 \boldsymbol v_{t-1,i}} + \zeta})}\Bigg|\mathcal F_t\right]\nonumber\\
     = & \frac{\eta}{2}(D_0+D_1 (\partial_i f(\boldsymbol x_t))^2)
    \mathbb E\left[\frac{\boldsymbol v_{t,i}+\zeta-(\beta_2\boldsymbol v_{t-1,i}+\zeta)}{\sqrt{{ \boldsymbol v_{t,i}} + \zeta}\sqrt{{\beta_2 \boldsymbol v_{t-1,i}} + \zeta}(\sqrt{{ \boldsymbol v_{t,i}} + \zeta}+\sqrt{{\beta_2 \boldsymbol v_{t-1,i}} + \zeta})}\Bigg|\mathcal F_t\right]\nonumber\\
    = &\frac{\eta}{2}(D_0+D_1 (\partial_i f(\boldsymbol x_t))^2)\mathbb E\left[\frac{1}{\sqrt{{\beta_2 \boldsymbol v_{t-1,i}} + \zeta}}-\frac{1}{\sqrt{{ \boldsymbol v_{t,i}} + \zeta}}\Bigg|\mathcal F_t\right],
\end{flalign}
where the second inequality is due to Assumption \ref{assump:variance}, and the fact that $\sqrt{v_{t,i}+\zeta}+\sqrt{\beta_2v_{t-1,i}+\zeta}>\sqrt{v_{t,i}+\zeta}$.
Thus for any $t\ge 1$, combining \eqref{proof:eq1}, \eqref{proof:eq2} and \eqref{proof:eq119},  we have that
\begin{flalign}\label{proof:t1}
  &\mathbb E\left[\left\langle \nabla f(\boldsymbol x_t), \left[\frac{-\eta }{\sqrt{\boldsymbol v_t+\zeta  }} +\frac{\eta}{\sqrt{\beta_2 \boldsymbol {v}_{t-1}+\zeta }}\right]\odot \boldsymbol g_t\right\rangle\Bigg|\mathcal F_t\right] \nonumber\\
     \le & \sum_{i=1}^d\frac{\eta (\partial_i f(\boldsymbol x_t))^2}{2\alpha_0\sqrt{{\beta_2 \boldsymbol v_{t-1,i}} + \zeta}}+\sum_{i=1}^d\frac{\eta\alpha_0}{2}(D_0+D_1 (\partial_i f(\boldsymbol x_t))^2)\mathbb E\left[\frac{1}{\sqrt{{\beta_2 \boldsymbol v_{t-1,i}} + \zeta}}-\frac{1}{\sqrt{{ \boldsymbol v_{t,i}} + \zeta}}\Bigg|\mathcal F_t\right].
\end{flalign}
For $t>1$, \eqref{proof:eq119} can be rewritten as 
\begin{flalign}\label{proof:eq4}
& \frac{(1-\beta_2)\eta}{2\sqrt{{\beta_2 \boldsymbol v_{t-1,i}} + \zeta}}\mathbb E[\boldsymbol g_{t,i}^2|\mathcal F_t]\mathbb E\left[\frac{\boldsymbol g_{t,i}^2}{(\sqrt{{ \boldsymbol v_{t,i}} + \zeta}+\sqrt{{\beta_2 \boldsymbol v_{t-1,i}} + \zeta})^2}\Bigg|\mathcal F_t\right]\nonumber\\
    \le &\frac{\eta D_0}{2}\mathbb E\left[\frac{1}{\sqrt{{\beta_2 \boldsymbol v_{t-1,i}} + \zeta}}-\frac{1}{\sqrt{{ \boldsymbol v_{t,i}} + \zeta}}\Bigg|\mathcal F_t\right]+\frac{\eta D_1}{2}\mathbb E\left[\frac{(\partial_i f(\boldsymbol x_{t-1}))^2}{\sqrt{{\beta_2 \boldsymbol v_{t-1,i}} + \zeta}}-\frac{(\partial_i f(\boldsymbol x_{t}))^2}{\sqrt{{ \boldsymbol v_{t,i}} + \zeta}}\Bigg|\mathcal F_t\right]\nonumber\\
    &+ \frac{\eta D_1}{2}\frac{(\partial_i f(\boldsymbol x_t))^2-(\partial_i f(\boldsymbol x_{t-1}))^2}{\sqrt{{\beta_2 \boldsymbol v_{t-1,i}} + \zeta}}.
\end{flalign}
For the last term in (\ref{proof:eq4}), by the fact that $f$ is $(L_0,L_1)$-smooth and for any $\alpha_1>0$, we have that
\begin{flalign}\label{proof:eq5}
    &\frac{\eta D_1}{2}\frac{(\partial_i f(\boldsymbol x_t))^2-(\partial_i f(\boldsymbol x_{t-1}))^2}{\sqrt{{\beta_2 \boldsymbol v_{t-1,i}} + \zeta}}\nonumber\\
    \le& \frac{\eta D_1}{2}\frac{2|\partial_i f(\boldsymbol x_t)|\big|\partial_i f(\boldsymbol x_t)-\partial_i f(\boldsymbol x_{t-1})\big|}{\sqrt{{\beta_2 \boldsymbol v_{t-1,i}} + \zeta}} \nonumber\\
    \le &\eta D_1\frac{|\partial_i f(\boldsymbol x_t)|(L_0+L_1|\partial_i f(\boldsymbol x_t)|)\eta\left\|\frac{1}{\sqrt{\boldsymbol v_{t-1}+\zeta}}\odot \boldsymbol g_{t-1}\right\|}{\sqrt{{\beta_2 \boldsymbol v_{t-1,i}} + \zeta}}\nonumber\\
    \le &\frac{\eta D_1|\partial_i f(\boldsymbol x_t)|}{\sqrt{{\beta_2 \boldsymbol v_{t-1,i}} + \zeta}}\left(L_0\eta\left\|\frac{1}{\sqrt{\boldsymbol v_{t-1}+\zeta}}\odot \boldsymbol g_{t-1}\right\|+L_1|\partial_i f(\boldsymbol x_t)|\eta\left\|\frac{1}{\sqrt{\boldsymbol v_{t-1}+\zeta}}\odot \boldsymbol g_{t-1}\right\|\right)\nonumber\\
    \le & \frac{\eta D_1}{2\sqrt{{\beta_2 \boldsymbol v_{t-1,i}} + \zeta}}\left(\frac{(\partial_i f(\boldsymbol x_t))^2}{\alpha_1 D_1}+\alpha_1 D_1L_0^2\eta^2\sum_{i=1}^d\frac{(\boldsymbol g_{t-1,i})^2}{\boldsymbol v_{t-1,i}+\zeta}+\frac{2\sqrt{d}\eta L_1(\partial_i f(\boldsymbol x_t))^2}{\sqrt{1-\beta_2}}\right)\nonumber\\
    \le  & \frac{\eta D_1}{2\sqrt{{\beta_2 \boldsymbol v_{t-1,i}} + \zeta}}\left(\frac{(\partial_i f(\boldsymbol x_t))^2}{\alpha_1 D_1}+\frac{\alpha_1 dD_1L_0^2\eta^2}{1-\beta_2}+\frac{2\sqrt{d}\eta L_1(\partial_i f(\boldsymbol x_t))^2}{\sqrt{1-\beta_2}}\right).
\end{flalign} 
where the first inequality is due to the fact that for any $a,b$, we have $a^2-b^2\le 2|a||a-b|$, the second inequality is by the $(L_0,L_1)$-smoothness and the third one is because that $ab\le \frac{a^2+b^2}{2}$ for any $a,b$ and {$\frac{|\boldsymbol g_{t,i}|}{\sqrt{\boldsymbol v_{t,i}+\zeta}}\le \frac{1}{\sqrt{1-\beta_2}}$}. Based on (\ref{proof:eq1}), (\ref{proof:eq2}), (\ref{proof:eq4}) and (\ref{proof:eq5}), for $t>1$ we can get that
\begin{flalign}\label{eq:lemma33result}
    &\mathbb E\left[\left\langle \nabla f(\boldsymbol x_t), \left[\frac{-\eta \boldsymbol g_t}{\sqrt{\boldsymbol v_t+\zeta}} +\frac{\eta \boldsymbol g_t}{\sqrt{\beta_2 \boldsymbol {v}_{t-1}+\zeta}}\right] \right\rangle\Bigg|\mathcal F_t\right] \nonumber\\
    \le & \sum_{i=1}^d \frac{\eta (\partial_i f(\boldsymbol x_t))^2}{2\alpha_0\sqrt{{\beta_2 \boldsymbol v_{t-1,i}} + \zeta}} +\sum_{i=1}^d\frac{\eta \alpha_0D_0}{2}\mathbb E\left[\frac{1}{\sqrt{{\beta_2 \boldsymbol v_{t-1,i}} + \zeta}}-\frac{1}{\sqrt{{ \boldsymbol v_{t,i}} + \zeta}}\Bigg|\mathcal F_t\right]\nonumber\\
    &+\sum_{i=1}^d\frac{\eta \alpha_0D_1}{2}\mathbb E\left[\frac{(\partial_i f(\boldsymbol x_{t-1}))^2}{\sqrt{{\beta_2 \boldsymbol v_{t-1,i}} + \zeta}}-\frac{(\partial_i f(\boldsymbol x_{t}))^2}{\sqrt{{ \boldsymbol v_{t,i}} + \zeta}}\Bigg|\mathcal F_t\right]\nonumber\\
    &+\sum_{i=1}^d \frac{\eta \alpha_0D_1}{2\sqrt{{\beta_2 \boldsymbol v_{t-1,i}} + \zeta}}\left(\frac{(\partial_i f(\boldsymbol x_t))^2}{\alpha_1 D_1}+\frac{\alpha_1d D_1L_0^2\eta^2}{1-\beta_2}+\frac{2\sqrt{d}\eta L_1(\partial_i f(\boldsymbol x_t))^2}{\sqrt{1-\beta_2}}\right).
\end{flalign}
Since we set $\boldsymbol x_0=\boldsymbol x_1$, the RHS of \eqref{eq:lemma33result} is an upper bound on the RHS of \eqref{proof:t1}. As a result,  \eqref{eq:lemma33result}  also holds for $t=1$. We then complete the proof.
\end{proof}

\section{Proof of Lemma \ref{co:1}}\label{proof:co1}
For $(L_0,L_1)$-smooth objective functions, we have the following descent inequality (Lemma 1 in \citet{crawshaw2022robustness}):
\begin{flalign}\label{eq:l0l1descent}
   &\underbrace{\mathbb E\left[\left\langle \nabla f(\boldsymbol x_t), \frac{\eta}{\sqrt{\tilde{\boldsymbol v}_t+\zeta}} \odot \boldsymbol g_t \right\rangle\Bigg|\mathcal F_t\right]}_{\text{first-order.a}}\nonumber\\
    \le &  f(\boldsymbol x_t)-\mathbb E[f(\boldsymbol x_{t+1})|\mathcal F_t]+\underbrace{\sum_{i=1}^d\frac{L_0}{2\sqrt{d}}\mathbb E[\|\boldsymbol x_{t+1}-\boldsymbol x_t\||\boldsymbol x_{t+1,i}-\boldsymbol x_{t,i}||\mathcal F_t]}_{\text{second-order}}\nonumber\\
    &+\underbrace{\sum_{i=1}^d\frac{L_1|\partial_i f(\boldsymbol x_t)|}{2}\mathbb E[\|\boldsymbol x_{t+1}-\boldsymbol x_t\||\boldsymbol x_{t+1,i}-\boldsymbol x_{t,i}||\mathcal F_t]}_{\text{additional term}}\nonumber\\
   &\zhang{-} \underbrace{\mathbb E\left[\left\langle \nabla f(\boldsymbol x_t), \left[\frac{\eta}{\sqrt{\boldsymbol v_t+\zeta}} +\frac{-\eta }{\sqrt{\tilde{\boldsymbol v}_t+\zeta}}\right]\odot \boldsymbol g_t\right\rangle|\mathcal F_t\right]}_{\text{first-order.b}}.
\end{flalign}
For the first-order.a item, given $\mathcal F_t$, we have $\boldsymbol g_t$ independent of  $\tilde{\boldsymbol v}_t$ and $\boldsymbol x_t$. It then follows that $$\mathbb E\left[\left\langle \nabla f(\boldsymbol x_t), \frac{\eta}{\sqrt{\tilde{\boldsymbol v}_t+\zeta}} \odot \boldsymbol  g_t\right\rangle\Big|\mathcal F_t\right]=\sum_{i=1}^d\frac{\eta (\partial_i f(\boldsymbol x_t))^2}{\sqrt{{\beta_2 \boldsymbol v_{t-1,i}} + \zeta}}.$$ Based on Lemma \ref{lemma:1}, we provide a lower bound on the first-order.b term.
Plugging Lemma \ref{lemma:1} to \eqref{eq:l0l1descent}, we have the following inequality for any $\alpha_0,\alpha_1>0$ and $t>1$:
\begin{flalign}\label{proof:2eq1}
    &\sum_{i=1}^d\frac{\eta (\partial_i f(\boldsymbol x_t))^2}{\sqrt{{\beta_2 \boldsymbol v_{t-1,i}} + \zeta}}\nonumber\\
      \le &  f(\boldsymbol x_t)-\mathbb E[f(\boldsymbol x_{t+1})|\mathcal F_t]+\underbrace{\sum_{i=1}^d\frac{L_0}{2\sqrt{d}}\mathbb E[\|\boldsymbol x_{t+1}-\boldsymbol x_t\||\boldsymbol x_{t+1,i}-\boldsymbol x_{t,i}||\mathcal F_t]}_{\text{second-order}}\nonumber\\
    &+\underbrace{\sum_{i=1}^d\frac{L_1|\partial_i f(\boldsymbol x_t)|}{2}\mathbb E[\|\boldsymbol x_{t+1}-\boldsymbol x_t\||\boldsymbol x_{t+1,i}-\boldsymbol x_{t,i}||\mathcal F_t]}_{\text{additional term}}\nonumber\\
     &+\sum_{i=1}^d \frac{\eta (\partial_i f(\boldsymbol x_t))^2}{2\alpha_0\sqrt{{\beta_2 \boldsymbol v_{t-1,i}} + \zeta}} +\sum_{i=1}^d\frac{\eta \alpha_0D_0}{2}\mathbb E\left[\frac{1}{\sqrt{{\beta_2 \boldsymbol v_{t-1,i}} + \zeta}}-\frac{1}{\sqrt{{ \boldsymbol v_{t,i}} + \zeta}}\Bigg|\mathcal F_t\right]\nonumber\\
    &+\sum_{i=1}^d\frac{\eta \alpha_0D_1}{2}\mathbb E\left[\frac{(\partial_i f(\boldsymbol x_{t-1}))^2}{\sqrt{{\beta_2 \boldsymbol v_{t-1,i}} + \zeta}}-\frac{(\partial_i f(\boldsymbol x_{t}))^2}{\sqrt{{ \boldsymbol v_{t,i}} + \zeta}}\Bigg|\mathcal F_t\right]\nonumber\\
    &+\sum_{i=1}^d \frac{\eta \alpha_0D_1}{2\sqrt{{\beta_2 \boldsymbol v_{t-1,i}} + \zeta}}\left(\frac{(\partial_i f(\boldsymbol x_t))^2}{\alpha_1 D_1}+\frac{\alpha_1d D_1L_0^2\eta^2}{1-\beta_2}+\frac{2\sqrt{d}\eta L_1(\partial_i f(\boldsymbol x_t))^2}{\sqrt{1-\beta_2}}\right).
\end{flalign}
Now we focus on the second-order term, which can be bounded as follows:
\begin{flalign}\label{proof:2eq2}
    &{\sum_{i=1}^d\frac{L_0}{2\sqrt{d}}\mathbb E\big[\|\boldsymbol x_{t+1}-\boldsymbol x_t\||\boldsymbol x_{t+1,i}-\boldsymbol x_{t,i}|\big|\mathcal F_t\big]}\nonumber\\
    \le& \sum_{i=1}^d \frac{L_0}{2\sqrt{d}}\mathbb E\left[\frac{\|\boldsymbol x_{t+1}-\boldsymbol x_t\|^2}{2\sqrt{d}}+\frac{\sqrt{d}}{2}(\boldsymbol x_{t+1,i}-\boldsymbol x_{t,i})^2\Bigg|\mathcal F_t\right]\nonumber\\
    =& \sum_{i=1}^d \frac{L_0}{2}\mathbb E\left[\eta^2 \frac{\boldsymbol g_{t,i}^2}{\boldsymbol v_{t,i}+\zeta}\Bigg|\mathcal F_t\right]\nonumber\\
    \le& \sum_{i=1}^d \frac{L_0}{2} \mathbb E\left[\eta^2 \frac{\boldsymbol g_{t,i}^2}{\beta_2 \boldsymbol v_{t-1,i}+\zeta}\Bigg|\mathcal F_t\right]\nonumber\\
    \le &\sum_{i=1}^d \frac{L_0\eta^2}{2} \frac{D_0+D_1(\partial_i f(\boldsymbol x_t))^2}{\beta_2 \boldsymbol v_{t-1,i}+\zeta}\nonumber\\
    \le &\sum_{i=1}^d \frac{L_0\eta^2}{2\sqrt{\zeta}} \frac{D_0+D_1(\partial_i f(\boldsymbol x_t))^2}{\sqrt{\beta_2 \boldsymbol v_{t-1,i}+\zeta}},
\end{flalign}
where the fourth inequality is due to Assumption \ref{assump:variance}.
Based on Assumption \ref{assump:variance}, for any $\alpha_2>0$, the addition term can be bounded as follows
\begin{flalign}\label{proof:2eq3}
    &\frac{L_1|\partial_i f (\boldsymbol x_t)|}{2}\mathbb E[\|\boldsymbol x_{t+1}-\boldsymbol x_t\||\boldsymbol x_{t+1,i}-\boldsymbol x_{t,i}||\mathcal F_t]\nonumber\\
    \le& \frac{\sqrt{d}L_1|\partial_i f(\boldsymbol x_t)|}{2\sqrt{1-\beta_2}}\mathbb E\left[\eta^2 \frac{|\boldsymbol g_{t,i}|}{\sqrt{\beta_2 \boldsymbol v_{t-1,i}+\zeta}}\Bigg|\mathcal F_t\right]\nonumber\\
    = & \frac{\eta^2\sqrt{d}L_1|\partial_i f(\boldsymbol x_t)|}{2\sqrt{1-\beta_2}\sqrt{\beta_2 \boldsymbol v_{t-1,i}+\zeta}}\sqrt{\mathbb E\left[ {\boldsymbol g_{t,i}^2}\big|\mathcal F_t\right]}\nonumber\\
    \le &\frac{\eta^2\sqrt{d}L_1|\partial_i f(\boldsymbol x_t)|(\sqrt{D_0}+\sqrt{D_1}|\partial_i f(\boldsymbol x_t)|)}{2\sqrt{1-\beta_2}\sqrt{\beta_2 \boldsymbol v_{t-1,i}+\zeta}}\nonumber\\
    \le & \frac{\eta^2\sqrt{d}L_1\sqrt{D_1}(\partial_i f(\boldsymbol x_t))^2}{2\sqrt{1-\beta_2}\sqrt{\beta_2 \boldsymbol v_{t-1,i}+\zeta}}+  \frac{\eta^2\sqrt{d}(\partial_i f(\boldsymbol x_t))^2}{2({1-\beta_2})\alpha_2\sqrt{\beta_2 \boldsymbol v_{t-1,i}+\zeta}}+  \frac{\eta^2\sqrt{d}\alpha_2L_1^2D_0}{2{}\sqrt{\beta_2 \boldsymbol v_{t-1,i}+\zeta}},
\end{flalign}
where the first inequality is due to the fact that $\|\boldsymbol x_{t+1}-\boldsymbol x_t\|\le \frac{\sqrt{d}\eta}{\sqrt{1-\beta_2}}$, the second inequality is because that $\mathbb E\left[ {\boldsymbol g_{t,i}}|\mathcal F_t\right]\le \sqrt{\mathbb E\left[ {\boldsymbol g_{t,i}^2}|\mathcal F_t\right]}$, and the third inequality is due to Assumption \ref{assump:generalsmooth} and the fact that for any $a,b\ge 0$, we have $\sqrt{a+b}\le \sqrt{a}+\sqrt{b}$.
Plug (\ref{proof:2eq2}) and (\ref{proof:2eq3}) into (\ref{proof:2eq1}), and it follows that
\begin{flalign}
    &\sum_{i=1}^d\frac{\eta (\partial_i f(\boldsymbol x_t))^2}{\sqrt{{\beta_2 \boldsymbol v_{t-1,i}} + \zeta}}\nonumber\\
      \le &  f(\boldsymbol x_t)-\mathbb E[f(\boldsymbol x_{t+1})|\mathcal F_t]+\sum_{i=1}^d \frac{L_0\eta^2}{2\sqrt{\zeta}} \frac{D_0+D_1(\partial_i f(\boldsymbol x_t))^2}{\sqrt{\beta_2 \boldsymbol v_{t-1,i}+\zeta}}\nonumber\\
    &+\sum_{i=1}^d\left(\frac{\eta^2\sqrt{d}L_1\sqrt{D_1}(\partial_i f(\boldsymbol x_t))^2}{2\sqrt{1-\beta_2}\sqrt{\beta_2 \boldsymbol v_{t-1,i}+\zeta}}+  \frac{\eta^2\sqrt{d}(\partial_i f(\boldsymbol x_t))^2}{2({1-\beta_2})\alpha_2\sqrt{\beta_2 \boldsymbol v_{t-1,i}+\zeta}}+  \frac{\eta^2\sqrt{d}\alpha_2L_1^2D_0}{2{}\sqrt{\beta_2 \boldsymbol v_{t-1,i}+\zeta}}\right)\nonumber\\
     &+\sum_{i=1}^d \frac{\eta (\partial_i f(\boldsymbol x_t))^2}{2\alpha_0\sqrt{{\beta_2 \boldsymbol v_{t-1,i}} + \zeta}} +\sum_{i=1}^d\frac{\eta \alpha_0D_0}{2}\mathbb E\left[\frac{1}{\sqrt{{\beta_2 \boldsymbol v_{t-1,i}} + \zeta}}-\frac{1}{\sqrt{{ \boldsymbol v_{t,i}} + \zeta}}\Bigg|\mathcal F_t\right]\nonumber\\
    &+\sum_{i=1}^d\frac{\eta \alpha_0D_1}{2}\mathbb E\left[\frac{(\partial_i f(\boldsymbol x_{t-1}))^2}{\sqrt{{\beta_2 \boldsymbol v_{t-1,i}} + \zeta}}-\frac{(\partial_i f(\boldsymbol x_{t}))^2}{\sqrt{{ \boldsymbol v_{t,i}} + \zeta}}\Bigg|\mathcal F_t\right]\nonumber\\
    &+\sum_{i=1}^d \frac{\eta \alpha_0D_1}{2\sqrt{{\beta_2 \boldsymbol v_{t-1,i}} + \zeta}}\left(\frac{(\partial_i f(\boldsymbol x_t))^2}{\alpha_1 D_1}+\frac{\alpha_1d D_1L_0^2\eta^2}{1-\beta_2}+\frac{2\sqrt{d}\eta L_1(\partial_i f(\boldsymbol x_t))^2}{\sqrt{1-\beta_2}}\right),
\end{flalign}
which can be written as 
\begin{flalign}\label{proof:4eq3}
    &\sum_{i=1}^d\left(\eta-\frac{\eta}{2\alpha_0}-\frac{\eta\alpha_0}{2\alpha_1}-\frac{L_0\eta^2D_1}{2\sqrt{\zeta}}-\frac{\eta^2L_1\sqrt{dD_1}}{2\sqrt{1-\beta_2}}-\frac{\eta^2\sqrt{d}}{2(1-\beta_2)\alpha_2}-\frac{\eta^2\sqrt{d}\alpha_0L_1D_1}{\sqrt{1-\beta_2}}\right)  \frac{ (\partial_i f(\boldsymbol x_t))^2}{\sqrt{{\beta_2 \boldsymbol v_{t-1,i}} + \zeta}} \nonumber\\
    \le & f(\boldsymbol x_t)-\mathbb E[f(\boldsymbol x_{t+1})|\mathcal F_t]+\frac{dL_0\eta^2D_0}{2{\zeta}}+\frac{\alpha_0\alpha_1d^2 L_0^2\eta^3 D_1^2}{2(1-\beta_2)} \frac{1}{\sqrt{ \zeta}}+  \frac{\eta^2\alpha_2d^{1.5}L_1^2D_0}{2{}\sqrt{\zeta}} \nonumber\\
    &+\sum_{i=1}^d\frac{\eta\alpha_0 D_0}{2}\mathbb E\left[\frac{1}{\sqrt{{\beta_2 \boldsymbol v_{t-1,i}} + \zeta}}-\frac{1}{\sqrt{{ \boldsymbol v_{t,i}} + \zeta}}\Bigg|\mathcal F_t\right]\nonumber\\
    &+\sum_{i=1}^d\frac{\eta\alpha_0 D_1}{2}\mathbb E\left[\frac{(\partial_i f(\boldsymbol x_{t-1}))^2}{\sqrt{{\beta_2 \boldsymbol v_{t-1,i}} + \zeta}}-\frac{(\partial_i f(\boldsymbol x_{t}))^2}{\sqrt{{ \boldsymbol v_{t,i}} + \zeta}}\Bigg|\mathcal F_t\right].
\end{flalign}
It is worth noting that the sum of the last two terms in \eqref{proof:4eq3} from $t=1$ to $T$ can be further bounded. Specifically, for any $i\in[d]$, taking the expectation with respect to $\mathcal F_t$, and the sum from $t=1$ to $T$, we have that 
\begin{flalign}\label{proof:coeq1}
    &\sum_{t=1}^T \mathbb E\left[\frac{1}{\sqrt{{\beta_2 \boldsymbol v_{t-1,i}} + \zeta}}-\frac{1}{\sqrt{{ \boldsymbol v_{t,i}} + \zeta}}\right]\nonumber\\
    =& \mathbb E\left[\frac{1}{\sqrt{{\beta_2 \boldsymbol v_{0,i}} + \zeta}}\right]+\sum_{t=1}^{T-1} \mathbb E\left[\frac{1}{\sqrt{{\beta_2 \boldsymbol v_{t,i}} + \zeta}}-\frac{1}{\sqrt{{ \boldsymbol v_{t,i}} + \zeta}}\right]-\mathbb E\left[\frac{1}{\sqrt{{ \boldsymbol v_{T,i}} + \zeta}}\right]\nonumber\\
    \le & \frac{1}{\sqrt{ \zeta}}+\sum_{t=1}^{T-1} \mathbb E\left[\frac{1}{\sqrt{{\beta_2 \boldsymbol v_{t,i}} +  \zeta}}-\frac{\sqrt{\beta_2}}{\sqrt{\beta_2 { \boldsymbol v_{t,i}} +  \zeta}}\right]\nonumber\\
    = & \frac{1}{\sqrt{ \zeta}}+\sum_{t=1}^{T-1} \mathbb E\left[\frac{1-\sqrt{\beta_2}}{\sqrt{\beta_2{ \boldsymbol v_{t,i}} +  \zeta}}\right]\nonumber\\
    \le &  \frac{1}{\sqrt{ \zeta}}+T \frac{1-\sqrt{\beta_2}}{\sqrt{\zeta}}.
\end{flalign}
Similarly, for the last term in \eqref{proof:4eq3}, the sum from $t=1$ to $T$ can be bounded as follows
\begin{flalign}\label{proof:coeq2}
    &\sum_{t=1}^T\mathbb E\left[\frac{(\partial_i f(\boldsymbol x_{t-1}))^2}{\sqrt{{\beta_2 \boldsymbol v_{t-1,i}} + \zeta}}-\frac{(\partial_i f(\boldsymbol x_{t}))^2}{\sqrt{{ \boldsymbol v_{t,i}} + \zeta}}\right]\nonumber\\
    =&\mathbb E\left[\frac{(\partial_i f(\boldsymbol x_{0}))^2}{\sqrt{{\beta_2 \boldsymbol v_{0,i}} + \zeta}}-\frac{(\partial_i f(\boldsymbol x_{1}))^2}{\sqrt{{ \boldsymbol v_{1,i}} + \zeta}}\right]+\sum_{t=2}^T \mathbb E\left[\frac{(\partial_i f(\boldsymbol x_{t-1}))^2}{\sqrt{{\beta_2 \boldsymbol v_{t-1,i}} + \zeta}}-\frac{(\partial_i f(\boldsymbol x_{t}))^2}{\sqrt{{ \boldsymbol v_{t,i}} + \zeta}}\right]\nonumber\\
    =& \mathbb E\left[\frac{(\partial_i  f(\boldsymbol x_{1}))^2}{\sqrt{{\beta_2 \boldsymbol v_{0,i}} + \zeta}}\right]+\sum_{t=1}^{T-1} \mathbb E\left[\frac{(\partial_i f(\boldsymbol x_{t}))^2}{\sqrt{{\beta_2 \boldsymbol v_{t,i}} + \zeta}}-\frac{(\partial_i f(\boldsymbol x_{t}))^2}{\sqrt{{ \boldsymbol v_{t,i}} + \zeta}}\right]-\mathbb E\left[\frac{(\partial_i f(\boldsymbol x_{T}))^2}{\sqrt{{\boldsymbol  v_{T,i}} + \zeta}}\right]\nonumber\\
    \le & \mathbb E\left[\frac{(\partial_i f(\boldsymbol x_{1}))^2}{\sqrt{ \zeta}}\right]+\sum_{t=1}^{T-1} \mathbb E\left[\left(\frac{1}{\sqrt{{\beta_2 \boldsymbol v_{t,i}} +  \zeta}}-\frac{\sqrt{\beta_2}}{\sqrt{\beta_2 { \boldsymbol v_{t,i}} +  \zeta}}\right)(\partial_i f(\boldsymbol x_{t}))^2\right]\nonumber\\
    = & \frac{(\partial_i f(\boldsymbol x_{1}))^2}{\sqrt{ \zeta}}+\sum_{t=1}^{T-1}(1-\sqrt{\beta_2}) \mathbb E\left[\frac{(\partial_i f(\boldsymbol x_{t}))^2}{\sqrt{\beta_2{ \boldsymbol v_{t,i}} +  \zeta}}\right]\nonumber\\
    \le & \frac{(\partial_i f(\boldsymbol x_{1}))^2}{\sqrt{ \zeta}}+\sum_{t=1}^{T-1}\frac{1-\sqrt{\beta_2}}{\sqrt{\beta_2}} \mathbb E\left[\frac{(\partial_i f(\boldsymbol x_{t}))^2}{\sqrt{{ \boldsymbol v_{t,i}} +  \zeta}}\right]\nonumber\\
    \le & \frac{(\partial_i f(\boldsymbol x_{1}))^2}{\sqrt{ \zeta}}+\sum_{t=1}^{T-1}({1-\beta_2}) \mathbb E\left[\frac{(\partial_i f(\boldsymbol x_{t}))^2}{\sqrt{\beta_2{ \boldsymbol v_{t-1,i}} +  \zeta}}\right],
\end{flalign}
since $\boldsymbol x_0=\boldsymbol x_1$, $\boldsymbol g_0=0$ and the last inequality is due to the fact that $\frac{1}{\sqrt{\beta_2}}\le 1+\sqrt{\beta_2}$.
By taking expectations and sums of (\ref{proof:4eq3}) from $t=1$ to $T$, based on (\ref{proof:coeq1}) and (\ref{proof:coeq2}), we have that
\begin{flalign}
       &\sum_{t=1}^T\sum_{i=1}^d\left(\eta-\frac{\eta}{2\alpha_0}-\frac{\eta\alpha_0}{2\alpha_1}-\frac{L_0\eta^2D_1}{2\sqrt{\zeta}}-\frac{\eta^2L_1\sqrt{dD_1}}{2\sqrt{1-\beta_2}}-\frac{\eta^2\sqrt{d}}{2(1-\beta_2)\alpha_2}-\frac{\eta^2\sqrt{d}\alpha_0L_1D_1}{\sqrt{1-\beta_2}}\right)  \frac{ (\partial_i f(\boldsymbol x_t))^2}{\sqrt{{\beta_2 \boldsymbol v_{t-1,i}} + \zeta}} \nonumber\\
    \le & f(\boldsymbol x_1)-\mathbb E[f(\boldsymbol x_{T+1})|\mathcal F_t]+\frac{\eta \alpha_0 D_1\|\nabla f(\boldsymbol x_1)\|^2}{2\sqrt{\zeta}}+\frac{\eta \alpha_0 dD_0 }{2\sqrt{\zeta}}+T\frac{dL_0\eta^2D_0}{2{\zeta}}+T\frac{\alpha_0\alpha_1d^2 L_0^2\eta^3 D_1^2}{2(1-\beta_2)} \frac{1}{\sqrt{ \zeta}} \nonumber\\
    &+  T\frac{\eta^2\alpha_2d^{1.5}L_1^2D_0}{2{}\sqrt{\zeta}}+T\frac{\eta \alpha_0dD_0(1-\sqrt{\beta_2})}{2\sqrt{ \zeta}}+\frac{\eta \alpha_0D_1(1-{\beta_2})}{2}\sum_{i=1}^d\sum_{t=1}^T \mathbb E\left[\frac{ \|\partial_i f(\boldsymbol x_t)\|^2}{\sqrt{{\beta_2 \boldsymbol v_{t-1,i}} + \zeta}} \right].
\end{flalign}
Define $\Delta=f(\boldsymbol x_1)-f(\boldsymbol x^*)+\frac{\eta\alpha_0 dD_0 }{2\sqrt{\zeta}}+\frac{\eta \alpha_0D_1\|\nabla f(\boldsymbol x_1)\|^2}{2\sqrt{\zeta}}$, where $f(\boldsymbol x^*)=\inf_{\boldsymbol x} f(\boldsymbol x)$. If we set $\alpha_0=1, \alpha_1=7$ and $\alpha_2=1$, obviously we can find some $1-\beta_2=\min\left(\frac{1}{7D_1},\frac{\sqrt{\zeta}\epsilon^2}{35dD_0}\right)= \mathcal O(\epsilon^{2})$, and 
\begin{align}
    \eta\le \min\left(\frac{\sqrt{\zeta}}{7L_0D_1},\frac{\sqrt{1-\beta_2}}{\max{(14L_1\sqrt{d}D_1,7L_1\sqrt{dD_1})}},\frac{1-\beta_2}{7\sqrt{d}},\frac{\zeta \epsilon^2}{35L_0dD_0},\frac{\epsilon \sqrt{1-\beta_2}\sqrt[4]{\zeta}}{7D_1L_0d\sqrt{5}},\frac{\sqrt{\zeta} \epsilon^2}{35L_1^2d^{1.5}D_0}\right)= \mathcal O(\epsilon^{2})\nonumber
\end{align}
such that
\begin{flalign}
    \frac{\eta}{14}\sum_{i=1}^d\sum_{t=1}^T \mathbb E\left[\frac{ (\partial_i f(\boldsymbol x_t))^2}{\sqrt{{\beta_2 \boldsymbol v_{t-1,i}} + \zeta}} \right]\le \Delta+T\frac{\eta}{70}\epsilon^2+T\frac{\eta}{70}\epsilon^2+T\frac{\eta}{70}\epsilon^2+T\frac{\eta}{70}\epsilon^2,
\end{flalign}
Set $T\ge\frac{70\Delta}{\eta\epsilon^2} = \mathcal O(\epsilon^{-4})$, and we have that
\begin{flalign}
   \frac{1}{T} \sum_{t=1}^T \mathbb E\left[\frac{ \|\nabla f(\boldsymbol x_t)\|^2}{\sqrt{{\beta_2 \|\boldsymbol v_{t-1}}\| + \zeta}} \right]\le\sum_{i=1}^d\frac{1}{T} \sum_{t=1}^T \mathbb E\left[\frac{ (\partial_i f(\boldsymbol x_t))^2}{\sqrt{{\beta_2 \boldsymbol v_{t-1,i}} + \zeta}} \right]\le \epsilon^2. 
\end{flalign}
This completes the proof.

\section{Proof of Lemma \ref{lemma:2}}\label{proof:lemma2}
For any $a,b\ge0$, we have that $\sqrt{a+b}\le \sqrt{a}+\sqrt{b}$. It then follows that
\begin{flalign}
    \mathbb E\Big[\sqrt{{\beta_2 \|\boldsymbol v_{t-1}}\| + \zeta}\Big]\le \sum_{i=1}^d\mathbb E\Big[\sqrt{{\beta_2 \boldsymbol v_{t-1,i}}}\Big]+\sqrt{\zeta}.
\end{flalign}
For the first term $ \mathbb E[\sqrt{{\beta_2 \boldsymbol v_{t-1,i}}}]$ and $t>1$ we have that
\begin{flalign}
     \mathbb E[\sqrt{{\beta_2 \boldsymbol v_{t-1,i}}}]=\mathbb E\left[ \mathbb E \left[\sqrt{{\beta_2 \boldsymbol v_{t-1,i}}}\Big|\mathcal F_{t-1}\right]\right]=\mathbb E\left[ \mathbb E \left[\sqrt{{\beta_2^2 \boldsymbol v_{t-2,i}}+\beta_2(1-\beta_2)(\boldsymbol g_{t-1,i})^2}\Big|\mathcal F_{t-1}\right]\right].
\end{flalign}
{Given} $\mathcal F_{t-1}$,  $\boldsymbol v_{t-2}$ is deterministic. By Jensen's inequality, we have that
\begin{flalign}\label{proof:3eq1}
    &\mathbb E\left[ \mathbb E \left[\sqrt{{\beta_2^2 \boldsymbol v_{t-2,i}}+\beta_2(1-\beta_2)(\boldsymbol g_{t-1,i})^2}\Big|\mathcal F_{t-1}\right]\right]\nonumber\\
    \le & \mathbb E\left[ \sqrt{\mathbb E  \left[ {\beta_2^2 \boldsymbol v_{t-2,i}}+\beta_2(1-\beta_2)(\boldsymbol g_{t-1,i})^2\big|\mathcal F_{t-1}\right]}\right]\nonumber\\
    \le & \mathbb E\left[ \sqrt{   {\beta_2^2 \boldsymbol v_{t-2,i}}+\beta_2(1-\beta_2)(D_0+D_1(\partial_i f(\boldsymbol x_{t-1}))^2)}\right]\nonumber\\
    \le & \mathbb E\left[ \sqrt{   {\beta_2^2 \boldsymbol v_{t-2,i}}+\beta_2(1-\beta_2)D_0}\right]+\mathbb E\Big[\sqrt{\beta_2(1-\beta_2)D_1}|\partial_i f(\boldsymbol x_{t-1,i})|\Big],
\end{flalign}
where the second inequality is according to Assumption \ref{assump:variance}. By recursively applying (\ref{proof:3eq1}) we have that
\begin{flalign}
    &\mathbb E\left[ \sqrt{   {\beta_2^2 \boldsymbol v_{t-2}}+\beta_2(1-\beta_2)D_0}\right]\nonumber\\
    \le & \mathbb E\left[ \sqrt{   {\beta_2^3 \boldsymbol v_{t-3,i}}+(\beta_2+\beta_2^2)(1-\beta_2)D_0}\right]+\mathbb E\Big[\sqrt{\beta_2^2(1-\beta_2)D_1}|\partial_i f(\boldsymbol x_{t-2})|\Big].
\end{flalign}
Specifically, we can get that
\begin{flalign}\label{eq:3recursive}
    \mathbb E\Big[\sqrt{{\beta_2 \boldsymbol v_{t-1,i}}}\Big]\le&   \mathbb E\left[ \sqrt{   {\beta_2^2 \boldsymbol v_{t-2,i}}+\beta_2(1-\beta_2)D_0}\right]+\mathbb E\Big[\sqrt{\beta_2(1-\beta_2)D_1}|\partial_i f(\boldsymbol x_{t-1})|\Big]\nonumber\\
    \le&\mathbb E\left[ \sqrt{   {\beta_2^3 \boldsymbol v_{t-3,i}}+(\beta_2+\beta_2^2)(1-\beta_2)D_0}\right]+\mathbb E\Big[\sqrt{\beta_2^2(1-\beta_2)D_1}|\partial_i f(\boldsymbol x_{t-t})|\Big]\nonumber\\
    &+\mathbb E\Big[\sqrt{\beta_2(1-\beta_2)D_1}|\partial_i f(\boldsymbol x_{t-1})|\Big]\nonumber\\
    \le&...\nonumber\\
    \le&\mathbb E\left[ \sqrt{ (\beta_2+\beta_2^2+\beta_2^3+...+\beta_2^t)(1-\beta_2)D_0+\boldsymbol v_{0,i}}\right]\nonumber\\
    &+\sum_{i=1}^{t-1} \mathbb E\Big[ \sqrt{{}\beta_2^{j}(1-\beta_2)D_1}|\partial_i f(\boldsymbol x_{t-j})|\Big]\nonumber\\
    \le& \sqrt{D_0+\|\boldsymbol v_0\|}+\sum_{j=1}^{t-1} \mathbb E\Big[ \sqrt{{}\beta_2^{j}(1-\beta_2)D_1}|\partial_i f(\boldsymbol x_{t-j})|\Big].
\end{flalign}
As a result, we have that
\begin{flalign}\label{eq:3recursive2}
        &\frac{1}{T} \sum_{t=1}^T  \mathbb E\Big[\sqrt{{\beta_2 \|\boldsymbol v_{t-1}}\| + \zeta}\Big]\nonumber\\
        &\le  \sqrt{\zeta}+d\sqrt{D_0+\|\boldsymbol v_0\|}+\frac{1}{T}\sum_{i=1}^d\sum_{t=1}^T \sum_{j=1}^{t-1} \mathbb E\Big[ \sqrt{{}\beta_2^{j}(1-\beta_2)D_1}|\partial_i f(\boldsymbol x_{t-j})|\Big]\nonumber\\
         &=  c+\frac{1}{T}\sum_{i=1}^d\sum_{t=1}^{T-1}  \left(\sum_{j=1}^{T-t}\sqrt{\beta_2^j}\right)\mathbb E\Big[ \sqrt{(1-\beta_2)D_1}|\partial_i f(\boldsymbol x_{t})|\Big]\nonumber\\
        &\le c+\sum_{i=1}^d\frac{\sum_{t=1}^T\mathbb E[|\partial_i f(\boldsymbol x_t)|]}{T} \sqrt{(1-\beta_2)D_1}\Bigg(\sum_{j=1}^T \sqrt{\beta_2^j}\Bigg) \nonumber\\
        &\le c+\frac{2\sqrt{D_1}}{\sqrt{(1-\beta_2)}} \sum_{i=1}^d\frac{\sum_{t=1}^T\mathbb E[|\partial_i f(\boldsymbol x_t)|]}{T}\nonumber\\
        &\le c+\frac{2\sqrt{dD_1}}{\sqrt{(1-\beta_2)}} \frac{\sum_{t=1}^T\mathbb E[\|\nabla f(\boldsymbol x_t)\|]}{T},
    \end{flalign}
where the last inequlity is due to $\sum_{i=1}^d \|\partial_i f(\boldsymbol x)\|\le \sqrt{d}\|f(\boldsymbol x)\|$. We then complete the proof.

\section{Formal Version of Theorem \ref{theorem:1} and Its Proof}\label{proof:theorem1}
Recall that $c=\sqrt{\zeta}+d\sqrt{D_0+\|\boldsymbol v_0\|}$, $\Delta=f(\boldsymbol x_1)-f(\boldsymbol x^*)+\frac{\eta\alpha_0 dD_0 }{2\sqrt{\zeta}}+\frac{\eta \alpha_0D_1\|\nabla f(\boldsymbol x_1)\|^2}{2\sqrt{\zeta}}$. Define $\Lambda_1=\frac{1}{\max{(14L_1\sqrt{d}D_1,7L_1\sqrt{dD_1})}}$, $\Lambda_2=\min\left(\frac{\zeta}{35L_0dD_0},\frac{\sqrt{\zeta}}{35L_1^2d^{1.5}D_0}\right)$ and $\Lambda_3=\frac{\sqrt[4]{\zeta}}{7D_1L_0d\sqrt{5}}$.
\begin{theorem}
    Let Assumptions \ref{assump:lowerbound}, \ref{assump:variance} and \ref{assump:generalsmooth} hold. Let $1-\beta_2=\min\left(\frac{1}{7D_1},\frac{\sqrt{\zeta}\epsilon^2}{35dD_0}\right)= \mathcal O(\epsilon^{2})$, $\eta\le\min\left(\frac{\sqrt{\zeta}}{7L_0D_1},\Lambda_1{\sqrt{1-\beta_2}},\frac{1-\beta_2}{7\sqrt{d}},\Lambda_2{ \epsilon^2},\Lambda_3{\epsilon \sqrt{1-\beta_2}}\right)=\mathcal O(\epsilon^{2})$,  and $T\ge\frac{70\Delta}{\eta\epsilon^2} = \mathcal O(\epsilon^{-4})$. 
    For small $\epsilon$ such that $\epsilon\le \frac{\sqrt{5dD_0}}{\sqrt{D_1}\sqrt[4]{\zeta}}$,
    we have that
 \begin{flalign}
     \frac{1}{T} \sum_{t=1}^T  \mathbb E[\|\nabla f(\boldsymbol x_t)\|]\le \left(\frac{2d\sqrt{35D_0D_1}}{\sqrt[4]{\zeta}}+\sqrt{c}\right)\epsilon.
\end{flalign}
\end{theorem}

\begin{proof}
According to Lemma \ref{co:1}, we have that
\begin{flalign}
        \frac{1}{T} \sum_{t=1}^T \mathbb E\left[\frac{ \|\nabla f(\boldsymbol x_t)\|^2}{\sqrt{{\beta_2 \|\boldsymbol v_{t-1}}\| + \zeta}} \right]\le \epsilon^2.
\end{flalign}

According to Lemma \ref{lemma:2}, we have that
 \begin{flalign}
        &\frac{1}{T} \sum_{t=1}^T  \mathbb E[\sqrt{{\beta_2 \|\boldsymbol v_{t-1}}\| + \zeta}]\nonumber\\
        &\le  c
   +\frac{2\sqrt{dD_1}}{\sqrt{(1-\beta_2)}} \frac{\sum_{t=1}^T\mathbb E[\|\nabla f(\boldsymbol x_t)\|]}{T}.
    \end{flalign}

Define $e=\frac{1}{T} \sum_{t=1}^T  \mathbb E[\|\nabla f(\boldsymbol x_t)\|]$. By H\"older's inequality, we have that
\begin{flalign}
    \left(\frac{1}{T} \sum_{t=1}^T  \mathbb E[\|\nabla f(\boldsymbol x_t)\|]\right)^2\le \left(\frac{1}{T} \sum_{t=1}^T \mathbb E\left[\frac{ \|\nabla f(\boldsymbol x_t)\|^2}{\sqrt{{\beta_2 \|\boldsymbol v_{t-1}}\| + \zeta}} \right]\right)\left(\frac{1}{T} \sum_{t=1}^T  \mathbb E[\sqrt{{\beta_2 \|\boldsymbol v_{t-1}}\| + \zeta}]\right).
\end{flalign}
By Lemma \ref{lemma:2} and Lemma \ref{co:1}, this can be written as 
\begin{flalign}
    e^2\le \epsilon^2\left(c+\frac{2\sqrt{dD_1}}{\sqrt{1-\beta_2}}e\right).
\end{flalign}
Thus we have that $$ \frac{1}{T} \sum_{t=1}^T  \mathbb E[\|\nabla f(x_t)\|]=e\le\frac{\sqrt{dD_1}\epsilon^2}{\sqrt{1-\beta_2}}+\frac{1}{2}\sqrt{\left(\frac{2\sqrt{dD_1}\epsilon^2}{\sqrt{1-\beta_2}}\right)^2+4c\epsilon^2}.  $$
Since $1-\beta_2=\min\left(\frac{1}{7D_1},\frac{\sqrt{\zeta}\epsilon^2}{35dD_0}\right)= \mathcal O(\epsilon^{2})$, if $\epsilon\le \frac{\sqrt{5dD_0}}{\sqrt{D_1}\sqrt[4]{\zeta}}$, we have $\frac{\epsilon^2}{\sqrt{1-\beta_2}} \le \frac{\sqrt{35dD_0}}{\sqrt[4]{\zeta}}\epsilon$. It demonstrates that
\begin{flalign}
     \frac{1}{T} \sum_{t=1}^T  \mathbb E[\|\nabla f(x_t)\|]\le \left(\frac{2d\sqrt{35D_0D_1}}{\sqrt[4]{\zeta}}+\sqrt{c}\right)\epsilon,
\end{flalign}
which completes the proof.
\end{proof}

\section{Lemmas for Theorem \ref{theorem:2}}
Here are some lemmas we need for the proof of Theorem \ref{theorem:2}. Lemma \ref{lemma:4} and Lemma \ref{lemma:5} are from \citet{wang2023closing}. We include their proof for completeness. 
\begin{lemma}\label{lemma:4}
(Lemma 6 in \citet{wang2023closing}) For any $\{c_t\}_{t=0}^\infty\ge0 $ and $a_t=\beta_2 a_{t-1}+(1-\beta_2)c_t^2$ and $b_t=\beta_1b_{t-1}+(1-\beta_1)c_t$, if $0<\beta_1^2<\beta_2<1$, we have that
\begin{flalign}
    \frac{b_t}{\sqrt{a_t+\zeta}}\le \frac{1-\beta_1}{\sqrt{1-\beta_2}\sqrt{1-\frac{\beta_1^2}{{\beta_2}}}}.
\end{flalign}
\end{lemma}
\begin{proof}
    We can show that 
    \begin{flalign}
        \frac{b_t}{\sqrt{a_t+\zeta}}\le\frac{b_t}{\sqrt{a_t}}&\le \frac{\sum_{i=0}^{t-1}(1-\beta_1)\beta_1^ic_{t-i}}{\sqrt{\sum_{i=0}^{t-1}(1-\beta_2)\beta_2^ic_{t-i}^2+a_0}}\nonumber\\
        &\le \frac{1-\beta_1}{\sqrt{1-\beta_2}}\frac{\sqrt{\sum_{i=0}^{t-1}\beta_2^ic_{t-i}^2}\sqrt{\sum_{i=0}^{t-1}\frac{\beta_1^{2i}}{\beta_2^i}}}{\sqrt{\sum_{i=0}^{t-1}\beta_2^ic_{t-i}^2}}
        \le \frac{1-\beta_1}{\sqrt{1-\beta_2}\sqrt{1-\frac{\beta_1^2}{{\beta_2}}}},
    \end{flalign}
where the third inequality is due to the Cauchy-Schwarz inequality.
\end{proof}

\begin{lemma}\label{lemma:5}
(Lemma 5 in \citet{wang2023closing}) For any $\{c_t\}_{t=0}^\infty\ge0 $ and $a_t=\beta_2 a_{t-1}+(1-\beta_2)c_t^2$ and $b_t=\beta_1b_{t-1}+(1-\beta_1)c_t$, if $0<\beta_1^2<\beta_2<1$ and $a_0>0$, we have that
\begin{flalign}
    \sum_{t=1}^T\frac{b_t^2}{{a_t}}\le \frac{(1-\beta_1)^2}{(1-\frac{\beta_1}{\sqrt{\beta_2}})^2(1-\beta_2)}\left(\ln\left(\frac{a_T}{a_0}\right)-T\ln(\beta_2)\right).
\end{flalign}
\end{lemma}
\begin{proof}
    Due to the monotonicity of the $\frac{1}{a}$ function and Lemma 5.2 in \citet{defossez2020simple}, we have that
    \begin{flalign}
        \frac{(1-\beta_2)c_t^2}{a_t}\le \int_{a=a_t-(1-\beta_2)c_t^2}^{a_t} \frac{1}{a}da=\ln\left(\frac{a_t}{a_t-(1-\beta_2)c_t^2}\right)=\ln\left(\frac{a_t}{a_{t-1}}\right)-\ln(\beta_2).
    \end{flalign}
    By telescoping, we have  $\sum_{t=1}^T \frac{c_t^2}{a_t}\le \frac{1}{1-\beta_2}\ln\left(\frac{a_T}{a_{0}}\right)-T\frac{\ln(\beta_2)}{(1-\beta_2)}$. Moreover, for $b_t$ we can show that
    \begin{flalign}
    \frac{b_t}{\sqrt{a_t}}\le (1-\beta_1)\sum_{i=1}^t\frac{\beta_1^{t-i}c_i}{\sqrt{a_t}}
    \le (1-\beta_1)\sum_{i=1}^t\left(\frac{\beta_1}{\sqrt{\beta_2}}\right)^{t-i}\frac{c_i}{\sqrt{a_i}}.
\end{flalign}
    
Further applying the Cauchy-Schwarz inequality, we get that
\begin{flalign}
    \frac{b_t^2}{{a_t}}
    \le& (1-\beta_1)^2\left(\sum_{i=1}^t\left(\frac{\beta_1}{\sqrt{\beta_2}}\right)^{t-i}\frac{c_i}{\sqrt{a_i}}\right)^2\nonumber\\
    \le &(1-\beta_1)^2\left(\sum_{i=1}^t\left(\frac{\beta_1}{\sqrt{\beta_2}}\right)^{t-i}\right)\left(\sum_{i=1}^t\left({\frac{\beta_1}{\sqrt{\beta_2}}}\right)^{t-i} \frac{c_i^2}{{a_i}}\right)\nonumber\\
    \le & \frac{(1-\beta_1)^2}{\left(1-\frac{\beta_1}{\sqrt{\beta_2}}\right)}\left(\sum_{i=1}^t\left(\frac{\beta_1}{\sqrt{\beta_2}}\right)^{t-i} \frac{c_i^2}{{a_i}}\right).
\end{flalign}
This further implies that
\begin{flalign}
    \sum_{t=1}^T\frac{b_t^2}{{a_t}}\le \frac{(1-\beta_1)^2}{(1-\frac{\beta_1}{\sqrt{\beta_2}})^2(1-\beta_2)}\left(\ln\left(\frac{a_T}{a_0}\right)-T\ln(\beta_2)\right).
\end{flalign}
\end{proof}

\begin{lemma}\label{lemma:6}
 For any $\{c_t\}_{t=0}^\infty\ge0 $ and $a_t=\beta_2 a_{t-1}+(1-\beta_2)c_t^2$ and $b_t=\beta_1b_{t-1}+(1-\beta_1)c_t$, if $0<\beta_1^4<\beta_2<1$ and $a_0>0$, we have that
\begin{flalign}
    \sum_{t=1}^T\frac{b_t^2}{\sqrt{a_t}}\le  \frac{(1-\beta_1)^2}{(1-\frac{\beta_1}{\sqrt[4]{\beta_2}})^2}\left(\frac{2}{1-\beta_2}(\sqrt{a_T}-\sqrt{a_{0}})+\sum_{t=1}^T2\sqrt{a_{t-1}}\right).
\end{flalign}
\end{lemma}
\begin{proof}
    Due to the monotonicity of the $\frac{1}{\sqrt{a}}$ function, we can show that
    \begin{flalign}
        \frac{(1-\beta_2)c_t^2}{\sqrt{a}_t}&\le \int_{a=a_t-(1-\beta_2)c_t^2}^{a_t} \frac{1}{\sqrt{a}}da=2\sqrt{a_t}-2\sqrt{a_t-(1-\beta_2)c_t^2}\nonumber\\
        &=2\sqrt{a_t}-2\sqrt{\beta_2a_{t-1}}\le 2\sqrt{a_t}-2\sqrt{a_{t-1}}+2(1-\beta_2)\sqrt{a_{t-1}}.
    \end{flalign}
    Taking sum up from $t=1$ to $T$, we obtain the following: 
    \begin{flalign}\label{lemma:e3}
    \sum_{t=1}^T \frac{(1-\beta_2)c_t^2}{\sqrt{a_t}}\le  2\sqrt{a_T}-2\sqrt{a_{0}}+2(1-\beta_2)\sum_{t=1}^T\sqrt{a_{t-1}}.
    \end{flalign}
    We then derive the following bound:
\begin{flalign}
    \frac{b_t}{\sqrt[4]{a_t}}\le (1-\beta_1)\sum_{i=1}^t\frac{\beta_1^{t-i}c_i}{\sqrt[4]{a_t}}
    \le (1-\beta_1)\sum_{i=1}^t\left(\frac{\beta_1}{\sqrt[4]{\beta_2}}\right)^{t-i}\frac{c_i}{\sqrt[4]{a_i}}.
\end{flalign}
It then follows that
\begin{flalign}
    \frac{b_t^2}{\sqrt{a_t}}
    \le& (1-\beta_1)^2\left(\sum_{i=1}^t\left(\frac{\beta_1}{\sqrt[4]{\beta_2}}\right)^{t-i}\frac{c_i}{\sqrt[4]{a_i}}\right)^2\nonumber\\
    \le &(1-\beta_1)^2\left(\sum_{i=1}^t\left(\frac{\beta_1}{\sqrt[4]{\beta_2}}\right)^{t-i}\right)\left(\sum_{i=1}^t\left({\frac{\beta_1}{\sqrt[4]{\beta_2}}}\right)^{t-i} \frac{c_i^2}{\sqrt{a_i}}\right)\nonumber\\
    \le & \frac{(1-\beta_1)^2}{1-\frac{\beta_1}{\sqrt[4]{\beta_2}}}\left(\sum_{i=1}^t\left({\frac{\beta_1}{\sqrt[4]{\beta_2}}}\right)^{t-i} \frac{c_i^2}{\sqrt{a_i}}\right).
\end{flalign}
Taking the sum from $t=1$ to $T$, we can derive that
\begin{flalign}\label{eq:proof59}
    \sum_{t=1}^T\frac{b_t^2}{\sqrt{a_t}}\le  \frac{(1-\beta_1)^2}{(1-\frac{\beta_1}{\sqrt[4]{\beta_2}})^2}\sum_{t=1}^T\frac{c_t^2}{\sqrt{a_t}}.
\end{flalign}
Plug \eqref{lemma:e3} in \eqref{eq:proof59}, and we complete the proof.
\end{proof}

\section{Lemma \ref{formal:lemma3} and Its Proof}\label{proof:lemma3}
Define $C_1=1-\frac{\beta_1}{\sqrt{\beta_2}},C_2=\sqrt{1-\frac{\beta_1^2}{\beta_2}}$, we have the following lemma:
\begin{lemma}\label{formal:lemma3}
     For any $\alpha_0,\alpha_1,\alpha_3, \alpha_4>0$, $\beta_2>0.5,$ and $0<\beta_1^2<\beta_2$,  we have that
\begin{flalign}
 &{\mathbb E[\langle \nabla f(\boldsymbol u_t), \boldsymbol u_{t}-\boldsymbol u_{t+1} \rangle|\mathcal F_t]}\nonumber\\  
 \ge& \Bigg(\frac{\eta (1-\beta_1)}{C_1}-\frac{\eta(1-\beta_1)}{C_1C_2}\left(\frac{1}{2\alpha_0}+\frac{\alpha_0}{2\alpha_1}+\frac{\eta \alpha_0\sqrt{d}D_1L_1(1-\beta_1)}{\sqrt{1-\beta_2}C_2}\right)\nonumber\\
 &-\frac{\eta \beta_1(1-\beta_1)\sqrt{\zeta}}{2\alpha_3C_1C_2}-\frac{\eta L_1(1-C_1)^2}{2\alpha_4C_1^2}-\frac{\eta L_1(1-C_1)}{2\alpha_4C_1^2}\Bigg)\times \sum_{i=1}^d\frac{(\partial_i f(\boldsymbol x_t))^2}{\sqrt{{\beta_2\boldsymbol v_{t-1,i}} + \zeta}}\nonumber\\
 &-\sum_{i=1}^d\frac{\eta (1-\beta_1)}{C_1C_2}\Bigg(\frac{\alpha_0D_0}{2}\Bigg(\frac{1}{\sqrt{\beta_2 \boldsymbol v_{t-1,i}+\zeta}}-\mathbb E\Bigg[{\frac{1}{\sqrt{\boldsymbol v_{t,i}+\zeta}}}\Bigg|\mathcal F_t\Bigg]\Bigg)\nonumber\\
 &+\frac{\alpha_0 D_1}{2}\mathbb E\left[\frac{(\partial_i  f(\boldsymbol x_{t-1}))^2}{\sqrt{{\beta_2 \boldsymbol v_{t-1,i}} + \zeta}}-\frac{(\partial_i f(\boldsymbol x_{t}))^2}{\sqrt{{ \boldsymbol v_{t,i}} + \zeta}}\Bigg|\mathcal F_t\right]\Bigg)\nonumber\\
 &-\frac{\alpha_0\alpha_1D_1^2L_0^2\eta^3d^2(1-\beta_1)^3}{2(1-\beta_2)C_1C_2^3\sqrt{\zeta})}-\frac{\alpha_3\eta d\beta_1(1-\beta_1)(1-\beta_2)}{2C_1C_2}\nonumber\\
 &-\sum_{i=1}^d\frac{\eta^2((1-C_1)^2+0.5(1-C_1))\sqrt{d}L_0}{C_1^2}  \frac{\boldsymbol m_{t-1,i}^2}{\boldsymbol v_{t-1,i}+\zeta}-\sum_{i=1}^d\frac{\eta^20.5(1-C_1)\sqrt{d}L_0}{C_1^2}\mathbb E\Bigg[ \frac{\boldsymbol m_{t,i}^2}{\boldsymbol v_{t,i}+\zeta}\Bigg|\mathcal F_t\Bigg]\nonumber\\
 &-\sum_{i=1}^d\frac{\alpha_4\eta^3(1-\beta_1)^2(1-C_1)^2dL_1}{2(1-\beta_2)C_1^2C_2^2}\frac{\boldsymbol m_{t-1,i}^2}{\sqrt{\boldsymbol v_{t-1,i}+\zeta}}-\sum_{i=1}^d\frac{\alpha_4\eta^3(1-\beta_1)^2(1-C_1)dL_1}{2(1-\beta_2)C_1^2C_2^2}\mathbb E\left[\frac{\boldsymbol m_{t,i}^2}{\sqrt{\boldsymbol v_{t,i}+\zeta}}\Bigg|\mathcal F_t\right].
\end{flalign}
\end{lemma}
\begin{proof}
Since $\boldsymbol u_t=\frac{\boldsymbol x_t-\frac{\beta_1}{\sqrt{\beta_2}}\boldsymbol x_{t-1}}{1-\frac{\beta_1}{\sqrt{\beta_2}}}$, we then have that
\begin{flalign}\label{eq:u}
    \boldsymbol u_{t+1}-\boldsymbol u_t=&\frac{\boldsymbol x_{t+1}-\boldsymbol x_t}{1-\frac{\beta_1}{\sqrt{\beta_2}}}-\frac{\beta_1}{\sqrt{\beta_2}}\frac{\boldsymbol x_{t}-\boldsymbol x_{t-1}}{1-\frac{\beta_1}{\sqrt{\beta_2}}}\nonumber\\
    =&\frac{ - \eta \odot \frac{\boldsymbol m_t}{\sqrt{{\boldsymbol v_t} + \zeta}}}{1-\frac{\beta_1}{\sqrt{\beta_2}}}+\frac{\beta_1}{\sqrt{\beta_2}}\frac{ \eta \odot \frac{\boldsymbol m_{t-1}}{\sqrt{{\boldsymbol v_{t-1}} + \zeta}}}{1-\frac{\beta_1}{\sqrt{\beta_2}}}\nonumber\\
    =&\frac{ - \eta \odot \frac{\boldsymbol m_t}{\sqrt{{\beta_2\boldsymbol v_{t-1}} + \zeta}}}{1-\frac{\beta_1}{\sqrt{\beta_2}}}+\frac{ - \eta \odot \frac{\boldsymbol m_t}{\sqrt{{\boldsymbol v_t} + \zeta}}+ \eta \odot \frac{\boldsymbol m_t}{\sqrt{{\beta_2\boldsymbol v_{t-1}} + \zeta}}}{1-\frac{\beta_1}{\sqrt{\beta_2}}}\nonumber\\
    &+\frac{ \eta \odot \frac{\beta_1\boldsymbol m_{t-1}}{\sqrt{\beta_2{\boldsymbol v_{t-1}} + \zeta}}}{1-\frac{\beta_1}{\sqrt{\beta_2}}}+\frac{ \eta \odot \frac{\beta_1\boldsymbol m_{t-1}}{\sqrt{\beta_2{\boldsymbol v_{t-1}} + \beta_2\zeta}}- \eta  \odot\frac{\beta_1\boldsymbol m_{t-1}}{\sqrt{\beta_2{\boldsymbol v_{t-1}} + \zeta}}}{1-\frac{\beta_1}{\sqrt{\beta_2}}}\nonumber\\
    =&\frac{ - \eta \odot \frac{(1-\beta_1)\boldsymbol g_t}{\sqrt{{\beta_2\boldsymbol v_{t-1}} + \zeta}}}{1-\frac{\beta_1}{\sqrt{\beta_2}}}+\frac{ - \eta \odot \frac{\boldsymbol m_t}{\sqrt{{\boldsymbol v_t} + \zeta}}+ \eta \odot \frac{\boldsymbol m_t}{\sqrt{{\beta_2\boldsymbol v_{t-1}} + \zeta}}}{1-\frac{\beta_1}{\sqrt{\beta_2}}}+\frac{ \eta\odot  \frac{\beta_1\boldsymbol m_{t-1}}{\sqrt{\beta_2{\boldsymbol v_{t-1}} + \beta_2\zeta}}- \eta\odot  \frac{\beta_1\boldsymbol m_{t-1}}{\sqrt{\beta_2{\boldsymbol v_{t-1}} + \zeta}}}{1-\frac{\beta_1}{\sqrt{\beta_2}}},
\end{flalign}
where the last equality is due to $\boldsymbol m_t=\beta_1 \boldsymbol m_{t-1}+(1-\beta_1)\boldsymbol g_t$.
For $(L_0,L_1)$-smooth objectives, from Lemma 1 in \citet{crawshaw2022robustness} we can get that
\begin{flalign}
    \underbrace{\mathbb E[\langle \nabla f(\boldsymbol u_t), \boldsymbol u_{t}-\boldsymbol u_{t+1} \rangle|\mathcal F_t]}_{\text{first-order}}&\le f(\boldsymbol u_t)-\mathbb E[f(\boldsymbol u_{t+1})|\mathcal F_t]
  +\underbrace{\sum_{i=1}^d\frac{L_0}{2\sqrt{d}}\mathbb E[\|\boldsymbol u_{t+1}-\boldsymbol u_t\||\boldsymbol u_{t+1,i}-\boldsymbol u_{t,i}||\mathcal F_t]}_{\text{second-order}}\nonumber\\
    &+\underbrace{\sum_{i=1}^d\frac{L_1|\partial_i f(\boldsymbol u_t)|}{2}\mathbb E[\|\boldsymbol u_{t+1}-\boldsymbol u_t\||\boldsymbol u_{t+1,i}-\boldsymbol u_{t,i}||\mathcal F_t]}_{\text{additional term}}.
\label{eq:generalsmoothforu}
\end{flalign}
We then focus on the first-order term. Based on \eqref{eq:u}, we divide our first-order term into four parts:
\begin{flalign}\label{eq:6eq8}
    {\mathbb E[\langle \nabla f(\boldsymbol u_t), \boldsymbol u_{t}-\boldsymbol u_{t+1} \rangle|\mathcal F_t]}=& {\mathbb E[\langle \nabla f(\boldsymbol x_t), \boldsymbol u_{t}-\boldsymbol u_{t+1} \rangle|\mathcal F_t]}+ {\mathbb E\left[\left\langle \nabla f(\boldsymbol u_t)-\nabla f(\boldsymbol x_t), \boldsymbol u_{t}-\boldsymbol u_{t+1} \right\rangle|\mathcal F_t\right]}\nonumber\\
    =& \underbrace{{\mathbb E\left[\left\langle \nabla f(\boldsymbol x_t),\frac{  \eta\odot  \frac{(1-\beta_1)\boldsymbol g_t}{\sqrt{{\beta_2\boldsymbol v_{t-1}} + \zeta}}}{1-\frac{\beta_1}{\sqrt{\beta_2}}} \right\rangle\Bigg|\mathcal F_t\right]}}_{\text{first-order main}}\nonumber\\
    &-\underbrace{{\mathbb E\left[\left\langle \nabla f(\boldsymbol x_t),\frac{ - \eta \odot \frac{\boldsymbol m_t}{\sqrt{{\boldsymbol v_t} + \zeta}}+ \eta\odot  \frac{\boldsymbol m_t}{\sqrt{{\beta_2\boldsymbol v_{t-1}} + \zeta}}}{1-\frac{\beta_1}{\sqrt{\beta_2}}}\right\rangle\Bigg|\mathcal F_t\right]}}_{\text{error 1}}\nonumber\\
    &-\underbrace{{\left\langle \nabla f(\boldsymbol x_t),\frac{ \eta \odot \frac{\beta_1\boldsymbol m_{t-1}}{\sqrt{\beta_2{\boldsymbol v_{t-1}} + \beta_2\zeta}}- \eta \odot \frac{\beta_1\boldsymbol m_{t-1}}{\sqrt{\beta_2{\boldsymbol v_{t-1}} + \zeta}}}{1-\frac{\beta_1}{\sqrt{\beta_2}}} \right\rangle}}_{\text{error 2}}\nonumber\\
    &- \underbrace{{\mathbb E[\langle \nabla f(\boldsymbol u_t)-\nabla f(\boldsymbol x_t), \boldsymbol u_{t+1}-\boldsymbol u_{t} \rangle|\mathcal F_t]}}_{\text{error 3}}.
\end{flalign}
The first-order main term is easy to bound. Since given $\mathcal F_t$, $\boldsymbol g_t$ is independent of $\boldsymbol v_{t-1}$ and we have that
\begin{flalign}\label{eq:6eq9}
    \underbrace{{\mathbb E\left[\left\langle \nabla f(\boldsymbol x_t),\frac{  \eta \odot \frac{(1-\beta_1)\boldsymbol g_t}{\sqrt{{\beta_2\boldsymbol v_{t-1}} + \zeta}}}{1-\frac{\beta_1}{\sqrt{\beta_2}}} \right\rangle\Bigg|\mathcal F_t\right]}}_{\text{first-order main}}=\sum_{i=1}^d\frac{\eta(1-\beta_1)}{1-\frac{\beta_1}{\sqrt{\beta_2}}}\frac{(\partial_i f(\boldsymbol x_t))^2}{\sqrt{{\beta_2\boldsymbol v_{t-1,i}} + \zeta}}.
\end{flalign}
We then focus on the error 1 term. Since $\boldsymbol v_t=\beta_2 \boldsymbol v_{t-1}+(1-\beta_2)\boldsymbol g_t\odot \boldsymbol g_t$, we {have \eqref{eq:lemma1eq1}} and it follows that
\begin{flalign}
    &\underbrace{{\mathbb E\left[\left\langle \nabla f(\boldsymbol x_t),\frac{ - \eta \odot \frac{\boldsymbol m_t}{\sqrt{{\boldsymbol v_t} + \zeta}}+ \eta\odot  \frac{\boldsymbol m_t}{\sqrt{{\beta_2\boldsymbol v_{t-1}} + \zeta}}}{1-\frac{\beta_1}{\sqrt{\beta_2}}}\right\rangle\Bigg|\mathcal F_t\right]}}_{\text{error 1}}\nonumber\\
    \le& \sum_{i=1}^d \frac{\eta}{1-\frac{\beta_1}{\sqrt{\beta_2}}}{\mathbb E\left[ |\partial_i f(\boldsymbol x_t)| \frac{(1-\beta_2)\boldsymbol g_{t,i}^2|\boldsymbol m_{t,i}|}{(\sqrt{{ \boldsymbol v_{t,i}} + \zeta})(\sqrt{{\beta_2 \boldsymbol v_{t-1,i}} + \zeta})(\sqrt{{ \boldsymbol v_{t,i}}+\zeta}+\sqrt{{\beta_2 \boldsymbol v_{t-1,i}}+\zeta})}\Bigg|\mathcal F_t\right]}.
\end{flalign}
According to Lemma \ref{lemma:4}, we have 
  $  \frac{|\boldsymbol m_{t,i}|}{\sqrt{\boldsymbol v_{t,i}+\zeta}}\le \frac{1-\beta_1}{\sqrt{1-\beta_2}\sqrt{1-\frac{\beta_1^2}{{\beta_2}}}}$, which demonstrates that
  \begin{flalign}
   \text{error 1}
    \le &\sum_{i=1}^d \frac{\eta}{1-\frac{\beta_1}{\sqrt{\beta_2}}} 
  \frac{1-\beta_1}{\sqrt{1-\frac{\beta_1^2}{\beta_2}}} 
  \frac{|\partial_i f(\boldsymbol x_t)|}{\sqrt{{\beta_2 \boldsymbol v_{t-1,i}} + \zeta}}\mathbb E\left[\left\| \frac{\sqrt{1-\beta_2}\boldsymbol g_{t,i}^2}{(\sqrt{{ \boldsymbol v_{t,i}} + \zeta}+\sqrt{{\beta_2 \boldsymbol v_{t-1,i}} + \zeta})}\right\|\Bigg|\mathcal F_t\right].
\end{flalign}
Similar to the proof in Appendix \ref{proof:lemma1},  we can show that for any $\alpha_0>0$ and $i\in[d]$,
\begin{flalign}\label{proof:6eq2}
     &  \frac{|\partial_i f(\boldsymbol x_t)|}{\sqrt{{\beta_2 \boldsymbol v_{t-1,i}} + \zeta}}\mathbb E\left[\left\| \frac{\sqrt{1-\beta_2}\boldsymbol g_{t,i}^2}{(\sqrt{{ \boldsymbol v_{t,i}} + \zeta}+\sqrt{{\beta_2 \boldsymbol v_{t-1,i}} + \zeta})}\right\|\Bigg|\mathcal F_t\right]\nonumber\\
     \le &\frac{ (\partial_i f(\boldsymbol x_t))^2}{2\alpha_0\sqrt{{\beta_2 \boldsymbol v_{t-1,i}} + \zeta}}+\frac{\alpha_0 }{2\sqrt{{\beta_2 \boldsymbol v_{t-1,i}} + \zeta}}\left(\mathbb E\left[\frac{\sqrt{1-\beta_2}\boldsymbol g_{t,i}^2}{(\sqrt{{ \boldsymbol v_{t,i}} + \zeta}+\sqrt{{\beta_2 \boldsymbol v_{t-1,i}} + \zeta})}\Bigg|\mathcal F_t\right]\right)^2.
\end{flalign}
For the last term, using H\"older's inequality, we have that 
\begin{flalign}\label{proof:6eq3}
    &\frac{\alpha_0 }{2\sqrt{{\beta_2 \boldsymbol v_{t-1,i}} + \zeta}}\left(\mathbb E\left[\frac{\sqrt{1-\beta_2}\boldsymbol g_{t,i}^2}{(\sqrt{{ \boldsymbol v_{t,i}} + \zeta}+\sqrt{{\beta_2 \boldsymbol v_{t-1,i}} + \zeta})}\Bigg|\mathcal F_t\right]\right)^2\nonumber\\
    \le& \frac{(1-\beta_2)\alpha_0}{2\sqrt{{\beta_2 \boldsymbol v_{t-1,i}} + \zeta}}\mathbb E[\boldsymbol g_{t,i}^2|\mathcal F_t]\mathbb E\left[\frac{\boldsymbol g_{t,i}^2}{(\sqrt{{ \boldsymbol v_{t,i}} + \zeta}+\sqrt{{\beta_2 \boldsymbol v_{t-1,i}} + \zeta})^2}\Bigg|\mathcal F_t\right]\nonumber\\
        \le & \frac{\alpha_0}{2}(D_0+D_1 (\partial_i f(\boldsymbol x_t))^2)\times 
    \mathbb E\left[\frac{(1-\beta_2)\boldsymbol g_{t,i}^2}{\sqrt{{ \boldsymbol v_{t,i}} + \zeta}\sqrt{{\beta_2 \boldsymbol v_{t-1,i}} + \zeta}(\sqrt{{ \boldsymbol v_{t,i}} + \zeta}+\sqrt{{\beta_2 \boldsymbol v_{t-1,i}} + \zeta})}\Bigg|\mathcal F_t\right]\nonumber\\
    \le &\frac{\alpha_0}{2}(D_0+D_1 (\partial_i f(\boldsymbol x_t))^2)\mathbb E\left[\frac{1}{\sqrt{{\beta_2  \boldsymbol v_{t-1,i}} + \zeta}}-\frac{1}{\sqrt{{ \boldsymbol v_{t,i}} + \zeta}}\Bigg|\mathcal F_t\right].\nonumber\\
\end{flalign}
Similar to \eqref{proof:eq4}, we then show the sum of \eqref{proof:6eq3} from $t=1$ to $T$ can be bounded. For $t>1$ we have that
\begin{flalign}\label{proof:6eq4}
& \frac{\alpha_0}{2}(D_0+D_1 (\partial_i f(\boldsymbol x_t))^2)\mathbb E\left[\frac{1}{\sqrt{{\beta_2 \boldsymbol v_{t-1,i}} + \zeta}}-\frac{1}{\sqrt{{ \boldsymbol v_{t,i}} + \zeta}}\Bigg|\mathcal F_t\right]\nonumber\\
    = &\frac{\alpha_0 D_0}{2}\mathbb E\left[\frac{1}{\sqrt{{\beta_2 \boldsymbol v_{t-1,i}} + \zeta}}-\frac{1}{\sqrt{{ \boldsymbol v_{t,i}} + \zeta}}\Bigg|\mathcal F_t\right]+\frac{\alpha_0 D_1}{2}\mathbb E\left[\frac{(\partial_i f(\boldsymbol x_{t-1}))^2}{\sqrt{{\beta_2 \boldsymbol v_{t-1,i}} + \zeta}}-\frac{(\partial_i f(\boldsymbol x_{t}))^2}{\sqrt{{ \boldsymbol v_{t,i}} + \zeta}}\Bigg|\mathcal F_t\right]\nonumber\\
    &+ \frac{\alpha_0 D_1}{2}\mathbb E\left[\frac{(\partial_i f(\boldsymbol x_t))^2-(\partial_i f(\boldsymbol x_{t-1}))^2}{\sqrt{{\beta_2 \boldsymbol v_{t-1,i}} + \zeta}}\Bigg|\mathcal F_t\right].
\end{flalign}
Note for the last term on RHS of \eqref{proof:6eq4}, due to the different update of $\boldsymbol x_t$, there is a little difference from \eqref{proof:eq5}. For the last term and any $\alpha_1 >0$ and $i\in [d]$, we have that
\begin{flalign}
    &\frac{\alpha_0 D_1}{2}\mathbb E\left[\frac{(\partial_i f(\boldsymbol x_t))^2-(\partial_i f(\boldsymbol x_{t-1}))^2}{\sqrt{{\beta \boldsymbol v_{t-1,i}} + \zeta}}\Bigg|\mathcal F_t\right]\nonumber\\
    \le& \frac{\alpha_0 D_1}{2}\mathbb E\left[\frac{2|\partial_i f(\boldsymbol x_t)|\big|\partial_i f(\boldsymbol x_t)-\partial_i f(\boldsymbol x_{t-1})\big|}{\sqrt{{\beta \boldsymbol v_{t-1,i}} + \zeta}}\Bigg|\mathcal F_t\right] \nonumber\\
    \le &\alpha_0 D_1\frac{|\partial_i f(\boldsymbol x_t)|(L_0+L_1|\partial_i f(\boldsymbol x_t)|)\eta\left\|\frac{1}{\sqrt{\boldsymbol v_{t-1}+\zeta}}\odot \boldsymbol m_{t-1}\right\|}{\sqrt{{\beta \boldsymbol v_{t-1,i}} + \zeta}}\nonumber\\
    \le & \frac{\alpha_0 D_1}{2\sqrt{{\beta \boldsymbol v_{t-1,i}} + \zeta}}
    \Bigg(\frac{(\partial_i f(\boldsymbol x_t))^2}{\alpha_1 D_1}+\alpha_1 D_1L_0^2\eta^2\left\|\frac{1}{\boldsymbol v_{t-1}+\zeta}\odot \boldsymbol m_{t-1}\right\|^2\nonumber\\
    &+2\eta L_1(\partial_i f(\boldsymbol x_t))^2\left\|\frac{1}{\boldsymbol v_{t-1}+\zeta}\odot \boldsymbol m_{t-1}\right\|\Bigg)\nonumber\\
    \le  & \frac{\alpha_0 D_1}{2\sqrt{{\beta \boldsymbol v_{t-1,i}} + \zeta}}\left(\frac{(\partial_i f(\boldsymbol x_t))^2}{\alpha_1 D_1}+{\alpha_1 D_1L_0^2\eta^2}{}\frac{d(1-\beta_1)^2}{{(1-\beta_2)}{(1-\frac{\beta_1^2}{\beta_2}})}+2\eta L_1(\partial_i f(\boldsymbol x_t))^2\frac{\sqrt{d}(1-\beta_1)}{\sqrt{1-\beta_2}\sqrt{1-\frac{\beta_1^2}{\beta_2}}}\right),
\end{flalign}
where the first inequality is due to the fact that for any $a,b$, we have $a^2-b^2\le 2|a||a-b|$, the second inequality is by the $(L_0,L_1)$-smoothness assumption and update process of $x_t$, and the last inequality is due to Lemma \ref{lemma:4}. As a result, for $\alpha_0,\alpha_1>0$, if $t=1$, it can be shown that
  \begin{flalign}
   \text{error 1}
    \le & \sum_{i=1}^d\frac{\eta}{1-\frac{\beta_1}{\sqrt{\beta_2}}} 
  \frac{1-\beta_1}{\sqrt{1-\frac{\beta_1^2}{\beta_2}}} 
 \Bigg(\frac{ (\partial_i f(\boldsymbol x_t))^2}{2\alpha_0\sqrt{{\beta_2 \boldsymbol v_{t-1,i}} + \zeta}}+\frac{\alpha_0 D_0}{2}\mathbb E\left[\frac{1}{\sqrt{{\beta_2 \boldsymbol v_{t-1,i}} + \zeta}}-\frac{1}{\sqrt{{ \boldsymbol v_{t,i}} + \zeta}}\Bigg|\mathcal F_t\right]\nonumber\\
 &+\frac{\alpha_0 D_1}{2}\mathbb E\left[\frac{(\partial_i f(\boldsymbol x_{t}))^2}{\sqrt{{\beta_2 \boldsymbol v_{t-1,i}} + \zeta}}-\frac{(\partial_i f(\boldsymbol x_{t}))^2}{\sqrt{{ \boldsymbol v_{t,i}} + \zeta}}\Bigg|\mathcal F_t\right]\Bigg),
\end{flalign}
and if $t>1$ we have that
  \begin{flalign}\label{proof:6eq5}
   \text{error 1}
    \le &\sum_{i=1}^d \frac{\eta}{1-\frac{\beta_1}{\sqrt{\beta_2}}} 
  \frac{1-\beta_1}{\sqrt{1-\frac{\beta_1^2}{\beta_2}}} 
 \Bigg(\frac{ (\partial_i f(\boldsymbol x_t))^2}{2\alpha_0\sqrt{{\beta_2 \boldsymbol v_{t-1,i}} + \zeta}}+\frac{\alpha_0 D_0}{2}\mathbb E\left[\frac{1}{\sqrt{{\beta_2 \boldsymbol v_{t-1,i}} + \zeta}}-\frac{1}{\sqrt{{ \boldsymbol v_{t,i}} + \zeta}}\Bigg|\mathcal F_t\right]\nonumber\\
 &+\frac{\alpha_0 D_1}{2}\mathbb E\left[\frac{(\partial_i f(\boldsymbol x_{t-1}))^2}{\sqrt{{\beta_2 \boldsymbol v_{t-1,i}} + \zeta}}-\frac{(\partial_i f(\boldsymbol x_{t}))^2}{\sqrt{{ \boldsymbol v_{t,i}} + \zeta}}\Bigg|\mathcal F_t\right]\nonumber\\
 &+\frac{\alpha_0 D_1}{2\sqrt{{\beta \boldsymbol v_{t-1,i}} + \zeta}}\Bigg(\frac{(\partial_i f(\boldsymbol x_t))^2}{\alpha_1 D_1}+{\alpha_1 D_1L_0^2\eta^2}{}\frac{d(1-\beta_1)^2}{{(1-\beta_2)}{(1-\frac{\beta_1^2}{\beta_2}})}\nonumber\\
 &+2\eta L_1(\partial_i f(\boldsymbol x_t))^2\frac{\sqrt{d}(1-\beta_1)}{\sqrt{1-\beta_2}\sqrt{1-\frac{\beta_1^2}{\beta_2}}}\Bigg)\Bigg).
\end{flalign}
Since we define $\boldsymbol m_0=0$, and $\boldsymbol x_{0}=\boldsymbol x_1 $, therefore this inequality holds for $t=1$.
Note many terms in the RHS of \eqref{proof:6eq5} can be reduced by telescoping, which will be demonstrated later. Now we move to the error 2 term. For any $\alpha_3>0,0<\beta_1^2<\beta_2$ and $\beta_2>0.5$, we have that
\begin{flalign}\label{eq:6eq10}
    &\underbrace{{\left\langle \nabla f(\boldsymbol x_t),\frac{ \eta \odot \frac{\beta_1\boldsymbol m_{t-1}}{\sqrt{\beta_2{\boldsymbol v_{t-1}} + \beta_2\zeta}}- \eta \odot \frac{\beta_1\boldsymbol m_{t-1}}{\sqrt{\beta_2{\boldsymbol v_{t-1}} + \zeta}}}{1-\frac{\beta_1}{\sqrt{\beta_2}}} \right\rangle}}_{\text{error 2}}\nonumber\\
    = & \frac{\eta \beta_1}{1-\frac{\beta_1}{\sqrt{\beta_2}}}{\left\langle \nabla f(\boldsymbol x_t), \frac{(1-\beta_2)\zeta }{(\sqrt{\beta_2{\boldsymbol v_{t-1}} + \zeta})\odot(\sqrt{\beta_2{\boldsymbol v_{t-1}} + \beta_2\zeta})\odot(\sqrt{\beta_2{\boldsymbol v_{t-1}} + \zeta}+\sqrt{\beta_2{v_{t-1}} + \beta_2\zeta})}\odot\boldsymbol m_{t-1}\right\rangle}\nonumber\\
    \le &\sum_{i=1}^d\frac{\eta \beta_1}{1-\frac{\beta_1}{\sqrt{\beta_2}}} |\partial_i f(\boldsymbol x_t)|
     \frac{(1-\beta_2)\zeta |\boldsymbol m_{t-1,i}|}{(\sqrt{\beta_2{\boldsymbol v_{t-1,i}} + \zeta})(\sqrt{\beta_2{\boldsymbol v_{t-1,i}} + \beta_2\zeta})(\sqrt{\beta_2{\boldsymbol v_{t-1,i}} + \zeta}+\sqrt{\beta_2{\boldsymbol v_{t-1,i}} + \beta_2\zeta})}\nonumber\\
     \le &\sum_{i=1}^d\frac{\eta \beta_1}{1-\frac{\beta_1}{\sqrt{\beta_2}}} |\partial_i f(\boldsymbol x_t)|
     \frac{(1-\beta_2)\zeta }{(\sqrt{\beta_2{\boldsymbol v_{t-1,i}} + \zeta})\sqrt{\zeta}}\frac{1-\beta_1}{\sqrt{1-\beta_2}\sqrt{1-\frac{\beta_1^2}{\beta_2}}}\nonumber\\
     \le &\sum_{i=1}^d\frac{\eta \beta_1 (1-\beta_1)\sqrt{\zeta}}{(1-\frac{\beta_1}{\sqrt{\beta_2}})\sqrt{1-\frac{\beta_1^2}{\beta_2}}}\left(\frac{(\partial_i f(\boldsymbol x_t))^2}{2\alpha_3\sqrt{\beta_2\boldsymbol v_{t-1,i}+\zeta}}+\frac{\alpha_3(1-\beta_2)}{2\sqrt{\beta_2\boldsymbol v_{t-1,i}+\zeta}}\right),
\end{flalign}
where the second inequality is due to Lemma \ref{lemma:4} and {$2\beta_2> 1$}.

For error 3, it can be bounded as follows
\begin{flalign}\label{eq:6eq11}
    &{\langle \nabla f(\boldsymbol u_t)-\nabla f(\boldsymbol x_t), \boldsymbol u_{t}-\boldsymbol u_{t+1} \rangle}\nonumber\\
    \le &\sum_{i=1}^d |\partial_i f(\boldsymbol u_t)-\partial_i f(\boldsymbol x_t)|| \boldsymbol u_{t,i}-\boldsymbol u_{t+1,i}|\nonumber\\
    \le &\sum_{i=1}^d(L_0+L_1|\partial_i f(\boldsymbol x_t)|)\|\boldsymbol u_t-\boldsymbol x_t\|| \boldsymbol u_{t,i}-\boldsymbol u_{t+1,i}|\nonumber\\
    \le &\sum_{i=1}^d(L_0+L_1|\partial_i f(\boldsymbol x_t)|)\frac{\frac{\beta_1}{\sqrt{\beta_2}}}{1-\frac{\beta_1}{\sqrt{\beta_2}}}\|\boldsymbol x_t-\boldsymbol x_{t-1}\|\left| \frac{\boldsymbol x_{t+1,i}-\boldsymbol x_{t,i}}{1-\frac{\beta_1}{\sqrt{\beta_2}}}-\frac{\beta_1}{\sqrt{\beta_2}}\frac{\boldsymbol x_{t,i}-\boldsymbol x_{t-1,i}}{1-\frac{\beta_1}{\sqrt{\beta_2}}}\right|,
\end{flalign}
where the second inequality is due to Assumption \ref{assump:generalsmooth}.
According to the update process of $\boldsymbol x_t$ and Lemma \ref{lemma:5}, it is easy to get that
\begin{flalign}\label{eq:6eq12}
    &\sum_{i=1}^dL_0\frac{\frac{\beta_1}{\sqrt{\beta_2}}}{1-\frac{\beta_1}{\sqrt{\beta_2}}}\|\boldsymbol x_t-\boldsymbol x_{t-1}\|\left| \frac{\boldsymbol x_{t+1,i}-\boldsymbol x_{t,i}}{1-\frac{\beta_1}{\sqrt{\beta_2}}}-\frac{\beta_1}{\sqrt{\beta_2}}\frac{\boldsymbol x_{t,i}-\boldsymbol x_{t-1,i}}{1-\frac{\beta_1}{\sqrt{\beta_2}}}\right|\nonumber\\
    &\le \sum_{i=1}^dL_0\frac{\frac{\beta_1}{\sqrt{\beta_2}}}{1-\frac{\beta_1}{\sqrt{\beta_2}}}\left(\frac{1}{1-\frac{\beta_1}{\sqrt{\beta_2}}}\left({\|\boldsymbol x_{t}-\boldsymbol x_{t-1}\|}|\boldsymbol x_{t+1,i}-\boldsymbol x_{t,i}| \right)+\frac{\frac{\beta_1}{\sqrt{\beta_2}}}{1-\frac{\beta_1}{\sqrt{\beta_2}}}\left({\|\boldsymbol x_{t}-\boldsymbol x_{t-1}\|}|\boldsymbol x_{t,i}-\boldsymbol x_{t-1,i}| \right)\right)\nonumber\\
        &\le \sum_{i=1}^dL_0\frac{\frac{\beta_1}{\sqrt{\beta_2}}}{1-\frac{\beta_1}{\sqrt{\beta_2}}}\Bigg(\frac{1}{1-\frac{\beta_1}{\sqrt{\beta_2}}}\left({\frac{\|\boldsymbol x_{t}-\boldsymbol x_{t-1}\|^2}{2\sqrt{d}}
        }+\frac{\sqrt{d}|\boldsymbol x_{t+1,i}-\boldsymbol x_{t,i}|^2}{2} \right)\nonumber\\
        &+\frac{\frac{\beta_1}{\sqrt{\beta_2}}}{1-\frac{\beta_1}{\sqrt{\beta_2}}}\left({\frac{\|\boldsymbol x_{t}-\boldsymbol x_{t-1}\|^2}{2\sqrt{d}}
        }+\frac{\sqrt{d}|\boldsymbol x_{t,i}-\boldsymbol x_{t-1,i}|^2}{2} \right)\Bigg)\nonumber\\
        &\le L_0\eta^2\sum_{i=1}^d\left(\frac{\frac{\beta_1}{2\sqrt{\beta_2}}\sqrt{d}+\frac{\beta_1^2}{\beta_2}\sqrt{d}}{\left(1-\frac{\beta_1}{\sqrt{\beta_2}}\right)^2}\frac{|\boldsymbol m_{t-1,i}|^2}{\boldsymbol v_{t-1,i}+\zeta}+\frac{\frac{\beta_1}{2\sqrt{\beta_2}}\sqrt{d}}{\left(1-\frac{\beta_1}{\sqrt{\beta_2}}\right)^2}\frac{\boldsymbol m_{t,i}^2}{\boldsymbol v_{t,i}+\zeta}\right).
\end{flalign}
However, it is hard to bound the remaining $|\partial_i f(\boldsymbol x_t)|\|\boldsymbol x_t-\boldsymbol x_{t-1}\||\boldsymbol u_{t,i}-\boldsymbol u_{t+1,i}|$ term by directly applying Lemma \ref{lemma:4}, which will induce a $\mathcal O(|\partial_i f(\boldsymbol x_t)| \frac{\eta^2}{1-\beta_2})$ term and with our affine noise variance, we can only bound the $\frac{(\partial_i f(\boldsymbol x_t))^2}{\sqrt{\beta_2\boldsymbol v_{t-1,i}+\zeta}}$ term. It is worth noting that this challenging additional term is due to the $(L_0,L_1)$-smoothness, and methods in \citet{wang2023closing,li2023convergence} do not generalize to our case. To solve this challenge, we first bound the additional terms as follows:
\begin{flalign}\label{eq:6eq13}
    &\sum_{i=1}^d|\partial_i f(\boldsymbol x_t)|\|\boldsymbol x_t-\boldsymbol x_{t-1}\||\boldsymbol x_{t,i}-\boldsymbol x_{t-1,i}|\nonumber\\
    \le & \sum_{i=1}^ d|\partial_i f(\boldsymbol x_t)|\|\boldsymbol  x_t-\boldsymbol x_{t-1}\|\frac{\eta |\boldsymbol m_{t-1,i}|}{\sqrt{{\boldsymbol v_{t-1,i}+\zeta}}}\nonumber\\
    \le &\sum_{i=1}^d  \frac{\eta (\partial_i f(\boldsymbol x_t))^2 }{2\alpha_4 \sqrt{\boldsymbol v_{t-1,i}+\zeta}}+
    \frac{\alpha_4 \eta \boldsymbol m_{t-1,i}^2}{2\sqrt{\boldsymbol v_{t-1,i}+\zeta}}|\boldsymbol x_t-\boldsymbol x_{t-1}|^2\nonumber\\
    \le &\sum_{i=1}^d \frac{\eta (\partial_i f(\boldsymbol x_t))^2 }{2\alpha_4 \sqrt{\beta_2 \boldsymbol v_{t-1,i}+\zeta}}+\sum_{i=1}^d\frac{\alpha_4 \eta^3 \boldsymbol m_{t-1,i}^2}{2\sqrt{\boldsymbol v_{t-1,i}+\zeta}}\frac{d(1-\beta_1)^2}{{(1-\beta_2)}{(1-\frac{\beta_1^2}{\beta_2})}},
\end{flalign}
for any $\alpha_4>0$, where the last inequality is due to Lemma \ref{lemma:4}. The motivation is to bound using $\frac{(\partial_i f(\boldsymbol x_t))^2}{\sqrt{\beta_2\boldsymbol v_{t-1,i}+\zeta}}$ and $\frac{\eta^3 \boldsymbol m_{t-1,i}^2}{2\sqrt{\boldsymbol v_{t-1,i}+\zeta}}$, where the latter can be bounded by Lemma \ref{lemma:6}.
Similarly, we have that
\begin{flalign} \label{eq:6eq14}
    &\sum_{i=1}^d|\partial_i f(\boldsymbol x_t)|\|\boldsymbol x_t-\boldsymbol x_{t-1}\||\boldsymbol x_{t+1,i}-\boldsymbol x_{t,i}|\nonumber\\
    \le &\sum_{i=1}^d \left[ \frac{\eta (\partial_i f(\boldsymbol x_t))^2 }{2\alpha_4 \sqrt{\boldsymbol v_{t,i}+\zeta}}+
    \frac{\alpha_4 \eta \boldsymbol m_{t,i}^2}{2\sqrt{\boldsymbol v_{t,i}+\zeta}}|\boldsymbol x_t-\boldsymbol x_{t-1}|^2\right]\nonumber\\
    \le &\sum_{i=1}^d \frac{\eta (\partial_i f(\boldsymbol x_t))^2 }{2\alpha_4 \sqrt{\beta_2 \boldsymbol v_{t-1,i}+\zeta}}+\sum_{i=1}^d\frac{\alpha_4 \eta^3 \boldsymbol m_{t,i}^2}{2\sqrt{\boldsymbol v_{t,i}+\zeta}}\frac{d(1-\beta_1)^2}{{(1-\beta_2)}{(1-\frac{\beta_1^2}{\beta_2})}}.
\end{flalign}
Combine \eqref{eq:6eq8}, \eqref{eq:6eq9}, \eqref{proof:6eq5}, \eqref{eq:6eq10}, \eqref{eq:6eq11}, \eqref{eq:6eq12}, \eqref{eq:6eq13} and \eqref{eq:6eq14}, and we then have that
\begin{flalign}\label{eq:6eq15}
     &{\mathbb E[\langle \nabla f(\boldsymbol u_t), \boldsymbol u_{t+1}-\boldsymbol u_{t} \rangle|\mathcal F_t]}\nonumber\\
     \le & -\sum_{i=1}^d\frac{\eta(1-\beta_1)}{1-\frac{\beta_1}{\sqrt{\beta_2}}}\frac{(\partial_i f(\boldsymbol x_t))^2}{\sqrt{{\beta_2\boldsymbol v_{t-1,i}} + \zeta}}\nonumber\\
     &+\sum_{i=1}^d \frac{\eta}{1-\frac{\beta_1}{\sqrt{\beta_2}}} 
  \frac{1-\beta_1}{\sqrt{1-\frac{\beta_1^2}{\beta_2}}} 
 \Bigg(\frac{ (\partial_i f(\boldsymbol x_t))^2}{2\alpha_0\sqrt{{\beta_2 \boldsymbol v_{t-1,i}} + \zeta}}+\frac{\alpha_0 D_0}{2}\mathbb E\left[\frac{1}{\sqrt{{\beta_2 \boldsymbol v_{t-1,i}} + \zeta}}-\frac{1}{\sqrt{{ \boldsymbol v_{t,i}} + \zeta}}\Bigg|\mathcal F_t\right]\nonumber\\
 &+\frac{\alpha_0 D_1}{2}\mathbb E\left[\frac{(\partial_i f(\boldsymbol x_{t-1}))^2}{\sqrt{{\beta_2 \boldsymbol v_{t-1,i}} + \zeta}}-\frac{(\partial_i f(\boldsymbol x_{t}))^2}{\sqrt{{ \boldsymbol v_{t,i}} + \zeta}}\Bigg|\mathcal F_t\right]\nonumber\\
 &+\frac{\alpha_0 D_1}{2\sqrt{{\beta \boldsymbol v_{t-1,i}} + \zeta}}\Bigg(\frac{(\partial_i f(\boldsymbol x_t))^2}{\alpha_1 D_1}+{\alpha_1 D_1L_0^2\eta^2}{}\frac{d(1-\beta_1)^2}{{(1-\beta_2)}{(1-\frac{\beta_1^2}{\beta_2}})}\nonumber\\
 &+2\eta L_1(\partial_i f(\boldsymbol x_t))^2\frac{\sqrt{d}(1-\beta_1)}{\sqrt{1-\beta_2}\sqrt{1-\frac{\beta_1^2}{\beta_2}}}\Bigg)\Bigg)\nonumber\\
     &+\sum_{i=1}^d\frac{\eta \beta_1 (1-\beta_1)\sqrt{\zeta}}{(1-\frac{\beta_1}{\sqrt{\beta_2}})\sqrt{1-\frac{\beta_1^2}{\beta_2}}}\left(\frac{(\partial_i f(\boldsymbol x_t))^2}{2\alpha_3\sqrt{\beta_2\boldsymbol v_{t-1,i}+\zeta}}+\frac{\alpha_3(1-\beta_2)}{2\sqrt{\beta_2\boldsymbol v_{t-1,i}+\zeta}}\right)\nonumber\\
     &+L_0\eta^2\sum_{i=1}^d\frac{\frac{\beta_1}{2\sqrt{\beta_2}}\sqrt{d}+\frac{\beta_1^2}{\beta_2}\sqrt{d}}{\left(1-\frac{\beta_1}{\sqrt{\beta_2}}\right)^2}\frac{\boldsymbol m_{t-1,i}^2}{\boldsymbol v_{t-1,i}+\zeta}+\frac{\frac{\beta_1}{2\sqrt{\beta_2}}\sqrt{d}}{\left(1-\frac{\beta_1}{\sqrt{\beta_2}}\right)^2}\frac{\boldsymbol m_{t,i}^2}{\boldsymbol v_{t,i}+\zeta}\nonumber\\
     &+L_1\frac{\frac{\beta_1^2}{{\beta_2}}}{\left(1-\frac{\beta_1}{\sqrt{\beta_2}}\right)^2}\left(\sum_{i=1}^d \frac{\eta (\partial_i f(\boldsymbol x_t))^2 }{2\alpha_4 \sqrt{\beta_2 \boldsymbol v_{t-1,i}+\zeta}}+\sum_{i=1}^d\frac{\alpha_4 \eta^3 \boldsymbol m_{t-1,i}^2}{2\sqrt{\boldsymbol v_{t-1,i}+\zeta}}\frac{d(1-\beta_1)^2}{{(1-\beta_2)}{(1-\frac{\beta_1^2}{\beta_2})}}\right)\nonumber\\
     &+{L_1}\frac{\frac{\beta_1}{\sqrt{\beta_2}}}{\left(1-\frac{\beta_1}{\sqrt{\beta_2}}\right)^2}\mathbb E\left[\sum_{i=1}^d \frac{\eta (\partial_i f(\boldsymbol x_t))^2 }{2\alpha_4 \sqrt{\beta_2 \boldsymbol v_{t-1,i}+\zeta}}+\sum_{i=1}^d\frac{\alpha_4 \eta^3 \boldsymbol m_{t,i}^2}{2\sqrt{\boldsymbol v_{t,i}+\zeta}}\frac{d(1-\beta_1)^2}{{(1-\beta_2)}{(1-\frac{\beta_1^2}{\beta_2})}}\right].
\end{flalign}
Since $C_1=1-\frac{\beta_1}{\sqrt{\beta_2}},$ and $C_2=\sqrt{1-\frac{\beta_1^2}{\beta_2}}$, \eqref{eq:6eq15} can be further bounded as follows
\begin{flalign}
 &{\mathbb E[\langle \nabla f(\boldsymbol u_t), \boldsymbol u_{t+1}-\boldsymbol u_{t} \rangle|\mathcal F_t]}\nonumber\\  
 \le& \Bigg[-\frac{\eta (1-\beta_1)}{C_1}+\frac{\eta(1-\beta_1)}{C_1C_2}\left(\frac{1}{2\alpha_0}+\frac{\alpha_0}{2\alpha_1}+\frac{\eta \alpha_0\sqrt{d}D_1L_1(1-\beta_1)}{\sqrt{1-\beta_2}C_2}\right)+\frac{\eta \beta_1(1-\beta_1)\sqrt{\zeta}}{2\alpha_3C_1C_2}\nonumber\\
 &+\frac{\eta L_1(1-C_1)^2}{2\alpha_4C_1^2}+\frac{\eta L_1(1-C_1)}{2\alpha_4C_1^2}\Bigg] \times \sum_{i=1}^d\frac{(\partial_i f(\boldsymbol x_t))^2}{\sqrt{{\beta_2\boldsymbol v_{t-1,i}} + \zeta}}\nonumber\\
 &+\sum_{i=1}^d\frac{\eta (1-\beta_1)}{C_1C_2}\Bigg[\frac{\alpha_0D_0}{2}\Bigg(\frac{1}{\sqrt{\beta_2 \boldsymbol v_{t-1,i}+\zeta}}-\mathbb E\Bigg[{\frac{1}{\sqrt{\boldsymbol v_{t,i}+\zeta}}}\Bigg|\mathcal F_t\Bigg]\Bigg)\nonumber\\
 &+\frac{\alpha_0 D_1}{2}\mathbb E\left[\frac{(\partial_i \boldsymbol f(x_{t-1}))^2}{\sqrt{{\beta_2 \boldsymbol v_{t-1,i}} + \zeta}}-\frac{(\partial_i f(\boldsymbol x_{t}))^2}{\sqrt{{ \boldsymbol v_{t,i}} + \zeta}}\Bigg|\mathcal F_t\right]\Bigg]+\frac{\alpha_0\alpha_1D_1^2L_0^2\eta^3d^2(1-\beta_1)^3}{2(1-\beta_2)C_1C_2^3\sqrt{\zeta}}+\frac{\alpha_3\eta d\beta_1(1-\beta_1)(1-\beta_2)}{2C_1C_2}\nonumber\\
 &+\sum_{i=1}^d\frac{\eta^2((1-C_1)^2+0.5(1-C_1))\sqrt{d}L_0}{C_1^2}  \frac{\boldsymbol m_{t-1,i}^2}{\boldsymbol v_{t-1,i}+\zeta}+\sum_{i=1}^d\frac{\eta^20.5(1-C_1)\sqrt{d}L_0}{C_1^2}\mathbb E\Bigg[ \frac{\boldsymbol m_{t,i}^2}{\boldsymbol v_{t,i}+\zeta}\Bigg|\mathcal F_t\Bigg]\nonumber\\
 &+\sum_{i=1}^d\frac{\alpha_4\eta^3(1-\beta_1)^2(1-C_1)^2dL_1}{2(1-\beta_2)C_1^2C_2^2}\frac{\boldsymbol m_{t-1,i}^2}{\sqrt{\boldsymbol v_{t-1,i}+\zeta}}+\sum_{i=1}^d\frac{\alpha_4\eta^3(1-\beta_1)^2(1-C_1)dL_1}{2(1-\beta_2)C_1^2C_2^2}\mathbb E\left[\frac{\boldsymbol m_{t,i}^2}{\sqrt{\boldsymbol v_{t,i}+\zeta}}\Bigg|\mathcal F_t\right].
\end{flalign}
This completes the proof.
\end{proof}

\section{Proof of Lemma \ref{co:2}}\label{proof:co2}
For $(L_0,L_1)$-smooth objective functions, we have the following descent inequality (Lemma 1 in \citet{crawshaw2022robustness}):
\begin{flalign}\label{eq:7eq1}
    \underbrace{\mathbb E[\langle \nabla f(\boldsymbol u_t), \boldsymbol u_{t}-\boldsymbol u_{t+1} \rangle|\mathcal F_t]}_{\text{first-order}}&\le f(\boldsymbol u_t)-\mathbb E[f(\boldsymbol u_{t+1})|\mathcal F_t]
  +\underbrace{\sum_{i=1}^d\frac{L_0}{2\sqrt{d}}\mathbb E[\|\boldsymbol u_{t+1}-\boldsymbol u_t\||\boldsymbol u_{t+1,i}-\boldsymbol u_{t,i}||\mathcal F_t]}_{\text{second-order}}\nonumber\\
    &+\underbrace{\sum_{i=1}^d\frac{L_1\|\partial_i f(\boldsymbol u_t)\|}{2}\mathbb E[\|\boldsymbol u_{t+1}-\boldsymbol u_t\||\boldsymbol u_{t+1,i}-\boldsymbol u_{t,i}||\mathcal F_t]}_{\text{additional term}}.
\end{flalign}

The first-order term is bounded by Lemma \ref{formal:lemma3}, we then only need to bound the remaining two terms. For the second-order term, based on the definition of $\boldsymbol u_t$ and update process of $\boldsymbol x_t$, we have that
\begin{flalign}\label{eq:7eq2}
   &\sum_{i=1}^d\frac{L_0}{2\sqrt{d}}\mathbb E[\|\boldsymbol u_{t+1}-\boldsymbol u_t\||\boldsymbol u_{t+1,i}-\boldsymbol u_{t,i}||\mathcal F_t]\nonumber\\
   \le&\sum_{i=1}^d\frac{L_0}{2\sqrt{d}}\mathbb E\left[ \frac{1}{2\sqrt{d}}\left\|\frac{\boldsymbol x_{t+1}-\boldsymbol x_t}{1-\frac{\beta_1}{\sqrt{\beta_2}}}-\frac{\beta_1}{\sqrt{\beta_2}}\frac{\boldsymbol x_{t}-\boldsymbol x_{t-1}}{1-\frac{\beta_1}{\sqrt{\beta_2}}}\right\|^2+\frac{\sqrt{d}}{2}\left|\frac{\boldsymbol x_{t+1,i}-\boldsymbol x_{t,i}}{1-\frac{\beta_1}{\sqrt{\beta_2}}}-\frac{\beta_1}{\sqrt{\beta_2}}\frac{\boldsymbol x_{t,i}-\boldsymbol x_{t-1,i}}{1-\frac{\beta_1}{\sqrt{\beta_2}}}\right|^2|\mathcal F_t\right]\nonumber\\
    \le &\sum_{i=1}^d\frac{L_0}{2\sqrt{d}}\mathbb E\left[ \frac{2\sqrt{d}}{(1-\frac{\beta_1}{\sqrt{\beta_2}})^2}|\boldsymbol x_{t+1,i}-\boldsymbol x_{t,i}|^2+\frac{2\sqrt{d}\frac{\beta_1^2}{{\beta_2}}}{(1-\frac{\beta_1}{\sqrt{\beta_2}})^2}|\boldsymbol x_{t,i}-\boldsymbol x_{t-1,i}|^2|\mathcal F_t\right]\nonumber\\
    =& \sum_{i=1}^d\frac{L_0}{2\sqrt{d}}\left(\frac{2\sqrt{d}}{C_1^2}\mathbb E\left[\frac{\eta^2\boldsymbol m_{t,i}^2}{\boldsymbol v_{t,i}+\zeta}|\mathcal F_t\right]+\frac{2\sqrt{d}(1-C_1)^2}{C_1^2}\frac{\eta^2 \boldsymbol m_{t-1,i}^2}{\boldsymbol v_{t-1,i}+\zeta}\right).
\end{flalign}

Now we focus on the additional term. According to the definition of $\boldsymbol u_t$ and update process of $\boldsymbol x_t$, for $\alpha_4>0$ we have that
\begin{flalign} \label{eq:7eq3}
    &\sum_{i=1}^d\frac{L_1\|\partial_i f(\boldsymbol u_t)\|}{2}\mathbb E[\|\boldsymbol u_{t+1}-\boldsymbol u_t\||\boldsymbol u_{t+1,i}-\boldsymbol u_{t,i}||\mathcal F_t]\nonumber\\
    \le &\sum_{i=1}^d\frac{L_1}{2}(|\partial_i f(\boldsymbol x_t)|+(L_0+L_1|\partial_i f(\boldsymbol x_t)|)\|\boldsymbol u_t-\boldsymbol x_t\|)\mathbb E[\|\boldsymbol u_{t+1}-\boldsymbol u_t\||\boldsymbol u_{t+1,i}-\boldsymbol u_{t,i}||\mathcal F_t]\nonumber\\
    \le &\sum_{i=1}^d\frac{L_1}{2}\left(|\partial_i f(\boldsymbol x_t)|+L_0\frac{\frac{\beta_1}{\sqrt{\beta_2}}}{1-\frac{\beta_1}{\sqrt{\beta_2}}}\|\boldsymbol x_t-\boldsymbol x_{t-1}\|+L_1|\partial_i f(\boldsymbol x_t)|\frac{\frac{\beta_1}{\sqrt{\beta_2}}}{1-\frac{\beta_1}{\sqrt{\beta_2}}}\|\boldsymbol x_t-\boldsymbol x_{t-1}\|\right)\nonumber\\
    &\times \mathbb E[\|\boldsymbol u_{t+1}-\boldsymbol u_t\||\boldsymbol u_{t+1,i}-\boldsymbol u_{t,i}||\mathcal F_t]\nonumber\\
    \le &\sum_{i=1}^d\frac{L_1}{2}\left(1+L_1\frac{1-C_1}{C_1}\eta\sqrt{d} \frac{1-\beta_1}{\sqrt{1-\beta_2}C_2}\right)\nonumber\\
    &\times \left(\sqrt{d}\frac{2-C_1}{C_1}\frac{\eta (1-\beta_1)}{\sqrt{1-\beta_2}C_2}|\partial_i f(\boldsymbol x_t)|\left(\frac{1}{C_1}\mathbb E\left[\frac{\eta |\boldsymbol m_{t,i}|}{\sqrt{\boldsymbol v_{t,i}+\zeta}}\Bigg|\mathcal F_t\right]+\frac{1-C_1}{C_1}\frac{\eta |\boldsymbol m_{t-1,i}|}{\sqrt{\boldsymbol v_{t-1,i}+\zeta}}\right)\right)\nonumber\\
        &+\sum_{i=1}^d\frac{L_1L_0}{2}\frac{1-C_1}{C_1}\eta\sqrt{d} \frac{1-\beta_1}{\sqrt{1-\beta_2}C_2}\times \left(\frac{2\sqrt{d}}{C_1^2}\mathbb E\left[\frac{\eta^2\boldsymbol m_{t,i}^2}{\boldsymbol v_{t,i}+\zeta}\Bigg|\mathcal F_t\right]+\frac{2\sqrt{d}(1-C_1)^2}{C_1^2}\frac{\eta^2 \boldsymbol m_{t-1,i}^2}{\boldsymbol v_{t-1,i}+\zeta}\right)\nonumber\\
      \le&\sum_{i=1}^d\frac{L_1}{2}\left(1+L_1\frac{1-C_1}{C_1}\eta\sqrt{d} \frac{1-\beta_1}{\sqrt{1-\beta_2}C_2}\right)\nonumber\\
      &\times \Bigg(\frac{2\sqrt{d}}{C_1^2}\left(\frac{\eta}{2\alpha_4}\frac{(\partial_i f(\boldsymbol x_t))^2}{\sqrt{\beta_2\boldsymbol v_{t-1,i}+\zeta}}+\frac{\alpha_4\eta^3(1-\beta_1)^2}{2(1-\beta_2)C_2^2}\mathbb E\left[\frac{\boldsymbol m_{t,i}^2}{\sqrt{\boldsymbol v_{t,i}+\zeta}}\Bigg|\mathcal F_t\right]\right)\nonumber\\
  &+\frac{2\sqrt{d}(2-C_1)^2}{C_1^2}\left(\frac{\eta }{2\alpha_4}\frac{(\partial_i f(\boldsymbol x_t))^2}{\sqrt{\beta_2\boldsymbol v_{t-1,i}+\zeta}}+\frac{\alpha_4\eta^3(1-\beta_1)^2}{2(1-\beta_2)C_2^2}\frac{\boldsymbol m_{t-1,i}^2}{\sqrt{\boldsymbol v_{t-1,i}+\zeta}}\right)\Bigg)\nonumber\\
  &+\sum_{i=1}^d\frac{L_1L_0}{2}\frac{1-C_1}{C_1}\eta\sqrt{d} \frac{1-\beta_1}{\sqrt{1-\beta_2}C_2}\left(\frac{2\sqrt{d}}{C_1^2}\mathbb E\left[\frac{\eta^2\boldsymbol m_{t,i}^2}{\boldsymbol v_{t,i}+\zeta}\Bigg|\mathcal F_t\right]+\frac{2\sqrt{d}(1-C_1)^2}{C_1^2}\frac{\eta^2 \boldsymbol m_{t-1,i}^2}{\boldsymbol v_{t-1,i}+\zeta}\right),
\end{flalign}
where the first inequality is due to the $(L_0,L_1)$-smoothness assumption, the third inequality is due to Lemma \ref{lemma:4} and \eqref{eq:7eq2}, and the last inequality is due to \eqref{eq:6eq13} and \eqref{eq:6eq14}. Combine Lemma \ref{formal:lemma3}, \eqref{eq:7eq1}, \eqref{eq:7eq2} and \eqref{eq:7eq3} together, and we have that
\begin{flalign}\label{eq:7eq4}
 &\Bigg(\frac{\eta (1-\beta_1)}{C_1}-\frac{\eta(1-\beta_1)}{C_1C_2}\left(\frac{1}{2\alpha_0}+\frac{\alpha_0}{2\alpha_1}+\frac{\eta \alpha_0\sqrt{d}D_1L_1(1-\beta_1)}{\sqrt{1-\beta_2}C_2}\right)-\frac{\eta \beta_1(1-\beta_1)\sqrt{\zeta}}{2\alpha_3C_1C_2}-\frac{\eta L_1(1-C_1)^2}{2\alpha_4C_1^2}\nonumber\\
 -&\frac{\eta L_1(1-C_1)}{2\alpha_4C_1^2}-\frac{\sqrt{d}L_1}{2}\left(1+L_1\frac{1-C_1}{C_1}\eta\sqrt{d} \frac{1-\beta_1}{\sqrt{1-\beta_2}C_2}\right)\left(\frac{\eta}{\alpha_4C_1^2}+\frac{\eta(2-C_1)^2}{\alpha_4C_1^2}\right)
 \Bigg)\times \sum_{i=1}^d\frac{(\partial_i f(\boldsymbol x_t))^2}{\sqrt{{\beta_2\boldsymbol v_{t-1,i}} + \zeta}}\nonumber\\  
 \le&f(\boldsymbol u_t)-\mathbb E[f(\boldsymbol u_{t+1})|\mathcal F_t]+\sum_{i=1}^d\frac{\eta (1-\beta_1)}{C_1C_2} \nonumber\\
 &\times \Bigg(\frac{\alpha_0D_0}{2}\Bigg(\frac{1}{\sqrt{\beta_2 \boldsymbol v_{t-1,i}+\zeta}}-\mathbb E\Bigg[{\frac{1}{\sqrt{\boldsymbol v_{t,i}+\zeta}}}\Bigg|\mathcal F_t\Bigg]\Bigg)+\frac{\alpha_0 D_1}{2}\mathbb E\left[\frac{(\partial_i f(\boldsymbol x_{t-1}))^2}{\sqrt{{\beta_2 \boldsymbol v_{t-1,i}} + \zeta}}-\frac{(\partial_i f(\boldsymbol x_{t}))^2}{\sqrt{{ \boldsymbol v_{t,i}} + \zeta}}\Bigg|\mathcal F_t\right]\Bigg)\nonumber\\
 &+\frac{\alpha_0\alpha_1D_1^2L_0^2\eta^3d^2(1-\beta_1)^3}{2(1-\beta_2)C_1C_2^3\sqrt{\zeta})}+\frac{\alpha_3\eta d\beta_1(1-\beta_1)(1-\beta_2)}{2C_1C_2}\nonumber\\
 &+\sum_{i=1}^d\left(\frac{\eta^2(2(1-C_1)^2+0.5(1-C_1))\sqrt{d}L_0}{C_1^2}+  \frac{dL_1L_0}{2}\frac{1-C_1}{C_1}\eta \frac{1-\beta_1}{\sqrt{1-\beta_2}C_2}\frac{2(1-C_1)^2}{C_1^2}\eta^2\right)  \frac{\boldsymbol m_{t-1,i}^2}{\boldsymbol v_{t-1,i}+\zeta}\nonumber\\
  &+\sum_{i=1}^d\left(\frac{\eta^2(0.5(1-C_1)+1)\sqrt{d}L_0}{C_1^2}+  \frac{dL_1L_0}{2}\frac{1-C_1}{C_1}\eta \frac{1-\beta_1}{\sqrt{1-\beta_2}C_2}\frac{2}{C_1^2}\eta^2\right)  \mathbb E\left[\frac{\boldsymbol m_{t,i}^2}{\boldsymbol v_{t,i}+\zeta}\Bigg|\mathcal F_t\right]\nonumber\\
 & +\sum_{i=1}^d\Bigg(\frac{\alpha_4\eta^3(1-\beta_1)^2(1-C_1)^2dL_1}{2(1-\beta_2)C_1^2C_2^2}\nonumber\\
 &+\frac{L_1}{2}\left(1+L_1\frac{1-C_1}{C_1}\eta\sqrt{d} \frac{1-\beta_1}{\sqrt{1-\beta_2}C_2}\right)\frac{2\sqrt{d}(2-C_1)^2}{C_1^2}\frac{\alpha_4\eta^3(1-\beta_1)^2}{2(1-\beta_2)C_2^2}\Bigg)\times \frac{\boldsymbol m_{t-1,i}^2}{\sqrt{\boldsymbol v_{t-1,i}+\zeta}}\nonumber\\
 & +\sum_{i=1}^d\left(\frac{\alpha_4\eta^3(1-\beta_1)^2(1-C_1)dL_1}{2(1-\beta_2)C_1^2C_2^2}+\frac{L_1}{2}\left(1+L_1\frac{1-C_1}{C_1}\eta\sqrt{d} \frac{1-\beta_1}{\sqrt{1-\beta_2}C_2}\right)\frac{2\sqrt{d}}{C_1^2}\frac{\alpha_4\eta^3(1-\beta_1)^2}{2(1-\beta_2)C_2^2}\right)\nonumber\\
 &\times \mathbb E\left[\frac{\boldsymbol m_{t,i}^2}{\sqrt{\boldsymbol v_{t,i}+\zeta}}\Bigg|\mathcal F_t\right].
\end{flalign}
It is worth noting that (\ref{proof:coeq1}) and (\ref{proof:coeq2}) still hold for Adam since the update of $\boldsymbol v_t$ does not change. Specifically, for any $i\in [d]$ we have that
\begin{flalign}\label{eq:7eq5}
    &\sum_{t=1}^T \mathbb E\left[\frac{1}{\sqrt{{\beta_2 \boldsymbol v_{t-1,i}} + \zeta}}-\frac{1}{\sqrt{{ \boldsymbol v_{t,i}} + \zeta}}\right]
    \le  \frac{1}{\sqrt{ \zeta}}+T \frac{1-\sqrt{\beta_2}}{\sqrt{\zeta}}.
\end{flalign}
Furthermore,  since $\boldsymbol x_0=\boldsymbol x_1$, we have that
\begin{flalign}\label{eq:7eq6}
    &\mathbb E\left[\frac{(\partial_i f(\boldsymbol x_{0}))^2}{\sqrt{{\beta_2 \boldsymbol v_{0,i}} + \zeta}}-\frac{(\partial_i f(\boldsymbol x_{1}))^2}{\sqrt{{ \boldsymbol v_{1,i}} + \zeta}}\right]+\sum_{t=2}^T \mathbb E\left[\frac{(\partial_i f(\boldsymbol x_{t-1}))^2}{\sqrt{{\beta_2 \boldsymbol v_{t-1,i}} + \zeta}}-\frac{(\partial_i f(\boldsymbol x_{t}))^2}{\sqrt{{ \boldsymbol v_{t,i}} + \zeta}}\right]\nonumber\\
    \le & \frac{(\partial_i f(\boldsymbol x_{1}))^2}{\sqrt{ \zeta}}+\sum_{t=1}^{T-1}(1-{\beta_2}) \mathbb E\left[\frac{(\partial_i f(\boldsymbol x_{t}))^2}{\sqrt{\beta_2{ \boldsymbol v_{t-1,i}} +  \zeta}}\right].
\end{flalign}

Taking expectations and telescoping (\ref{eq:7eq4}) for $t=1$ to $T$, and based on (\ref{eq:7eq5}), (\ref{eq:7eq6}), Lemma \ref{lemma:5} and Lemma \ref{lemma:6}, we can show that
\begin{flalign}\label{eq:7eq7}
 &\Bigg(\frac{\eta (1-\beta_1)}{C_1}-\frac{\eta(1-\beta_1)}{C_1C_2}\left(\frac{1}{2\alpha_0}+\frac{\alpha_0}{2\alpha_1}+\frac{\eta \alpha_0\sqrt{d}D_1L_1(1-\beta_1)}{\sqrt{1-\beta_2}C_2}\right)-\frac{\eta \beta_1(1-\beta_1)\sqrt{\zeta}}{2\alpha_3C_1C_2}-\frac{\eta L_1(1-C_1)^2}{2\alpha_4C_1^2}\nonumber\\
 -&\frac{\eta L_1(1-C_1)}{2\alpha_4C_1^2}-\frac{\sqrt{d}L_1}{2}\left(1+L_1\frac{1-C_1}{C_1}\eta\sqrt{d} \frac{1-\beta_1}{\sqrt{1-\beta_2}C_2}\right)\left(\frac{\eta}{\alpha_4C_1^2}+\frac{\eta(2-C_1)^2}{\alpha_4C_1^2}\right)
 \Bigg)\nonumber\\
 &\times \sum_{t=1}^T\sum_{i=1}^d\frac{(\partial_i f(\boldsymbol x_t))^2}{\sqrt{{\beta_2\boldsymbol v_{t-1,i}} + \zeta}}\nonumber\\  
 \le&f(\boldsymbol u_1)-\mathbb E[f(\boldsymbol u_{T+1})|\mathcal F_t]+\sum_{i=1}^d\frac{\eta (1-\beta_1)}{C_1C_2}\nonumber\\
 &\times \Bigg(\frac{\alpha_0D_0}{2}\Bigg(\frac{1}{\sqrt{ \zeta}}+T \frac{1-\sqrt{\beta_2}}{\sqrt{\zeta}}\Bigg)+\frac{\alpha_0 D_1}{2}\left(\frac{(\partial_i f(\boldsymbol x_{1}))^2}{\sqrt{ \zeta}}+\sum_{t=1}^{T-1}(1-{\beta_2}) \mathbb E\left[\frac{(\partial_i f(\boldsymbol x_{1}))^2}{\sqrt{\beta_2{ \boldsymbol v_{t-1,i}} +  \zeta}}\right]\right)\Bigg)\nonumber\\
 &+T\frac{\alpha_0\alpha_1D_1^2L_0^2\eta^3d^2(1-\beta_1)^3}{2(1-\beta_2)C_1C_2^3\sqrt{\zeta}}+T\frac{\alpha_3\eta d\beta_1(1-\beta_1)(1-\beta_2)}{2C_1C_2}\nonumber\\
 &+\Bigg(\frac{\eta^22(1-C_1)^2\sqrt{d}L_0}{C_1^2} + \frac{\eta^2(2-C_1)\sqrt{d}L_0}{C_1^2} +\frac{dL_1L_0}{2}\frac{1-C_1}{C_1}\eta \frac{1-\beta_1}{\sqrt{1-\beta_2}C_2}\frac{2(1-C_1)^2+2}{C_1^2}\eta^2\Bigg)\nonumber\\
 &\times\sum_{i=1}^d\mathbb E\Bigg[ \frac{(1-\beta_1)^2}{(1-\frac{\beta_1}{\sqrt{\beta_2}})^2(1-\beta_2)}\left(\ln\left(\frac{\boldsymbol v_{T,i}}{\boldsymbol v_{0,i}}\right)-T\ln(\beta_2)\right)\Bigg]\nonumber\\
 &+\Bigg(\frac{\alpha_4\eta^3(1-\beta_1)^2dL_1((1-C_1)+(1-C_1)^2)}{2(1-\beta_2)C_1^2C_2^2}\nonumber\\
 &+\frac{\sqrt{d}L_1}{2}\left(1+L_1\frac{1-C_1}{C_1}\eta\sqrt{d} \frac{1-\beta_1}{\sqrt{1-\beta_2}C_2}\right)\frac{2+2(2-C_1)^2}{C_1^2}\frac{\alpha_4\eta^3(1-\beta_1)^2}{2(1-\beta_2)C_2^2}\Bigg)\nonumber\\
 &\times \sum_{i=1}^d\mathbb E\left[\frac{(1-\beta_1)^2}{(1-\frac{\beta_1}{\sqrt[4]{\beta_2}})^2}\left(\frac{2}{1-\beta_2}(\sqrt{\boldsymbol v_{T,i}}-\sqrt{\boldsymbol v_{0,i}})+\sum_{t=1}^T2\sqrt{\boldsymbol v_{t-1,i}}\right)\right].
\end{flalign}
Moreover, for any $a>0$ we have that $\ln(a)\le a-1$. We then have that
\begin{flalign}\label{eq:7eq8}
    \frac{1}{T}\mathbb E[\ln(\boldsymbol v_{T,i})]=\frac{2}{T}\mathbb E[\ln(\sqrt{\boldsymbol v_{T,i}})]&\le \frac{2}{T}\mathbb E[\sqrt{\boldsymbol v_{T,i}}]\nonumber\\
    &\le \frac{2\sqrt{D_0+\boldsymbol v_{0,i}}}{T}+\frac{2}{T}\sum_{i=0}^{T-1} \mathbb E\left[ \sqrt{{}\beta_2^{i}(1-\beta_2)D_1}|\partial_i f(\boldsymbol x_{T-i})|\right]\nonumber\\
    &\le \frac{2\sqrt{D_0+\boldsymbol v_{0,i}}}{T}+\frac{2}{T}\sum_{i=0}^{T-1} \mathbb E\left[ \sqrt{{}(1-\beta_2)D_1}|\partial_i f(\boldsymbol x_{T-i})|\right],
\end{flalign}
where the second inequality is due to \eqref{eq:3recursive}. In addition, for any $\beta_2\ge 0.5$, we have that
\begin{flalign}\label{eq:7eq9}
    -\ln(\beta_2)=\ln\left(\frac{1}{\beta_2}\right)\le \frac{1}{\beta_2}-1\le 2(1-\beta_2).
\end{flalign}
Combining \eqref{eq:3recursive2}, \eqref{eq:7eq7}, \eqref{eq:7eq8} and \eqref{eq:7eq9}, we then can show $\sum_{i=1}^d\sum_{t=1}^T\mathbb E\left[\frac{(\partial_i f(\boldsymbol x_t))^2}{\sqrt{\beta_2 \boldsymbol v_{t-1,i}+\zeta}}\right]$ is upper bounded by a function of $\sum_{t=1}^T\mathbb E\left[{\|\nabla f(\boldsymbol x_t)\|}\right]$. Specifically, we can get
\begin{flalign}\label{eq:7eq10}
 &\sum_{t=1}^T\sum_{i=1}^d\mathbb E\left[\frac{(\partial_i f(\boldsymbol x_t))^2}{\sqrt{{\beta_2\boldsymbol v_{t-1,i}} + \zeta}}\right]\times \Bigg(\frac{\eta (1-\beta_1)}{C_1}-\frac{\eta(1-\beta_1)}{C_1C_2}\left(\frac{1}{2\alpha_0}+\frac{\alpha_0}{2\alpha_1}+\frac{\eta \alpha_0\sqrt{d}D_1L_1(1-\beta_1)}{\sqrt{1-\beta_2}C_2}\right)\nonumber\\
 &-\frac{\eta \beta_1(1-\beta_1)\sqrt{\zeta}}{2\alpha_3C_1C_2}-\frac{\eta L_1(1-C_1)^2}{2\alpha_4C_1^2}-\frac{\eta L_1(1-C_1)}{2\alpha_4C_1^2}\nonumber\\
 &-\frac{\sqrt{d}L_1}{2}\left(1+L_1\frac{1-C_1}{C_1}\eta\sqrt{d} \frac{1-\beta_1}{\sqrt{1-\beta_2}C_2}\right)\left(\frac{\eta}{\alpha_4C_1^2}+\frac{\eta(2-C_1)^2}{\alpha_4C_1^2}\right)-\frac{\eta (1-\beta_1)}{C_1C_2}\frac{\alpha_0 D_1(1-\beta_2)}{2}
 \Bigg)\nonumber\\
 \le&f(\boldsymbol u_1)-\mathbb E[f(\boldsymbol u_{T+1})|\mathcal F_t]+ \frac{\eta \alpha_0dD_0(1-\beta_1)}{2C_1C_2\sqrt{\zeta}}+ \frac{\eta \alpha_0D_1(1-\beta_1)\|\nabla f(\boldsymbol x_1)\|^2}{2C_1C_2\sqrt{\zeta}}\nonumber\\
 &+T\frac{\eta d\alpha_0D_0 (1-\beta_1)(1-\beta_2)}{2C_1C_2\sqrt{\zeta}}
 +T\frac{\alpha_0\alpha_1D_1^2L_0^2\eta^3d^2(1-\beta_1)^3}{2(1-\beta_2)C_1C_2^3\sqrt{\zeta}}+T\frac{\alpha_3\eta d\beta_1(1-\beta_1)(1-\beta_2)}{2C_1C_2}\nonumber\\
  &+\Bigg(\frac{\eta^22(1-C_1)^2\sqrt{d}L_0}{C_1^2} + \frac{\eta^2(2-C_1)\sqrt{d}L_0}{C_1^2} +\frac{dL_1L_0}{2}\frac{1-C_1}{C_1}\eta \frac{1-\beta_1}{\sqrt{1-\beta_2}C_2}\frac{2(1-C_1)^2+2}{C_1^2}\eta^2\Bigg)\nonumber\\
  &\times \frac{(1-\beta_1)^2}{(1-\frac{\beta_1}{\sqrt{\beta_2}})^2(1-\beta_2)}\sum_{i=1}^d\left(2\sqrt{D_0+\boldsymbol v_{0,i}}+\sum_{t=1}^T\mathbb E[2\sqrt{(1-\beta_2)D_1}|\partial_i f(\boldsymbol x_t)|]+2T(1-\beta_2)-\ln(\boldsymbol v_{0,i})\right)\nonumber\\
  &+\Bigg(\frac{\alpha_4\eta^3(1-\beta_1)^2dL_1((1-C_1)+(1-C_1)^2)}{2(1-\beta_2)C_1^2C_2^2}\nonumber\\
  &+\frac{\sqrt{d}L_1}{2}\left(1+L_1\frac{1-C_1}{C_1}\eta\sqrt{d} \frac{1-\beta_1}{\sqrt{1-\beta_2}C_2}\right)\frac{2+2(2-C_1)^2}{C_1^2}\frac{\alpha_4\eta^3(1-\beta_1)^2}{2(1-\beta_2)C_2^2}\Bigg)\nonumber\\
   &\times\sum_{i=1}^d \frac{(1-\beta_1)^2}{(1-\frac{\beta_1}{\sqrt[4]{\beta_2}})^2}\Bigg(\frac{2}{1-\beta_2}\left(\sqrt{D_0+\boldsymbol v_{0,i}}+\sum_{t=1}^T\mathbb E[\sqrt{(1-\beta_2)D_1}|\partial_i f(\boldsymbol x_t)|]\right)\nonumber\\
   &+2T\sqrt{D_0+\boldsymbol v_{0,i}}+\frac{4\sqrt{D_1}}{\sqrt{1-\beta_2}}\sum_{t=1}^T\mathbb E[|\partial_i f(\boldsymbol x_t)|]\Bigg).
\end{flalign}
If we have 
$\alpha_0\ge \frac{21}{2C_2},\alpha_1\ge \frac{21\alpha_0}{2C_2},\alpha_3\ge \frac{7\beta_1\sqrt{\zeta}}{2C_2}, \alpha_4\ge \frac{14L_1\sqrt{d}(2-C_1)^2}{C_1(1-\beta_1)}$, for 
$\eta\le \min\left(\frac{C_2^2\sqrt{1-\beta_2}}{21\alpha_0\sqrt{d}D_1L_1(1-\beta_1)}, \frac{C_1C_2\sqrt{1-\beta_2}}{\sqrt{d}L_1(1-C_1)(1-\beta_1)}\right)$, $1-\beta_2\le (\frac{2C_2}{7\alpha_0D_1})$, then it can be shown that
\begin{flalign}
 &\frac{\eta (1-\beta_1)}{C_1}-\frac{\eta(1-\beta_1)}{C_1C_2}\left(\frac{1}{2\alpha_0}+\frac{\alpha_0}{2\alpha_1}+\frac{\eta \alpha_0\sqrt{d}D_1L_1(1-\beta_1)}{\sqrt{1-\beta_2}C_2}\right)-\frac{\eta \beta_1(1-\beta_1)\sqrt{\zeta}}{2\alpha_3C_1C_2}\nonumber\\
 -&\frac{\eta L_1(1-C_1)^2}{2\alpha_4C_1^2}-\frac{\eta L_1(1-C_1)}{2\alpha_4C_1^2}-\frac{\eta (1-\beta_1)}{C_1C_2}\frac{\alpha_0 D_1(1-\beta_2)}{2}\nonumber\\
 -&\frac{\sqrt{d}L_1}{2}\left(1+L_1\frac{1-C_1}{C_1}\eta\sqrt{d} \frac{1-\beta_1}{\sqrt{1-\beta_2}C_2}\right)\left(\frac{\eta}{\alpha_4C_1^2}+\frac{\eta(2-C_1)^2}{\alpha_4C_1^2}\right)
 \ge \frac{\eta (1-\beta_1)}{7C_1}.
\end{flalign}
Let $\Delta'=f(\boldsymbol u_1)-\inf_{\boldsymbol{x}}f(\boldsymbol{x})+ \frac{\eta \alpha_0dD_0(1-\beta_1)}{2C_1C_2\sqrt{\zeta}}+ \frac{\eta \alpha_0D_1(1-\beta_1)\|\nabla f(\boldsymbol x_1)\|^2}{2C_1C_2\sqrt{\zeta}}$,
$C_3=\Big(\frac{2(1-C_1)^2\sqrt{d}L_0}{C_1^2} + \frac{(2-C_1)\sqrt{d}L_0}{C_1^2} +{\sqrt{d}L_0}{}\frac{(1-C_1)^2+1}{C_1^2}\Big)$, 
$C_4=\left(\frac{\alpha_4(1-\beta_1)^2dL_1((1-C_1)+(1-C_1)^2)}{2C_1^2C_2^2}+{\sqrt{d}L_1}\frac{2+2(2-C_1)^2}{C_1^2}\frac{\alpha_4(1-\beta_1)^2}{2C_2^2}\right)$,
we then have that
\begin{flalign}\label{eq:ll90}
    &\frac{\eta(1-\beta_1)}{7C_1}\sum_{i=1}^d\sum_{t=1}^T\mathbb E\left[\frac{(\partial_i f(\boldsymbol x_t))^2}{\sqrt{\beta_2 \boldsymbol v_{t-1,i}+\zeta}}\right]\nonumber\\
    \le& \Delta'+T\frac{\eta d\alpha_0D_0 (1-\beta_1)(1-\beta_2)}{2C_1C_2\sqrt{\zeta}}
 +T\frac{\alpha_0\alpha_1D_1^2L_0^2\eta^3d^2(1-\beta_1)^3}{2(1-\beta_2)C_1C_2^3\sqrt{\zeta}}+T\frac{\alpha_3\eta d\beta_1(1-\beta_1)(1-\beta_2)}{2C_1C_2}\nonumber\\
    &+\frac{\eta^2C_3(1-\beta_1)^2}{C_1^2(1-\beta_2)}\sum_{i=1}^d \left(2\sqrt{D_0+\boldsymbol v_{0,i}}+\sum_{t=1}^T\mathbb E[2\sqrt{(1-\beta_2)D_1}|\partial_i f(\boldsymbol x_t)|]+2T(1-\beta_2)-\ln(\boldsymbol v_{0,i})\right)\nonumber\\
    &+\frac{\eta^3C_4}{(1-\beta_2)}\frac{(1-\beta_1)^2}{(1-\frac{\beta_1}{\sqrt[4]{\beta_2}})^2}\times\sum_{i=1}^d \Bigg(\frac{2}{1-\beta_2}\left(\sqrt{D_0+\boldsymbol v_{0,i}}+\sum_{t=1}^T\mathbb E[\sqrt{(1-\beta_2)D_1}|\partial_i f(\boldsymbol x_t)|]\right)\nonumber\\
       &+2T\sqrt{D_0+\boldsymbol v_{0,i}}+\frac{4\sqrt{D_1}}{\sqrt{1-\beta_2}}\sum_{t=1}^T\mathbb E[|\partial_i f(\boldsymbol x_t)|]\Bigg).
\end{flalign}
By rearranging the items in \eqref{eq:ll90}, it further follows that
\begin{flalign}\label{eq:ll91}
    &\frac{1}{T}\sum_{t=1}^T\mathbb E\left[\frac{\|\nabla f(\boldsymbol x_t)\|^2}{\sqrt{\beta_2 \|\boldsymbol v_{t-1}\|+\zeta}}\right]\nonumber\\
    \le &\frac{1}{T}\sum_{i}^d\sum_{t=1}^T\mathbb E\left[\frac{(\partial_i f(\boldsymbol x_t))^2}{\sqrt{\beta_2 \boldsymbol v_{t-1,i}+\zeta}}\right]\nonumber\\
    \le & \frac{1}{T}\left(\frac{7C_1\Delta'}{\eta(1-\beta_1)}+\frac{7\eta C_3(1-\beta_1)}{C_1(1-\beta_2)}(2d\sqrt{D_0+\|\boldsymbol v_{0}\|}-\sum_{i=1}^d\ln(\boldsymbol v_{0,i}))+\frac{14\eta^2C_1C_4(1-\beta_1)d}{(1-\beta_2)^2(1-\frac{\beta_1}{\sqrt[4]{\beta_2}})^2}\sqrt{D_0+\|\boldsymbol v_{0}\|} \right)\nonumber\\
&+\frac{ 7\alpha_0dD_0(1-\beta_2)}{2C_2\sqrt{\zeta}}+\frac{7\alpha_0\alpha_1D_1^2L_0^2\eta^2d^2(1-\beta_1)^2}{2(1-\beta_2)C_2^3\sqrt{\zeta}}+\frac{7\alpha_3 \beta_1(1-\beta_2)d}{2C_2}+\frac{14\eta dC_3(1-\beta_1)}{C_1}\nonumber\\
&+\frac{14\eta^2C_1C_4(1-\beta_1)d\sqrt{D_0+\|\boldsymbol v_0\|}}{(1-\beta_2)(1-\frac{\beta_1}{\sqrt[4]{\beta_2}})^2}\nonumber\\
    &+\left(\frac{14\eta C_3(1-\beta_1)\sqrt{D_1}}{C_1\sqrt{1-\beta_2}}
+\frac{42\eta^2C_1C_4(1-\beta_1)\sqrt{D_1}}{(1-\beta_2)^{1.5}(1-\frac{\beta_1}{\sqrt[4]{\beta_2}})^2}\right)\left(\frac{\sum_{t=1}^T\sqrt{d}\mathbb E[\|\nabla f(\boldsymbol x_t)\|]}{T}\right).
\end{flalign}
For $\eta\le C_5(1-\beta_2)$, (where $C_5 >0$ and will be introduced in Appendix \ref{proof:theorem2}),  $C_6=\min\Bigg(\frac{C_2\sqrt{\zeta}}{21\alpha_0dD_0},\frac{C_2^3\sqrt{\zeta}}{21\alpha_0\alpha_1D_1^2L_0^2(1-\beta_1)^2d^2C_5^2},\frac{C_2}{21\alpha_3d\beta_1},
\frac{C_1}{84C_3C_5d(1-\beta_1)},\frac{(1-\frac{\beta_1}{\sqrt[4]{\beta_2}})^2}{84C_1C_4C_5^2(1-\beta_1)d\sqrt{D_0+\|\boldsymbol v_0\|}},\frac{C_1^2}{784C_3^2C_5^2(1-\beta_ 1)^2dD_1},\\ \frac{(1-\frac{\beta_1}{\sqrt[4]{\beta_2}})^4}{7056C_1^2C_4^2C_5^4dD_1}\Bigg)$,
 $1-\beta_2\le C_6 \epsilon^2$, 
and $T\ge\max\Bigg(\frac{126C_1\Delta'}{\eta(1-\beta_1)\epsilon^2}, \frac{126 C_3C_5(1-\beta_1)}{C_1}(2d\sqrt{D_0+ \|\boldsymbol v_0\|}-\sum_{i=1}^d\ln(\boldsymbol v_{0,i}))\epsilon^{-2}, \frac{252C_5^2C_1C_4(1-\beta_1)d}{(1-\frac{\beta_1}{\sqrt[4]{\beta_2}})^2}\sqrt{D_0+\|\boldsymbol v_0\|}\epsilon^{-2} \Bigg)$,
\eqref{eq:ll91} can be further written as 
\begin{flalign}\label{eq:7bound}
   &\frac{1}{T}\sum_{t=1}^T\mathbb E\left[\frac{\|\nabla f(\boldsymbol x_t)\|^2}{\sqrt{\beta_2 \|\boldsymbol v_{t-1}\|+\zeta}}\right]\nonumber\\
   &\le \epsilon^2+\left(\frac{14\eta C_3(1-\beta_1)\sqrt{dD_1}}{C_1\sqrt{1-\beta_2}}
+\frac{42\eta^2C_1C_4(1-\beta_1)\sqrt{dD_1}}{(1-\beta_2)^{1.5}(1-\frac{\beta_1}{\sqrt[4]{\beta_2}})^2}\right)\left(\frac{\sum_{t=1}^T\mathbb E[\|\nabla f(\boldsymbol x_t)\|]}{T}\right).
\end{flalign}
Since we have that $\eta\le C_5(1-\beta_2)$ and $1-\beta_2\le C_6\epsilon^2,$ thus we have that
\begin{flalign}
   &\frac{1}{T}\sum_{t=1}^T\mathbb E\left[\frac{\|\nabla f(\boldsymbol x_t)\|^2}{\sqrt{\beta_2 \|\boldsymbol v_{t-1}\|+\zeta}}\right]
   \le \epsilon^2+\epsilon\left(\frac{\sum_{t=1}^T\mathbb E[\|\nabla f(\boldsymbol x_t)\|]}{T}\right).
\end{flalign}

This completes the proof.

\section{Formal Version of Theorem \ref{theorem:2} and Its Proof}\label{proof:theorem2}
For $\alpha_0\ge \frac{21}{2C_2},\alpha_1\ge \frac{21\alpha_0}{2C_2},\alpha_3\ge \frac{7\beta_1\sqrt{\zeta}}{2C_2}, \alpha_4\ge \frac{14L_1\sqrt{d}(2-C_1)^2}{C_1(1-\beta_1)}$, $C_1, C_2$ defined in Appendix \ref{proof:lemma3}, define
\begin{flalign}
    \Delta'=&f(\boldsymbol u_1)-\inf_{\boldsymbol{x}}f(\boldsymbol{x})+ \frac{\eta \alpha_0dD_0(1-\beta_1)}{2C_1C_2\sqrt{\zeta}}+ \frac{\eta \alpha_0D_1(1-\beta_1)\|\nabla f(\boldsymbol x_1)\|^2}{2C_1C_2\sqrt{\zeta}},\nonumber\\
    C_3=&\Bigg(\frac{2(1-C_1)^2\sqrt{d}L_0}{C_1^2} + \frac{(2-C_1)\sqrt{d}L_0}{C_1^2} +{\sqrt{d}L_0}{}\frac{(1-C_1)^2+1}{C_1^2}\Bigg),\nonumber\\
    C_4=&\left(\frac{\alpha_4(1-\beta_1)^2dL_1((1-C_1)+(1-C_1)^2)}{2C_1^2C_2^2}+{\sqrt{d}L_1}\frac{2+2(2-C_1)^2}{C_1^2}\frac{\alpha_4(1-\beta_1)^2}{2C_2^2}\right),\nonumber\\
    C_5=&\min\left(\frac{C_1}{112C_3(1-\beta_1)dD_1},\frac{1-\frac{\beta_1}{\sqrt[4]{\beta_2}}}{168D_1C_1C_4(1-\beta_1)d}\right),\nonumber\\
    C_6=&\min\Bigg(\frac{C_2\sqrt{\zeta}}{21\alpha_0dD_0},\frac{C_2^3\sqrt{\zeta}}{21\alpha_0\alpha_1D_1^2L_0^2(1-\beta_1)^2d^2C_5^2},\frac{C_2}{21\alpha_3d\beta_1}, \nonumber\\
    &\frac{C_1}{84C_3C_5d(1-\beta_1)},\frac{(1-\frac{\beta_1}{\sqrt[4]{\beta_2}})^2}{84C_1C_4C_5^2(1-\beta_1)d\sqrt{D_0+\|\boldsymbol v_0\|}},\frac{C_1^2}{784C_3^2C_5^2(1-\beta_ 1)^2dD_1}, \frac{(1-\frac{\beta_1}{\sqrt[4]{\beta_2}})^4}{7056C_1^2C_4^2C_5^4dD_1}\Bigg),\nonumber\\
    \Lambda_4=&\min\left(\frac{C_2^2}{21\alpha_0\sqrt{d}D_1L_1(1-\beta_1)}, \frac{C_1C_2}{
\sqrt{d}L_1(1-C_1)(1-\beta_1)}\right),\nonumber\\
\Lambda_5= &\left(\frac{126 C_3C_5(1-\beta_1)}{C_1}(2d\sqrt{D_0+\|\boldsymbol v_0\|}-\sum_{i=1}^d\ln(\boldsymbol v_{0,i}))\right),\nonumber\\
\Lambda_6= &\left( \frac{252C_5^2C_1C_4(1-\beta_1)d}{(1-\frac{\beta_1}{\sqrt[4]{\beta_2}})^2}\sqrt{D_0+\|\boldsymbol v_0\|} \right).\nonumber
\end{flalign}
We then have the following theorem:
\begin{theorem}
    Let Assumptions \ref{assump:lowerbound}, \ref{assump:variance} and \ref{assump:generalsmooth} hold. 
Let $1-\beta_2= \min\left(\frac{2C_2}{7\alpha_0D_1},C_6 \epsilon^2\right)= \mathcal O(\epsilon^2)$, $0<\beta_1\le \sqrt{\beta_2}<1$, 
$\eta\le \min\left(\Lambda_4\sqrt{1-\beta_2},C_5(1-\beta_2)\right)= \mathcal O(\epsilon^2)$, and
$T\ge\max\left(\frac{126C_1\Delta'}{\eta(1-\beta_1)\epsilon^2},\Lambda_5\epsilon^{-2},\Lambda_6 \epsilon^{-2}\right)= \mathcal O(\epsilon^{-4})$.
    For small $\epsilon$ such that $\epsilon\le \frac{\sqrt{2C_2}}{\sqrt{7\alpha_0C_6D_1}}$,
    we have that
 \begin{flalign}
     \frac{1}{T} \sum_{t=1}^T  \mathbb E[\|\nabla f(x_t)\|]\le \left(2c+\sqrt{2c}+\frac{4\sqrt{dD_1}}{\sqrt{C_6}}\right)\epsilon.
\end{flalign}
\end{theorem}
\begin{proof}
The proof follows similarly as the one in Appendix \ref{proof:theorem1}. According to \eqref{eq:7bound} in the proof of Corollary \ref{co:2}, we have that
\begin{flalign}
        &\frac{1}{T} \sum_{t=1}^T \mathbb E\left[\frac{ \|\nabla f(\boldsymbol x_t)\|^2}{\sqrt{{\beta_2 \|\boldsymbol v_{t-1}}\| + \zeta}} \right]\nonumber\\
        &\le \epsilon^2+\left(\frac{14\eta C_3(1-\beta_1)\sqrt{dD_1}}{C_1\sqrt{1-\beta_2}}
+\frac{42\eta^2C_1C_4(1-\beta_1)\sqrt{dD_1}}{(1-\beta_2)^{1.5}(1-\frac{\beta_1}{\sqrt[4]{\beta_2}})^2}\right)\left(\frac{\sum_{t=1}^T\mathbb E[\|\nabla f(\boldsymbol x_t)\|]}{T}\right).
\end{flalign}

According to Lemma \ref{lemma:2}, we have that
 \begin{flalign}
        &\left(\frac{1}{T} \sum_{t=1}^T  \mathbb E[\sqrt{{\beta_2 \|\boldsymbol v_{t-1}}\| + \zeta}]\right)\nonumber\\
        &\le  c
   +\frac{2\sqrt{dD_1}}{\sqrt{(1-\beta_2)}} \frac{\sum_{t=1}^T\mathbb E[\|\nabla f(\boldsymbol x_t)\|]}{T}.
    \end{flalign}

Define $e=\frac{1}{T} \sum_{t=1}^T  \mathbb E[\|\nabla f(\boldsymbol x_t)\|]$. By H\"older's inequality, we can show that
\begin{flalign}\label{eq:ll97}
    \left(\frac{1}{T} \sum_{t=1}^T  \mathbb E[\|\nabla f(\boldsymbol x_t)\|]\right)^2\le \left(\frac{1}{T} \sum_{t=1}^T \mathbb E\left[\frac{ \|\nabla f(\boldsymbol x_t)\|^2}{\sqrt{{\beta_2 \|\boldsymbol v_{t-1}}\| + \zeta}} \right]\right)\left(\frac{1}{T} \sum_{t=1}^T  \mathbb E[\sqrt{{\beta_2 \|\boldsymbol v_{t-1}}\| + \zeta}]\right).
\end{flalign}
By Lemma \ref{lemma:2} and Corollary \ref{co:2}, \eqref{eq:ll97} can be further written as 
\begin{flalign}
     e^2&\le \left(\epsilon^2+\left(\frac{14\eta C_3(1-\beta_1)\sqrt{dD_1}}{C_1\sqrt{1-\beta_2}}
+\frac{42\eta^2C_1C_4(1-\beta_1)\sqrt{dD_1}}{(1-\beta_2)^{1.5}(1-\frac{\beta_1}{\sqrt[4]{\beta_2}})^2}\right)e\right)\left(c+\frac{2\sqrt{dD_1}}{\sqrt{1-\beta_2}}e \right)\nonumber\\
    &\le c\epsilon^2+ce\epsilon+\frac{2\sqrt{dD_1}}{\sqrt{C_6}\epsilon}e\epsilon^2 +\frac{e^2}{2},
\end{flalign}
where the second inequality is due to the fact that $\left(\frac{14\eta C_3(1-\beta_1)\sqrt{dD_1}}{C_1\sqrt{1-\beta_2}}
+\frac{42\eta^2C_1C_4(1-\beta_1)\sqrt{dD_1}}{(1-\beta_2)^{1.5}(1-\frac{\beta_1}{\sqrt[4]{\beta_2}})^2}\right)\le \epsilon$ and $\left(\frac{14\eta C_3(1-\beta_1)\sqrt{dD_1}}{C_1\sqrt{1-\beta_2}}
+\frac{42\eta^2C_1C_4(1-\beta_1)\sqrt{dD_1}}{(1-\beta_2)^{1.5}(1-\frac{\beta_1}{\sqrt[4]{\beta_2}})^2}\right)\frac{2\sqrt{dD_1}}{\sqrt{1-\beta_2}}\le \frac{1}{2}$ if  $\eta\le  C_5(1-\beta_2)$, $1-\beta_2=\min\left(\frac{2C_2}{7\alpha_0D_1},C_6 \epsilon^2\right)= C_6\epsilon^2$ and $C_5=\min\left(\frac{C_1}{112C_3(1-\beta_1)dD_1},\frac{1-\frac{\beta_1}{\sqrt[4]{\beta_2}}}{168D_1C_1C_4(1-\beta_1)d}\right)$.

Thus, we have that $$ \frac{1}{T} \sum_{t=1}^T  \mathbb E[\|\nabla f(x_t)\|]=e\le\left(2c+\sqrt{2c}+\frac{4\sqrt{dD_1}}{\sqrt{C_6}}\right)\epsilon, $$
which completes the proof.
\end{proof}

\section{Experiments}
{
In this section, we provide numerical experiments to verify the coordinate-wise generalized smoothness and affine noise variance conditions. We follow the same setting of the LSTM language model \citep{zhang2019gradient} for the  Penn Treebank (PTB) \citep{mikolov2010recurrent} dataset. The model is a 3-layer LSTM language model with hidden size of $1150$ and embedding size of $400$. The training details follow \cite{merity2017regularizing}.

Given $\boldsymbol{x_t}$ and $\boldsymbol{x_{t+1}}$, we
estimate the coordinate-wise smoothness by 
\begin{flalign}
    L_{t,i}=\max_{\gamma\in \{\delta_1, \delta_2,...., \delta_N\}}\frac{|\partial_i f(\boldsymbol x_t+\gamma(\boldsymbol x_{t+1}-\boldsymbol x_t))-\partial_i f(\boldsymbol x_{t})|}{\gamma\|\boldsymbol x_t-\boldsymbol x_{t+11}\|},
\end{flalign}
where $\{\delta_1, \delta_2,...., \delta_N\}$ denotes for the sample locations.
We then show the training results for coordinate-wise smoothness vs. absolute gradient value in Fig. \ref{fig:combined}. 
In Fig. \ref{fig:combined2}, we plot the coordinate-wise gradient standard deviation vs. absolute gradient value. 
}
\begin{figure}[ht]
    \centering
    \begin{subfigure}{0.45\textwidth}
        \centering
        \includegraphics[width=\linewidth]{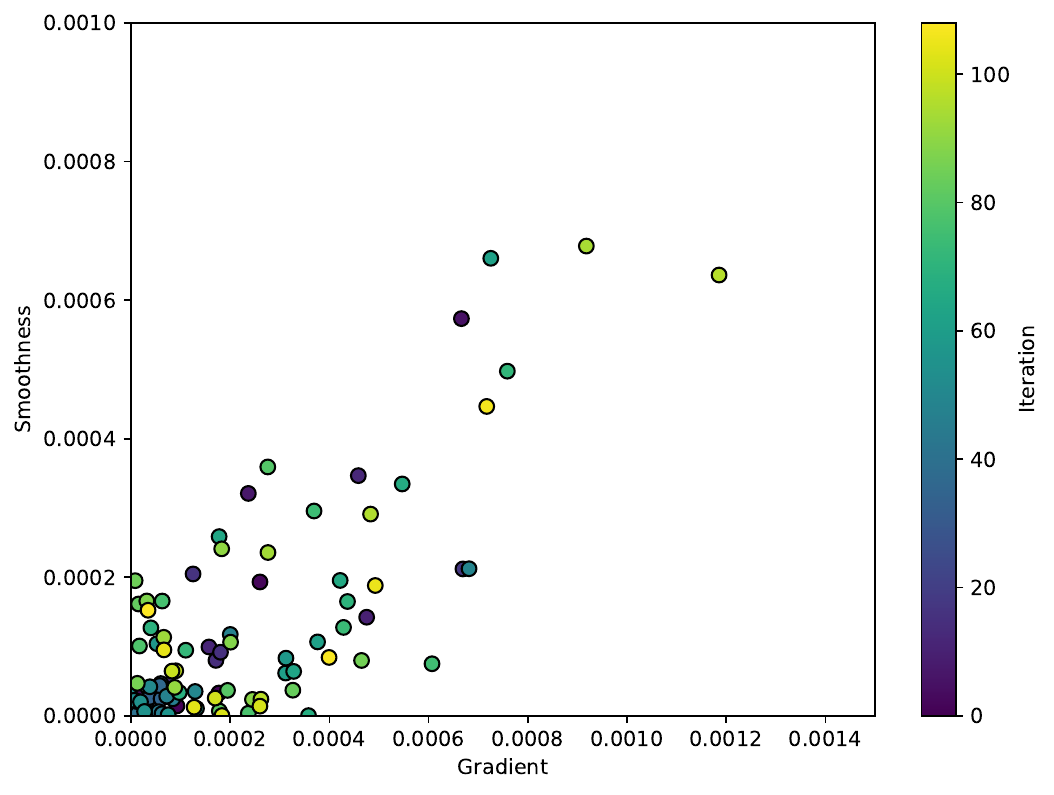} 
    \end{subfigure}
    \hfill
    \begin{subfigure}{0.45\textwidth}
        \centering
        \includegraphics[width=\linewidth]{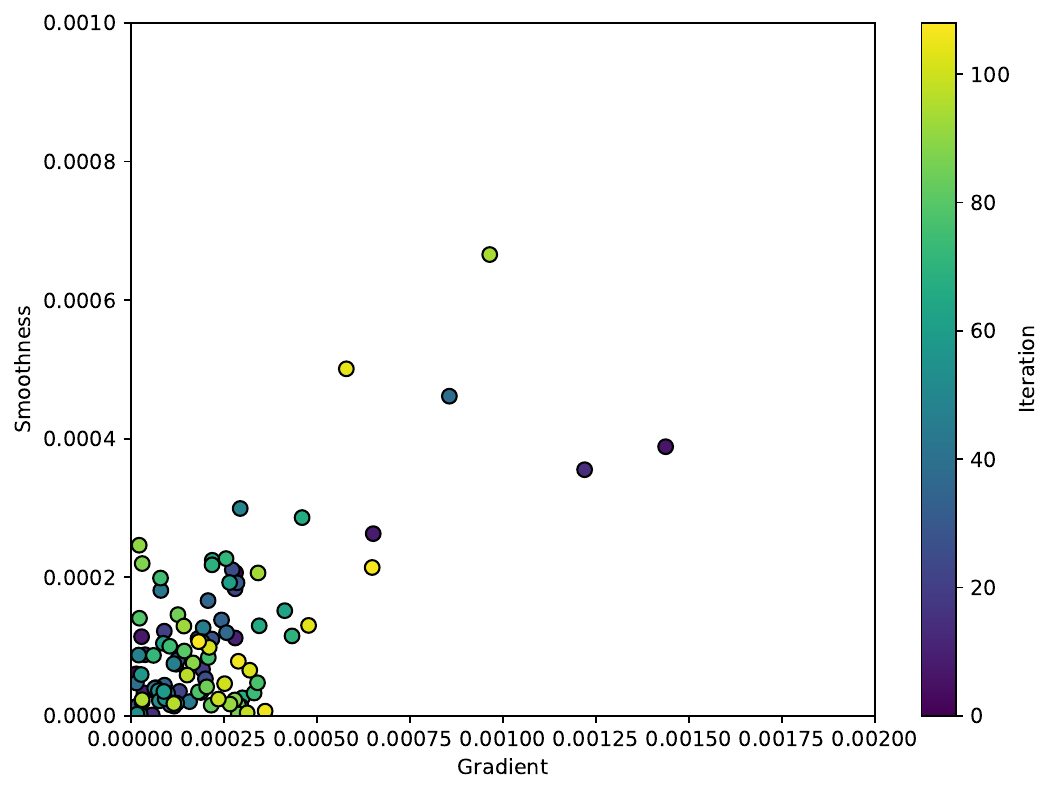} 
    \end{subfigure}    
    \begin{subfigure}{0.45\textwidth}
        \centering
        \includegraphics[width=\linewidth]{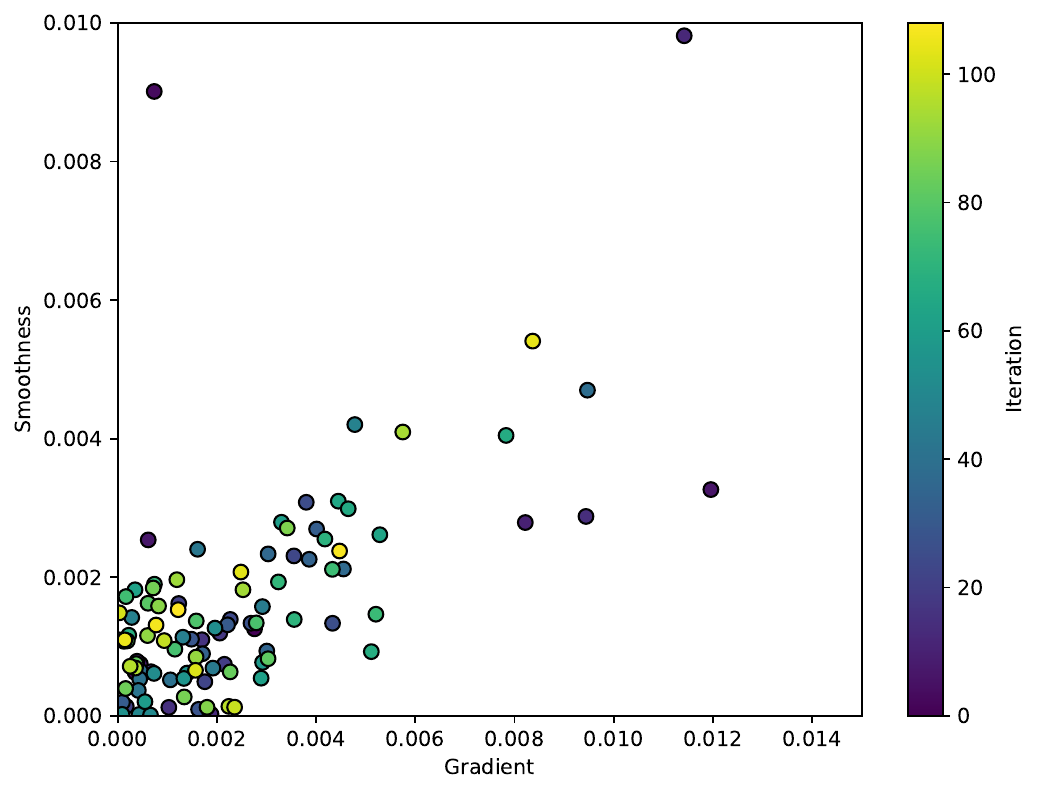} 
    \end{subfigure}
    \hfill
    \begin{subfigure}{0.45\textwidth}
        \centering
        \includegraphics[width=\linewidth]{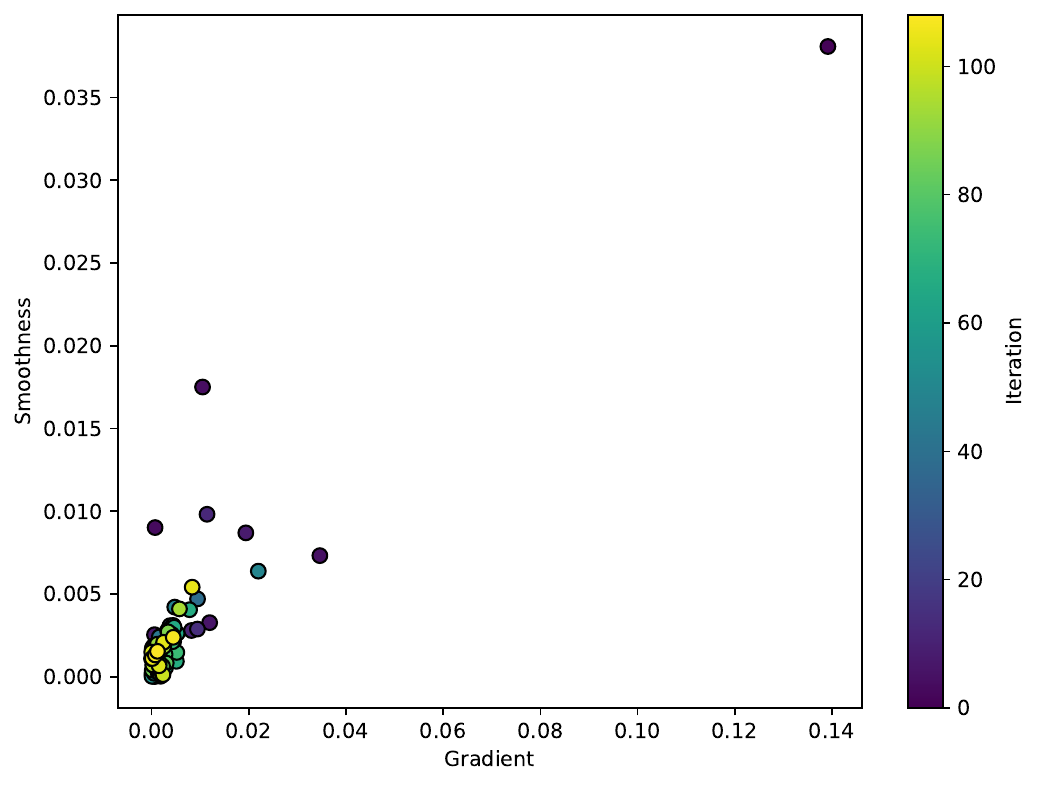} 
    \end{subfigure}
    \caption{Coordinate-wise smoothness vs. absolute gradient value on LSTM language model for the PTB datatset. Each figure presents one randomly selected coordinate.  }
    \label{fig:combined}
\end{figure}
\zhang{
\begin{figure}[ht]
    \centering
    \begin{subfigure}{0.45\textwidth}
        \centering
        \includegraphics[width=\linewidth]{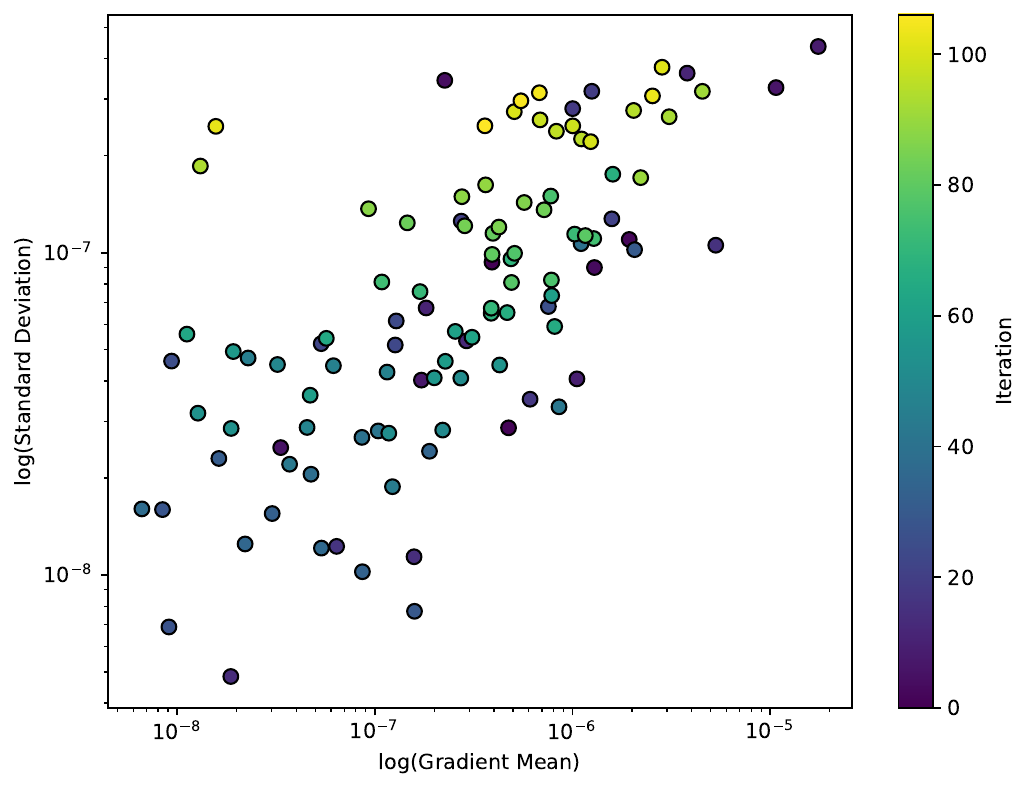} 
    \end{subfigure}
    \hfill
    \begin{subfigure}{0.45\textwidth}
        \centering
        \includegraphics[width=\linewidth]{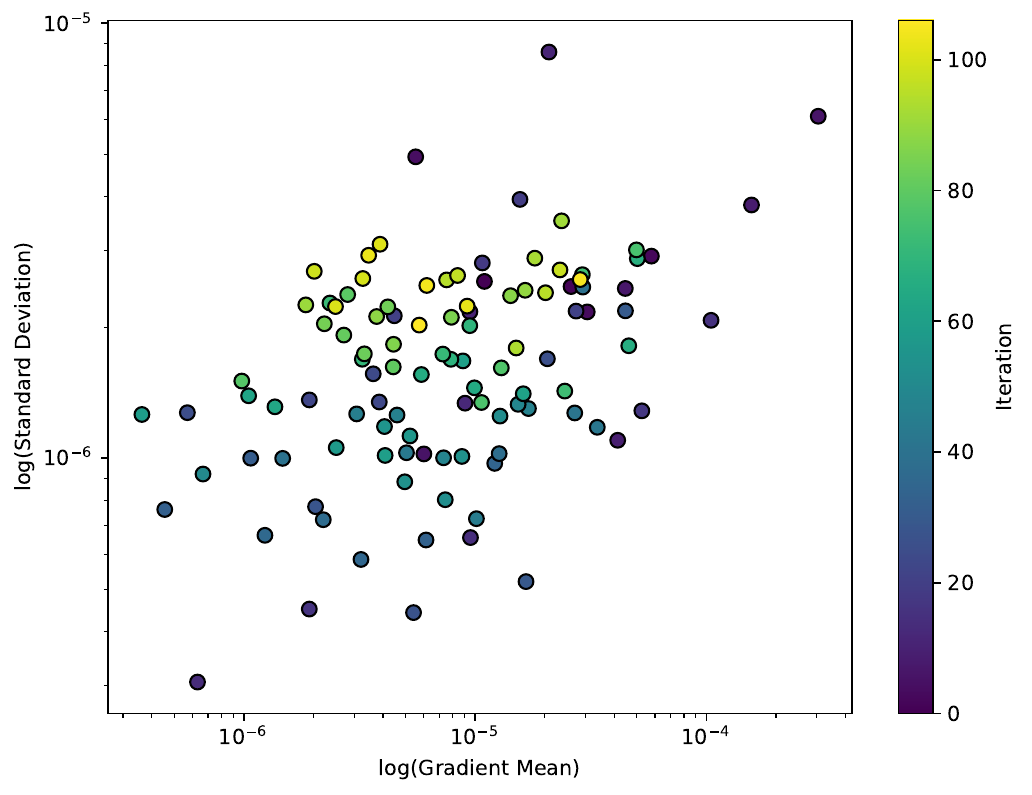} 
    \end{subfigure}    
    \begin{subfigure}{0.45\textwidth}
        \centering
        \includegraphics[width=\linewidth]{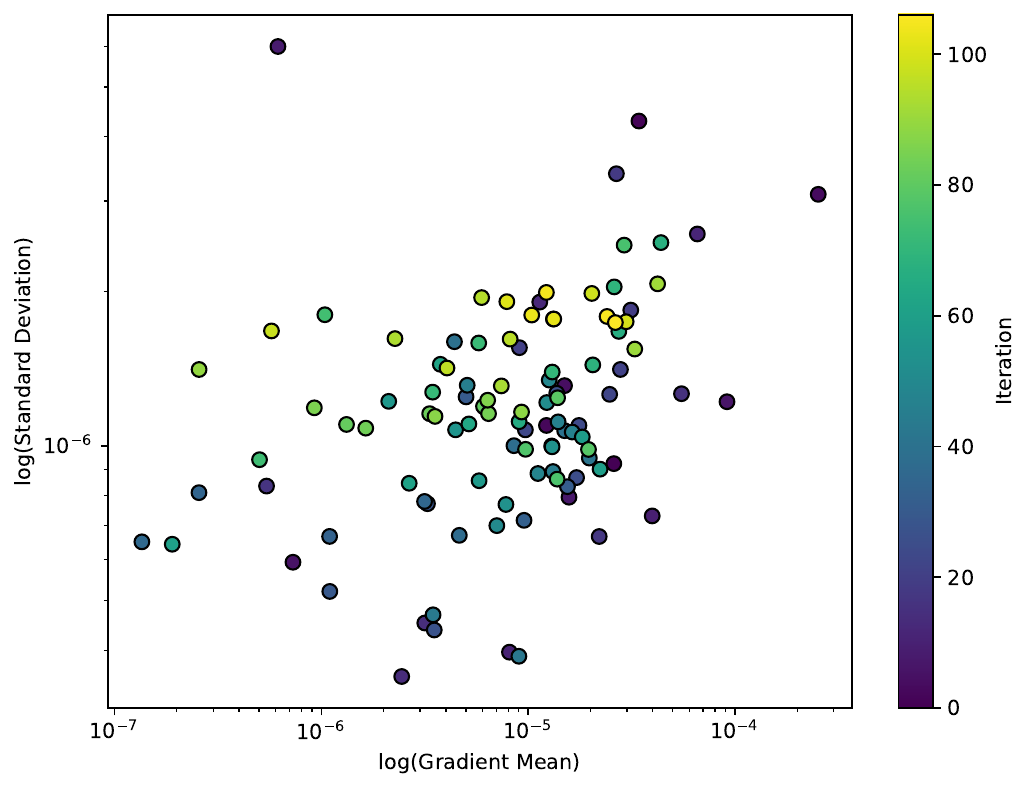} 
    \end{subfigure}
    \hfill
    \begin{subfigure}{0.45\textwidth}
        \centering
        \includegraphics[width=\linewidth]{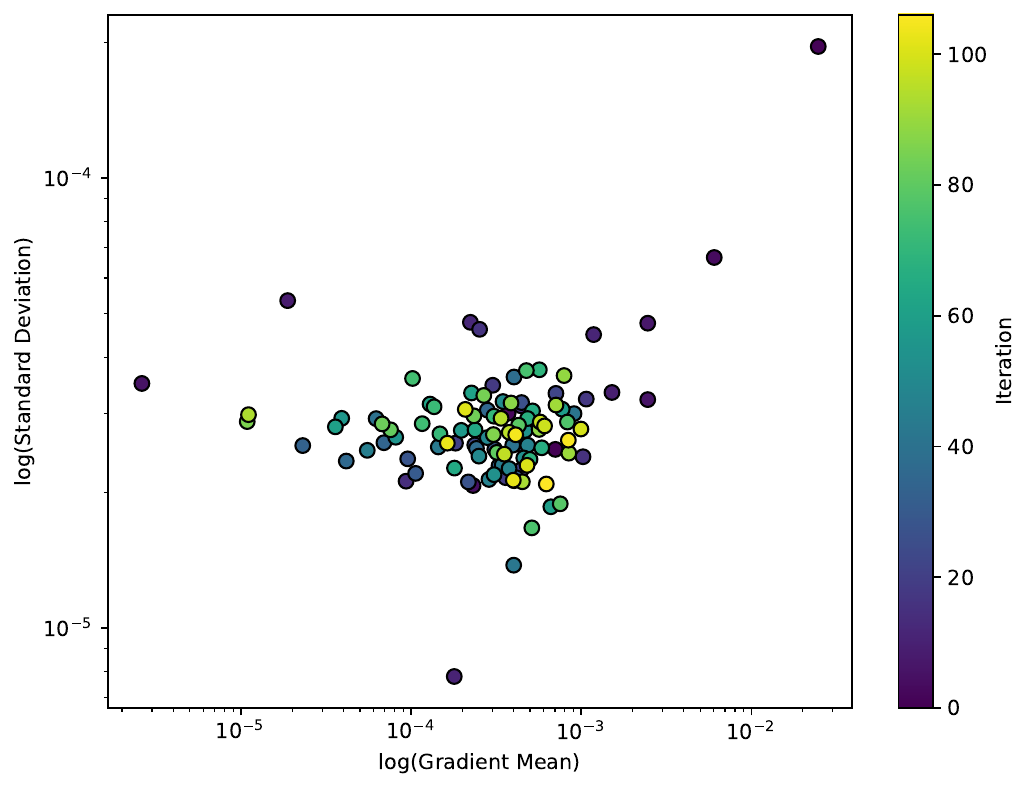} 
    \end{subfigure}
    \caption{Coordinate-wise gradient standard deviation vs. absolute gradient value on LSTM language model for the PTB datatset. Each figure presents one randomly selected coordinate.  }
    \label{fig:combined2}
\end{figure}

}

\end{document}